\documentclass{article}

    \usepackage[preprint]{neurips_2025}

\usepackage[utf8]{inputenc} 
\usepackage[T1]{fontenc}
\usepackage{hyperref}       
\usepackage{url}            
\usepackage{booktabs}       
\usepackage{nicefrac}       
\usepackage{microtype}      
\usepackage{xcolor}         

\usepackage{enumitem}
\usepackage{wrapfig}

\usepackage{mymacros}

\usepackage{tikz}
\usetikzlibrary{arrows,automata}
\usetikzlibrary{positioning}
\usetikzlibrary{arrows.meta,calc,decorations.markings,math}

\title{Robustness in the Face of Partial Identifiability\\in Reward Learning}

\author{%
Filippo Lazzati\\
Politecnico di Milano\\
Milan, Italy\\
\texttt{filippo.lazzati@polimi.it} \\
\And
Alberto Maria Metelli \\
Politecnico di Milano\\
Milan, Italy\\
}

\usepackage[ruled,linesnumbered]{algorithm2e}
\usepackage{subcaption}
\SetKwComment{Comment}{/* }{ */} 
\allowdisplaybreaks[4]

\let\oldnl\nl
\newcommand{\nonl}{\renewcommand{\nl}{\let\nl\oldnl}}

\begin{document}
\setlength{\abovedisplayskip}{5pt}
\setlength{\belowdisplayskip}{5pt}
\setlength{\textfloatsep}{11pt}

\renewcommand\thmcontinues[1]{Continued}

\maketitle

\begin{abstract}
    In Reward Learning (ReL), we are given \emph{feedback} on an unknown
    \emph{target reward}, and the goal is to use this information to recover it
    in order to carry out some downstream \emph{application}, e.g., planning.
    When the feedback is not informative enough, the target reward is only
    \emph{partially identifiable}, i.e., there exists a set of rewards, called
    the \emph{feasible set}, that are equally plausible candidates for the
    target reward.
    In these cases, the ReL algorithm might recover a reward function different
    from the target reward, possibly leading to a failure in the application.
    In this paper, we introduce a general ReL framework that permits to
    \emph{quantify} the drop in ``performance'' suffered in the considered
    application because of identifiability issues.
    Building on this, we propose a \emph{robust} approach to address the
    identifiability problem in a principled way, by maximizing the
    ``performance'' with respect to the worst-case reward in the feasible set.
    We then develop \rob, a ReL algorithm that applies this robust approach to
    the subset of ReL problems aimed at assessing a preference between two
    policies, and we provide theoretical guarantees on sample and iteration
    complexity for \rob.
    We conclude with a proof-of-concept experiment to illustrate the considered
    setting.
\end{abstract}

\section{Introduction}\label{sec: introduction}

Reward Learning (ReL) is the problem of learning a reward function from data
\cite{jeon2020rewardrational}. When the data are demonstrations, ReL is known as
Inverse Reinforcement Learning (IRL) \cite{russell1998learning}, whereas when
the data are (pairwise) comparisons of trajectories, ReL is usually called
Preference-based Reinforcement Learning (PbRL) \cite{wirth2017pbrl} or
Reinforcement Learning from Human Feedback (RLHF)
\cite{kaufmann2024surveyreinforcementlearninghuman}.

The main strength of ReL is that the reward function that it aims to learn,
referred to as the \emph{target reward}, corresponds to ``a succinct and
transferable representation of the preferences of an agent''
\cite{russell1998learning,arora2020survey}. As such, \emph{ideally}, ReL allows
the use of datasets of demonstrations and comparisons for a variety of important
applications, such as reward design \cite{hadfield2017inverse}, Imitation
Learning (IL) \cite{abbeel2004apprenticeship}, risk-sensitive IL
\cite{lacotte2019riskil}, preference inference \cite{hadfield2016cirl}, behavior
transfer across environments \cite{Fu2017LearningRR}, behavior improvement
\cite{syed2007game}, and, more generally, any task that can be carried out using
a reward.

However, in practice, ReL has been successfully applied mainly to IL
\cite{finn2016guided} and reward design \cite{christiano2017deep}. The primary
obstacle to the adoption of ReL algorithms for other applications is
\emph{partial identifiability}
\cite{cao2021identifiability,kim2021rewardidentification,skalse2023invariancepolicyoptimisationpartial}.
This arises when the available feedback (demonstrations, comparisons, or
otherwise) does not allow to \emph{uniquely} identify the target reward, but
instead leads to a set of rewards (referred to as the \emph{feasible set}
\cite{metelli2021provably,metelli2023towards}) that represent equally-plausible
candidates for the target reward. As a result, the recovered reward might differ
from the target reward, potentially causing failure in the downstream
application. As noted by several works
\cite{cao2021identifiability,skalse2023invariancepolicyoptimisationpartial,finn2016guided},
most existing ReL methods, including
\cite{ng2000algorithms,ratliff2006maximum,ziebart2008maxent,boularias2011relative,wulfmeier2016maximumentropydeepirl,christiano2017deep},
are sensitive to this issue.

The standard solution in the literature is to try to ensure that the feasible
set contains (almost) only the target reward. This is typically achieved by
collecting additional feedback
\cite{amin2016resolving,cao2021identifiability,schlaginhaufen2024transferability,lazzati2024utility}
or by imposing further assumptions on the available feedback
\cite{kim2021rewardidentification}. However, in practice, additional feedback
may not be available, and the added assumptions may be too strong.

A more powerful and general approach was recently proposed by
\cite{skalse2023invariancepolicyoptimisationpartial}.
Rather than requiring the target reward to be uniquely identifiable, they
instead ask that the feasible set contains only reward functions that are
``equivalent'' to the target reward \emph{w.r.t. the considered application}.
Intuitively, this milder condition is applicable only when we have prior
knowledge of the intended use (i.e., the application) of the learned reward,
which is almost always the case
\cite{ng2000algorithms,ziebart2008maxent,Fu2017LearningRR,christiano2017deep}.

However, this approach has two drawbacks. First, it is \emph{difficult to apply}
in practice because, except for simple feedback and applications, it is
non-trivial to verify whether the equivalence condition holds. Second, it is
\emph{qualitative}: if the feasible set contains a reward that is not equivalent
to all others, then \cite{skalse2023invariancepolicyoptimisationpartial}
classify the ReL problem as prone to failure, without quantifying how severe the
difference is. Intuitively, if this is small, the downstream application might
still be carried out nearly successfully.

In this paper, we present a novel general framework for ReL that enables
\emph{quantitative} considerations. Based on this, we propose an
\emph{easy-to-apply} robust approach for addressing the identifiability problem.

\textbf{Contributions.}~~The contributions of this paper are summarized as
follows.
\begin{itemize}
[noitemsep, leftmargin=*, topsep=-2pt]
  \item We introduce a new quantitative framework for ReL (Section \ref{sec: rel
  problem formulation}).
  \item We propose a robust approach for tackling the identifiability problem
  (Section \ref{sec: robustness}).
  \item We present \rob, an efficient algorithm implementing the robust approach
  (Section \ref{sec: use case}).
  \item We conduct an illustrative experiment to exemplify the problem setting
  and the method (Section \ref{sec: numerical simulations}).
\end{itemize}

All results are proved in Appendix \ref{apx: proofs rob}, while additional
related work is discussed in Appendix \ref{apx: additional rel work}.

\section{Preliminaries}\label{sec: preliminaries}

{\thinmuskip=2.5mu
\medmuskip=2.5mu \thickmuskip=2.5mu
\paragraph{Notation.} Given $N \in \mathbb{N}$, we denote
$\dsb{N}\coloneqq\{1,\dots,N\}$.
Given a finite set $\cX$, we denote by $|\cX|$ its cardinality and by
$\Delta^\cX$ the probability simplex on $\cX$. Given two sets $\cX$ and $\cY$,
we denote the set of conditional distributions as
$\Delta_\cY^\cX\coloneqq\{q:\cY\rightarrow\Delta^\cX\}$.
We use $\RR_+^k$ to denote the non-negative orthant in $k$ dimensions.
A vector $v\in\RR^k$ is a subgradient for a function $h:\RR^k\to\RR$ at
$u\in\RR^k$ if, for all $w\in\RR^k$ in the domain of $h$, it holds that $h(w)\ge
h(u)+v^\intercal (w-u)$.
Sometimes, we use $\dotp{v,w}=v^\intercal w$ for the dot product of vectors
$v,w\in\RR^k$.
We say that a function $d:\cX\times\cX\to\RR_+$ is a \emph{premetric} if, for
all $x\in\cX$, we have $d(x,x)=0$.
Moreover, for any $x\in\cX$, we denote the $\ell_2$-projection onto a set $\cY$
as any point such that: $\Pi_{\cY}(x)\in\argmin_{y\in\cY}\|x-y\|_2$.}

{\thinmuskip=2.5mu
\medmuskip=2.5mu \thickmuskip=2.5mu
\paragraph{Markov Decision Processes (MDPs).} A finite-horizon \emph{Markov
decision process} (MDP) \cite{puterman1994markov} is defined as a tuple
$\cM\coloneqq\tuple{\cS,\cA,H,s_0,p,r}$, where $\cS$ is the finite state space
($S\coloneqq|\cS|$), $\cA$ is the finite action space ($A\coloneqq|\cA|$),
$s_0\in\cS$ is the initial state, $H \in \mathbb{N}$ is the horizon,
$p\in\Delta^\cS_{\SAH}$ is the transition model, and $r\in\fR\coloneqq\{r:\SAH\to[0,1]\}$ is the
reward.
A policy is a mapping $\pi\in\Pi\coloneqq\Delta_{\SH}^{\cA}$. We let $\P_{\pi}$
denote the probability distribution induced by $\pi$ in $\cM$ starting from
$s_0$ (we omit $s_0,p$ for simplicity), and $\E_{\pi}$ denote the expectation
w.r.t. $\P_{\pi}$.
%
The visitation distribution induced by $\pi$ in $\cM$ is defined as
$d^{\pi}_h(s,a)\coloneqq\P_{\pi}(s_h=s,a_h=a)$ for all $s,a,h$, so that
$\sum_{(s,a)\in\SA}d^{\pi}_h(s,a) = 1$ for every $h \in \dsb{H}$.
We denote the set of all state-action trajectories as $\Omega\coloneqq
(\SA)^H\times\cS$.
Given a trajectory $\omega=\tuple{s_1,a_1,\dotsc,s_H,a_H,s_{H+1}}\in\Omega$, we
define the ``visitation distribution'' $d^\omega$ of $\omega$ at each $s,a,h$ as
$d^\omega_h(s,a)=\indic{s=s_h,a=a_h}$.
%
Moreover, we let $G(\omega;r)\coloneqq\sum_{h\in\dsb{H}}r_h(s_h,a_h)$ be the
return of $\omega$ under reward $r\in\fR$, and note that
$G(\omega;r)=\dotp{d^\omega,r}$.
We denote the expected return of a policy $\pi$ in MDP $\cM$ as
$J^\pi(r;p)\coloneqq \E_{\pi}\bigs{\sum_{h\in\dsb{H}}r_h(s_h,a_h)} =
\dotp{d^{\pi},r}$, the optimal policy $\pi^*$ as any policy in $\argmax_\pi
J^\pi(r;p)$, and the optimal expected return as $J^*(r;p)\coloneqq\max_\pi
J^\pi(r;p)$.
%
Finally, for any $\beta\ge0$ and stochastic policy $\pi$, we let
 $J^\pi_\beta(r;p)\coloneqq \E_{\pi}\bigs{\sum_{h\in\dsb{H}}
 (r_h(s_h,a_h)-\beta\log\pi_h(a_h|s_h))}$ be the entropy-regularized return
 \cite{Ziebart2010ModelingPA,haarnoja2017rldeepenergypolicies}.}

 \paragraph{Reward Learning (ReL).}

 In the literature
 \cite{russell1998learning,jeon2020rewardrational,skalse2023invariancepolicyoptimisationpartial},
 ReL is defined as the problem of learning an unknown \emph{target reward}
 $r^\star$ from a certain amount of \emph{feedback}, i.e., data, like
 demonstrations \cite{ng2000algorithms} or trajectory comparisons
 \cite{wirth2017pbrl}, that ``leak information'' about $r^\star$. The ultimate
 goal is to use the recovered reward for some downstream \emph{application}
 \cite{skalse2023invariancepolicyoptimisationpartial}, such as finding the
 optimal policy (planning).
 The concept of \emph{partial identifiability}
 \cite{cao2021identifiability,kim2021rewardidentification,skalse2023invariancepolicyoptimisationpartial}
 refers to the existence of multiple rewards that are equally plausible
 candidates for the target reward with respect to the given feedback.
 This set of rewards is called the \emph{feasible set}
 \cite{metelli2021provably,metelli2023towards}.

\section{A Quantitative Framework for Reward Learning}
\label{sec: rel problem formulation}

In this section, we present a new framework for studying ReL problems.
Beyond modeling feedback in a simple yet flexible way, our framework crucially
models applications in a \emph{quantitative} manner,\footnote{In Appendix
\ref{apx: comparison with skalse}, we provide a comparison of our framework with
the \emph{qualitative} framework of
\cite{skalse2023invariancepolicyoptimisationpartial}, while in Appendix
\ref{apx: misspecification and active learning} we provide a quantitative
discussion on model selection through our new framework.} paving the way to new
approaches to ReL (e.g., see Section \ref{sec: robustness}).

In our framework, we define a ReL problem
\cite{russell1998learning,jeon2020rewardrational} as a pair $\tuple{\cF,g}$,
where $\cF=\{f_i\}_i$ is a set of \emph{feedback} and $g$ is an
\emph{application}. Informally, the feedback $\cF$ represent what we know about
the unknown target reward $r^\star$, while the application $g$ represents what
we want to do with it.
In the following two sections, we formalize these important concepts.

\subsection{Feedback}\label{sec: feedback}

\begin{table}[t!]
    \centering
    \resizebox{0.92\columnwidth}{!}{%
     \begin{tabular}{||c c c||} 
     \hline
     Feeback $f$ & Feedback type and $Q_f$ & Feasible set $\cR_f$\\
     \hline\hline
     optimal expert \cite{ng2000algorithms} & \multirow{3}{*}{demonstrations
     $\pi^E$} & $\{r:\,J^{\pi^E}(r;p)=J^*(r;p)\}$\\
     $\beta$-MCE expert \cite{Ziebart2010ModelingPA} & &
     $\{r:\,\pi^E=\argmax_\pi J^\pi_\beta(r;p)\}$\\
      $t$-suboptimal expert
      \cite{poiani2024inversereinforcementlearningsuboptimal}, $t\ge 0$ & &
      $\{r:\,J^{\pi^E}(r;p)\ge J^*(r;p)-t\}$\\
     \hline
     BTL with $q$ \cite{christiano2017deep} & \multirow{2}{*}{trajectory
     comparison $(\omega^1,\omega^2)$} &
     $\{r:\,q=e^{G(\omega^1;r)}/\sum_{i\in\{1,2\}} e^{G(\omega^i;r)}\}$\\
     hard preference \cite{jeon2020rewardrational} & & $\{r:\,G(\omega^1;r)\le
     G(\omega^2;r)\}$\\
     \hline
     BTL with $q$ (\emph{new}) & \multirow{2}{*}{policy comparison $(\pi^1,\pi^2)$} &
     $\{r:\,q=e^{J^{\pi^1}(r;p)}/\sum_{i\in\{1,2\}} e^{J^{\pi^i}(r;p)}\}$\\
     hard preference (\emph{new}) & & $\{r:\,J^{\pi^1}(r;p)\le
     J^{\pi^2}(r;p)\}$\\
     \hline
     \end{tabular}%
    }
     \caption{A list of some feedback considered in literature.
     %
    For simplicity, we have grouped different feedback $f$ based on the quantity
    $Q_f$, obtaining the three categories in the first column.
    Note that the \emph{policy comparison} feedback
    are introduced in this paper for the first time and capture the situation in
    which we are given a preference on the behavior of two other agents (more in
    Appendix \ref{apx: examples of feedback}).
     MCE stands for ``Maximum Causal Entropy'', while BTL abbreviates the
     ``Bradley-Terry-Luce'' model.}
 \label{table: feedback}
\end{table}

A feedback $f$ relates a known quantity $Q_f$ with the unknown target reward
$r^\star$. We consider as feedback only those statements that can be translated
into a constraint on $r^\star$ of the type $r^\star\in\cR_f$, where
$\cR_f\subseteq\fR$ is some set of rewards associated with feedback $f$, that we
call \emph{feasible set}. See Table \ref{table: feedback} for a list of popular
feedback and their corresponding feasible sets.

For instance, saying that ``\emph{policy $Q_f=\pi^E$ is optimal for $r^\star$}''
(i.e., the ``optimal expert'' \cite{ng2000algorithms} entry in Table \ref{table:
feedback}) is an example of feedback drawn from the IRL literature, and it is
equivalent to saying that
$r^\star\in\cR_f=\{r\in\fR:\,$\scalebox{0.92}{$J^{\pi^E}(r;p)=J^*(r;p)$}$\}$.
Another example of feedback, taken from the PbRL literature, is
``\emph{trajectories $Q_f=(\omega^1, \omega^2)$ are such that the return under
$r^\star$ of $\omega^1$ is no more than that of $\omega^2$}'' (i.e., the ``hard
preference'' \cite{jeon2020rewardrational} entry in Table \ref{table:
feedback}), and corresponds to $r^\star\in\cR_f=\{r\in\fR:\,G(\omega^1;r)\le
G(\omega^2;r)\}$.
Note that our formulation is very flexible and allows us to work with almost any
feedback we desire, such as ``\emph{given $Q_f=(s,a,\omega)$, the reward
$r^\star$ of the pair $(s,a)$ is 80\% of the return of $\omega$}'',
corresponding to $r^\star\in\cR_f=\{r\in\fR:\,r(s,a)=0.8\cdot G(\omega;r)\}$.

If we are given multiple feedback $\cF=\{f_i\}_i$, then we can combine them to
obtain a smaller feasible set $\cR_\cF$ of candidates for $r^\star$. Formally,
we define the feasible set of $\cF$ as the intersection
$\cR_\cF\coloneqq\bigcap_i \cR_{f_i}$ of the feasible sets of all the feedback
in $\cF$.
Note that $\cR_\cF\subseteq \cR_{f_i}$ for every $i$, meaning that combining
multiple feedback permits to reduce our ``uncertainty'' on $r^\star$.
If $\cR_\cF\neq\{r^\star\}$, then we suffer from \emph{partial identifiability}.

\subsection{Applications}

\begin{table}[t!]
    \centering
    \resizebox{0.85\columnwidth}{!}{%
        \begin{tabular}{||c c c||} 
        \hline
        Application $g$ & Set $\cX_g$ & Loss $\cL_g(r,x)$\\
        \hline\hline
        Imitation of $\overline{\pi}$ \cite{abbeel2004apprenticeship} & $\Pi$ &
        $|J^{\overline{\pi}}(r;p)-J^x(r;p)|$\\
        Planning in $p'$ \cite{christiano2017deep,Fu2017LearningRR} & $\Pi$ &
        $J^*(r;p')-J^x(r;p')$\\
        Constrained planning with $c,k$ \cite{schlaginhaufen2023identifiability} & $\Pi_{c,k}$ &
        $\max_{\pi\in\Pi_{c,k}} J^\pi(r;p)-J^x(r;p)$\\
        Assessing a trajectory preference $\omega^1,\omega^2$ & $\RR$
        & $| x - (G(\omega^1;r) - G(\omega^2;r)) |$\\
        Assessing a policy preference $\pi^1,\pi^2$ & $\RR$
        & $| x - (J^{\pi^1}(r;p) - J^{\pi^2}(r;p)) |$\\
        Learning a reward \cite{Ramachandran2007birl} & $\fR$ & $\| x - r
        \|_2$\\
        \hline
        \end{tabular}%
    } \caption{A list of some applications considered in literature.
    Note that we used $\Pi_{c,k}\coloneqq\{\pi:\, J^\pi(c;p)\le k\}$, where $c$
        is the cost and $k$ the threshold. }
        \label{table: application}
\end{table}

We define an application $g$ as a pair $\tuple{\cX_g,\cL_g}$, where $\cX_g$ is a
set, and $\cL_g:\fR\times\cX_g\to\RR_+$ is a ``loss'' function.
An application $g$ is \emph{carried out} by choosing an $x\in\cX_g$, which
results in suffering from a loss $\cL_g(r^\star,x)$ (see examples in Table
\ref{table: application}).
To \emph{solve a ReL problem} $\tuple{\cF,g}$, we must carry out the application
$g$ while incurring the minimum possible loss, i.e., we must select an object
$x\in\cX_g$ for deployment such that the loss $\cL_g(r^\star,x)$ is as small as
possible.
Thus, ideally, the goal is to output:
\begin{align*}
    x^\star\in\argmin\limits_{x'\in\cX_g}\cL_g(r^\star,x').
\end{align*}
However, a ReL problem is not an optimization problem because the target reward
$r^\star$, and therefore the loss $\cL_g(r^\star,\cdot)$, are unknown.
For this reason, the function $\cL_g$ is defined over \emph{all} rewards
$r\in\fR$, with the meaning that $\cL_g(r,\cdot)$ quantifies the \emph{loss we
would suffer if $r$ were the target reward $r^\star$}.

An example of a ReL application $g$ is the well-known IL problem
\cite{ho2016modelfreeimitationlearningpolicy,osa2018IL} (see ``imitation of
$\overline{\pi}$'' in Table \ref{table: application}), where we aim to output a
policy $x$ that ``imitates'' some given policy $\overline{\pi}$, i.e., that
matches its expected return under the unknown reward $r^\star$. Thus, we can set
$\cX_g=\Pi$ and $\cL_g(r,x)=|J^{\overline{\pi}}(r;p)-J^x(r;p)|$, with the
intuition that, if $r=r^\star$, then the imitation error of $x$ is $\cL_g(r,x)$.

Another example of an application $g$ is planning (see ``planning in $p'$'' in
Table \ref{table: application}), which arises in many contexts, including reward
design \cite{christiano2017deep} and transferring behavior
\cite{Fu2017LearningRR}. Here, we aim to find a policy $x$ with the largest
possible expected return under $r^\star$ in some environment with different
dynamics $p'$. Thus, we have $\cX_g=\Pi$ and $\cL_g(r,x)=J^*(r;p')-J^x(r;p')$.

As a final example, consider the problem of assessing how much a trajectory
$\omega^1$ is preferred to $\omega^2$ by some agent (see Table \ref{table:
application}). Assuming that $r^\star$ models the agent's preferences, we can
view this problem as an application $g$ where $\cX_g=\RR$ and $\cL_g(r,x)=| x -
(G(\omega^1;r) - G(\omega^2;r))|$.

Note that our framework can also be used in scenarios where the ultimate goal is
learning $r^\star$ (e.g., because the application $g$ is not known yet) by
setting $\cX_g=\fR$ and using some distance between rewards for $\cL_g$, e.g.,
$\cL_g(r,x)=\|x-r\|_2$ \cite{Ramachandran2007birl} (see more in Appendix
\ref{apx: object deploy reward}).

\begin{remark}\label{remark: finite data} In this section, we considered
  feedback $f$ and applications $g$ that are fully known, in the sense that all
  the quantities (e.g., policies, transition models and other parameters)
  involved in the definitions of $\cR_f$, $\cX_g$, $\cL_g$ are known exactly.
  However, in practice, these quantities are unknown and must be estimated
  from finite samples.
  We will consider the finite-sample regime in Section \ref{sec: use case}.
\end{remark}

\section{A Robust Approach to Tackle Partial Identifiability}
\label{sec: robustness}

In the previous section, we presented a framework for formalizing ReL problems.
In this section, we introduce a novel, principled way to \emph{solve} a ReL
problem $(\cF,g)$, i.e., to select the object $x\in\cX_g$ to deploy.
We begin by reviewing the existing approaches adopted in the literature.

\paragraph{Existing approaches.}

In the literature, the majority of existing ReL methods, including the most
popular IRL
\cite{ziebart2008maxent,boularias2011relative,wulfmeier2016maximumentropydeepirl,finn2016guided}
and PbRL
\cite{christiano2017deep,ibarz2018rewardlearning,jeon2020rewardrational}
algorithms, solve a ReL problem $(\cF,g)$ by first drawing an \emph{arbitrary}
reward $\widetilde{r}$ from the feasible set $\cR_\cF$, and then deploying the
object $\widetilde{x}$ that minimizes the ``loss'' $\cL_g(\widetilde{r},\cdot)$
w.r.t. the recovered reward $\widetilde{r}$, as if $\widetilde{r}$ were the true
target reward $r^\star$:
\begin{align}\label{eq: approach existing literature}
  \widetilde{x}\in\argmin\limits_{x'\in\cX_g}\cL_g(\widetilde{r},x').
\end{align}
However, there are two main problems with this approach: $(i)$ there is no clear
motivation for why this choice of $\widetilde{x}$ is a ``good'' way for solving
the ReL problem,
because in general $\widetilde{r}\neq r^\star$
\cite{amin2016resolving,cao2021identifiability,kim2021rewardidentification}, and
so the minima of $\cL_g(\widetilde{r},\cdot)$ might incur a very large value of
the true loss $\cL_g(r^\star,\cdot)$;
$(ii)$ none of the aforementioned works provides an estimate of the true loss
$\cL_g(r^\star,\widetilde{x})$ incurred by selecting $\widetilde{x}$, thus
providing no information on whether the chosen $\widetilde{x}$ can be safely
deployed or not.
In the next two paragraphs, we present our approach to overcome these
limitations $(i)$ and $(ii)$.

\paragraph{Our approach.}

We propose a \emph{robust} approach that arises quite naturally once the
framework introduced in Section \ref{sec: rel problem formulation} is adopted.
Specifically, the idea is that, whatever choice $x'\in\cX_g$ we make, since we
know that the target reward $r^\star$ belongs to the feasible set $\cR_\cF$,
then, in the worst-case, the true loss $\cL_g(r^\star,x')$ incurred by deploying
$x'$ is upper bounded as:
\begin{align*}
  \cL_g(r^\star,x')\le \max\limits_{r\in\cR_\cF}\cL_g(r,x').
\end{align*}
For this reason, we propose to deploy the object $x_{\cF,g}\in\cX_g$ that
minimizes the loss associated with the worst possible value that $r^\star$ can
take:
\begin{align}\label{eq: minimax robust}
    x_{\cF,g}\in\argmin\limits_{x'\in\cX_g}\max\limits_{r\in\cR_\cF}\cL_g(r,x').
\end{align}
Some observations are in order.
First, whether the optimization problem in Eq. \eqref{eq: minimax robust} can be
solved efficiently depends on the specific application $g$ and feedback $\cF$ in
question.
Next, we choose to be robust (minimax) because our problem setting is not
Bayesian \cite{Ramachandran2007birl}, i.e., we do not have a distribution over
the set of rewards, but we only know that $r^\star\in\cR_\cF$.
Finally, note that many IL algorithms (e.g.,
\cite{abbeel2004apprenticeship,syed2007game,ho2016modelfreeimitationlearningpolicy})
can be seen as adopting our robust approach (see Appendix \ref{apx: additional
rel work}).

\paragraph{Quantifying the error.}

Eq. \eqref{eq: minimax robust} represents a principled way to solve ReL problems
that, unlike the approach commonly adopted in the literature (see Eq. \ref{eq:
approach existing literature}), provides worst-case guarantees.
However, we can\emph{not} solve every ReL problem $(\cF,g)$ by merely outputting
$x_{\cF,g}$, because although $x_{\cF,g}$ is the choice with the smallest
worst-case loss, the loss associated with $x_{\cF,g}$ might still be too large
in the worst case.
In other words, there are ReL problems $(\cF,g)$ that \emph{cannot be solved
robustly}, because we cannot guarantee that, in the worst case, the true loss
falls below some pre-specified threshold.
Intuitively, this happens when the application $g$ requires significant
knowledge about $r^\star$, which is not sufficiently provided by the feedback
$\cF$.
In such cases, we must collect additional feedback if available; otherwise, we
must tolerate weaker guarantees than those worst-case.

For these reasons, it is important to \emph{quantify} the loss suffered in the
worst case by $x_{\cF,g}$. Thanks to our new framework, we can compute it
as:
\begin{align}\label{eq: informativeness}
    \cI_{\cF,g}\coloneqq\max\limits_{r\in\cR_\cF}\cL_g(r,x_{\cF,g})
    =\min\limits_{x'\in\cX_g}\max\limits_{r\in\cR_\cF}\cL_g(r,x') .
\end{align}
Since $\cI_{\cF,g}$ measures how \emph{un}informative is $\cF$ for $g$, we call
it the \emph{uninformativeness} of $\cF$ for $g$.

\paragraph{A special case.}

We conclude this section with some observations on the special case where
$\cX_g=\fR$, i.e., when we aim to output a reward. In such cases, the robust
choice $x_{\cF,g}$ in Eq. \eqref{eq: minimax robust} can be interpreted as the
\emph{Chebyshev center} \cite{alimov2019chebyshev} of the feasible set $\cR_\cF$
in the premetric space $\tuple{\fR,\cL_g}$.
Building on the properties of the Chebyshev center, it is possible to derive
interesting results. The main of these is that the robust reward
$x_{\cF,g}\in\fR$ \emph{does not necessarily belong to the feasible set}
$\cR_\cF$, which is rather counterintuitive, especially because \emph{the entire
ReL literature has focused on recovering a reward function from the feasible
set} $\cR_\cF$.
We provide more details on this setting in Appendix \ref{apx: object deploy
reward}.

\section{Rob-ReL: A Robust Algorithm for ReL}
\label{sec: use case}

The goal of this section is to introduce an algorithm for solving ReL problems
using the robust approach presented in Section \ref{sec: robustness}. 
To this aim, we make two important observations.
\begin{itemize}
  [leftmargin=*]
    \item Solving ReL problems using the robust approach is \emph{not merely an
    optimization problem}. In fact, $\cF,g$ usually have to be estimated from
    finite data (see Remark \ref{remark: finite data}).
    \item \emph{No single and simple algorithm} can solve robustly all ReL
    problems. Indeed, even with infinite data, depending on $\cF$ and $g$, Eq.
    \eqref{eq: minimax robust} exhibits different properties from the optimization
    viewpoint.
\end{itemize}
For these reasons, we now focus on a specific \emph{subset} of ReL problems in
the \emph{finite-sample} regime and we present \rob, a provably efficient
algorithm for solving this subclass of ReL problems.

\subsection{The Family of ReL Problems Solvable by Rob-ReL}
\label{sec: family problems}

We consider an \emph{interesting} and \emph{explanatory} family of ReL problems
$\tuple{\cF,g}$ that have received limited attention in the literature.
Specifically, we let $g$ be the application of ``assessing a policy preference''
(see Table \ref{table: application}) between policies $\pi^1,\pi^2$ in a target
MDP without reward $\cM=\tuple{\cS,\cA,H,s_0,p}$, i.e., we aim to output a
scalar $x$ as close as possible to the difference in their expected returns:
\begin{align*}
    \cX_g=\RR,\qquad \cL_g(r,x)=\big| x -
    (J^{\pi^1}(r;p) - J^{\pi^2}(r;p)) \big|.
\end{align*}
In addition, we require that the set of feedback $\cF=\cF_{\text{D}}\cup
\cF_{\text{TC}} \cup\cF_{\text{PC}}$ contains only \emph{demonstrations}
$\cF_{\text{D}}$, \emph{trajectory comparisons} $\cF_{\text{TC}}$ and
\emph{policy comparisons} $\cF_{\text{PC}}$ feedback (see Table \ref{table:
feedback}), of the following kind.

We allow for $m_{\text{D}}\ge0$ demonstrations feedback
$\cF_{\text{D}}=\{f_{\text{D},i}\}_{i=1}^{m_{\text{D}}}$, where each
$f_{\text{D},i}$ is a ``$t_i$-suboptimal expert'' feedback (Table \ref{table:
feedback}) in some MDP without reward\footnote{All the considered environments
share the same $\cS,\cA,H$ because we work with rewards on this domain.}
$\cM_{\text{D},i}=\tuple{\cS,\cA,H,s_{0,\text{D},i},p_{\text{D},i}}$, with
$t_i\in[0,H]$, corresponding to the feasible set:
\begin{align}\label{eq: RF D}
    \cR_{f_{\text{D},i}}=\{r\in\fR:\,
  J^{\pi_{\text{D},i}}(r;p_{\text{D},i})\ge
  J^*(r;p_{\text{D},i})-t_i\},\qquad \forall i\in\dsb{m_{\text{D}}}.
\end{align}
Moreover, we allow for $m_{\text{TC}}\ge0$ trajectory comparison feedback
$\cF_{\text{TC}}=\{f_{\text{TC},i}\}_{i=1}^{m_{\text{TC}}}$, where each
$f_{\text{TC},i}$ is a ``hard preference'' feedback (Table \ref{table:
feedback}) in some MDP without reward
$\cM_{\text{TC},i}=\tuple{\cS,\cA,H,s_{0,\text{TC},i},p_{\text{TC},i}}$, with
feasible set:
\begin{align}\label{eq: RF TC}
    \cR_{f_{\text{TC},i}}=\{r\in\fR:\, G(\omega^1_{\text{TC},i};r)\le
  G(\omega^2_{\text{TC},i};r)\},\qquad \forall i\in\dsb{m_{\text{TC}}}.
 \end{align}
Finally, we allow for $m_{\text{PC}}\ge0$ policy comparison feedback
$\cF_{\text{PC}}=\{f_{\text{PC},i}\}_{i=1}^{m_{\text{PC}}}$, where each
$f_{\text{PC},i}$ is a ``hard preference'' feedback (Table \ref{table:
feedback}) in some MDP without reward
$\cM_{\text{PC},i}=\tuple{\cS,\cA,H,s_{0,\text{PC},i},p_{\text{PC},i}}$, with
feasible set:
  \begin{align}\label{eq: RF C}
    \cR_{f_{\text{PC},i}}=\{r\in\fR:\,
  J^{\pi^1_{\text{PC},i}}(r;p_{\text{PC},i})\le
  J^{\pi^2_{\text{PC},i}}(r;p_{\text{PC},i})\},\qquad \forall i\in\dsb{m_{\text{PC}}}.
  \end{align}

\paragraph{Finite data.}

To keep things realistic, we assume that the policies
$\pi^1,\pi^2,\pi_{\text{D},i}$, $\pi^1_{\text{PC},i}$, $\pi^2_{\text{PC},i}$ and
the transition models $p,p_{\text{D},i}$, $p_{\text{TC},i}$, $p_{\text{PC},i}$
are \emph{not} known and must instead be estimated from data.
We adopt a mixed offline-online setting that is common in the literature (e.g.,
see GAIL \cite{ho2016generativeadversarialimitationlearning}).
To estimate the policies $\pi^1,\pi^2,\pi_{\text{D},i}$, $\pi^1_{\text{PC},i}$,
$\pi^2_{\text{PC},i}$, we assume access to batch datasets of trajectories
$\cD^1,\cD^2,\cD_{\text{D},i}$, $\cD^1_{\text{PC},i}$, $\cD^2_{\text{PC},i}$
obtained by executing the policies in the corresponding environments
$\cM,\cM_{\text{D},i}$, $\cM_{\text{PC},i}$ for $n^1,n^2,n_{\text{D},i}$,
$n^1_{\text{PC},i}$, $n^2_{\text{PC},i}$ trajectories, respectively.
To estimate the transition models $p, p_{\text{D},i}$, $p_{\text{TC},i}$,
$p_{\text{PC},i}$, we assume access to a forward sampling model\footnote{A
\emph{forward model} \cite{dann2015episodic,kakade2003sample} of an MDP $\cM'$
permits to collect trajectories from $\cM'$ by exploring at will.} for each MDP
without reward $\cM, \cM_{\text{D},i}$, $\cM_{\text{TC},i}$,
$\cM_{\text{PC},i}$, from which we can collect $N, N_{\text{D},i}$,
$N_{\text{TC},i}$, $N_{\text{PC},i}$ trajectories, respectively.

\subsection{Rob-ReL}

We now present \rob (\roblong, Algorithm \ref{alg: rob}), a ReL algorithm for
solving this family of ReL problems using the robust approach from Section
\ref{sec: robustness}.
Specifically, in this setting, even with \emph{infinite} data available, the
robust approach (Eq. \ref{eq: minimax robust}) requires solving the following
optimization problem, where the constraints define the feasible set $\cR_\cF$:
\begin{align}\label{eq: opt prob rob}
        x_{\cF,g}\in
        \argmin\limits_{x'\in\RR}\max\limits_{r\in\fR}&\;
    | x' - \dotp{\popred{d^{\pi^1}}-\popred{d^{\pi^2}},r}|\;&&\\
    \text{s.t.: }& \max_\pi J^\pi(r;\popred{p_{\text{D},i}})-\dotp{\popred{d^{\pi_{\text{D},i}}},r}\le
    t_i&&\forall i\in\dsb{m_{\text{D}}},\nonumber\\
    & \dotp{d^{\omega^1_{\text{TC,i}}}-d^{\omega^2_{\text{TC,i}}}, r}\le 0
    &&\forall i\in\dsb{m_{\text{TC}}},\nonumber\\
    & \dotp{\popred{d^{\pi^1_{\text{PC},i}}}-\popred{d^{\pi^2_{\text{PC},i}}}, r}\le 0
    &&\forall i\in\dsb{m_{\text{PC}}}.\nonumber
\end{align}
However, with \emph{finite} data, the quantities highlighted in red are not
known. Therefore, \rob instead solves the optimization problem obtained by
replacing these quantities with their estimates.
The next two paragraphs describe these estimates are computed by \rob and the
optimization method it employs.

\paragraph{Estimation.}

Given a dataset
$\cD=\{\tuple{s_1^j,a_1^j,\dotsc,s_H^j,a_H^j,s_{H+1}^j}\}_{j\in\dsb{n}}$ of $n$
trajectories collected by some policy $\pi$, we can estimate the visit
distribution of $\pi$ at all $(s,a,h)\in\SAH$ as:
\begin{align}\label{eq: empirical estimates visit distributions}
    \widehat{d}_h^\pi(s,a)=\frac{1}{n}\sum\limits_{j\in\dsb{n}}
    \indic{s_h^j=s,a_h^j=a}.
\end{align}
In this way, \rob estimates $d^{\pi^1}$, $d^{\pi^2}$, $d^{\pi_{\text{D},i}}$,
$d^{\pi^1_{\text{PC},i}}$, $d^{\pi^2_{\text{PC},i}}$ from the corresponding
datasets $\cD^1,\cD^2,\cD_{\text{D},i}$, $\cD^1_{\text{PC},i}$,
$\cD^2_{\text{PC},i}$ (Line \ref{line: est visit distribs}).
Next, to estimate $p_{\text{D},i}$, \rob executes RF-Express
\cite{menard2021fast} (Line \ref{line: rfexpress}), a minimax-optimal
\emph{reward-free} exploration algorithm
\cite{jin2020RFE,lazzati2024compatibility}.
In short, RF-Express collects $N_{\text{D},i}$ trajectories from each
$\cM_{\text{D},i}$, and then uses the resulting data to estimate
$p_{\text{D},i}$.
Note that, since $p$, $p_{\text{TC},i}$, and $p_{\text{PC},i}$ do not appear in
Eq. \eqref{eq: opt prob rob}, then \rob does not need to estimate them.

\begin{figure}[!t]
  \centering
  \begin{minipage}[t!]{0.48\textwidth}
    \centering
    \input{subgradient.tex}
\end{minipage}
\hfill
\begin{minipage}[t!]{0.48\textwidth}
    \centering
    \vspace{-0.5pt}
    \input{pdsm_min.tex}
\end{minipage}
 \end{figure}

\paragraph{Optimization.}

Let $\widehat{\cR}_\cF$ and $\widehat{\cL}_g$ be the empirical counterparts of
$\cR_\cF$ and $\cL_g$ obtained by replacing the quantities in red in Eq.
\eqref{eq: opt prob rob} with the estimates described above. Then, \rob
addresses the optimization problem:
\begin{align}\label{eq: estimated opt prob rob}
    \widehat{x}_{\cF,g}\in \argmin\limits_{x'\in\RR}
    \max\limits_{r\in\widehat{\cR}_\cF}\widehat{\cL}_g(r,x').
\end{align}
This is a minimax problem with non-trivial constraints. 
Nevertheless, we can simplify it. Define $\widehat{M},\widehat{m}$ as the
largest and smallest values of
$\dotp{\widehat{d}^{\pi^1}-\widehat{d}^{\pi^2},r}$ over $\widehat{\cR}_\cF$:
\begin{align}\label{eq: est M m}
  &\widehat{M}\coloneqq \max\limits_{r\in\widehat{\cR}_\cF}
  \dotp{\widehat{d}^{\pi^1}-\widehat{d}^{\pi^2},r},
  \qquad \widehat{m}\coloneqq \min\limits_{r\in\widehat{\cR}_\cF}
  \dotp{\widehat{d}^{\pi^1}-\widehat{d}^{\pi^2},r}.
\end{align}
Then, we can rewrite Eq. \eqref{eq: estimated opt prob rob} in a more convenient
form:
\begin{restatable}{prop}{rewriteobjectivesusecase}\label{prop: rewrite
  objectives use case}%
  It holds that: $\widehat{x}_{\cF,g}=(\widehat{M}+\widehat{m})/2$.
\end{restatable}
As a result, instead of directly solving Eq. \eqref{eq: estimated opt prob rob},
\rob computes $\widehat{x}_{\cF,g}$ by first computing $\widehat{M}$ and
$\widehat{m}$ via Eq. \eqref{eq: est M m} (see Lines \ref{line: pdsm
min}-\ref{line: pdsm max}), and then combining the results through Proposition
\ref{prop: rewrite objectives use case} (see Line \ref{line: x}).
Observe that the optimization problems in Eq. \eqref{eq: est M m} are
\emph{convex}, since both the objective functions and constraints are linear
or convex (the pointwise maximum of linear functions is convex
\cite{boyd2004convex}).
To solve them, \rob finds saddle points of the Lagrangian function using the
\emph{primal-dual subgradient method} (PDSM) \cite{nedic2009subgradient}, which
alternates subgradient updates for the primal and dual variables.
Specifically, for any $r\in\fR$ and $\lambda\coloneqq
(\lambda_{\text{D}}^\intercal,\lambda_{\text{TC}}^\intercal,
\lambda_{\text{PC}}^\intercal)^\intercal\in\RR^{m_{\text{D}}+
m_{\text{TC}}+m_{\text{PC}}}$, the Lagrangian $\widehat{L}$ of both problems in
Eq. \eqref{eq: est M m} is:
\begin{align}\label{eq: estimated lagrangian}
   \widehat{L}(r,\lambda)&
  =\dotp{\widehat{d}^{\pi^1}-\widehat{d}^{\pi^2},r}
  +\sum\nolimits_{i\in\dsb{m_{\text{D}}}}\lambda_{\text{D}}^i
  \bigr{\max\nolimits_\pi J^\pi(r;\widehat{p}_{\text{D},i})-\dotp{\widehat{d}^{\pi_{\text{D},i}},r}
  -t_i}\\
  & \qquad+\sum\nolimits_{i\in\dsb{m_{\text{TC}}}}\lambda_{\text{TC}}^i\dotp{d^{\omega^1_{\text{TC,i}}}-d^{\omega^2_{\text{TC,i}}}, r}
  +\sum\nolimits_{i\in\dsb{m_{\text{PC}}}}\lambda_{\text{PC}}^i \dotp{\widehat{d}^{\pi^1_{\text{PC},i}}
  -\widehat{d}^{\pi^2_{\text{PC},i}}, r}.
  \nonumber
\end{align}
Then, the subroutines \texttt{PDSM-MIN} (Algorithm \ref{alg: pdsm min}) and
\texttt{PDSM-MAX} (Algorithm \ref{alg: pdsm max}, Appendix \ref{apx: proofs
rob}), invoked by \rob at Lines \ref{line: pdsm min}-\ref{line: pdsm max}, aim
to compute $\max_{\lambda\ge 0}\min_{r\in\fR} \widehat{L}(r,\lambda)$ and
$\min_{\lambda\le 0}\max_{r\in\fR} \widehat{L}(r,\lambda)$ by alternating
between one subgradient step for $r$ and one for $\lambda$.
Here, $\alpha>0$ is the step size, $s$ is a hyperparameter (see Theorem
\ref{thr: guarantees caty} for principled choices of $\alpha$ and $s$). The
expressions $\partial_r \widehat{L}(r_k,\lambda_k)$ and $\partial_\lambda
\widehat{L}(r_k,\lambda_k)$ denote subgradients of $\widehat{L}$ w.r.t. $r$ and
$\lambda$, evaluated at $(r_k,\lambda_k)$ (see Appendix \ref{apx: subgradients
lagrangian} for their formulas). To compute $\max\nolimits_\pi
J^\pi(r;\widehat{p}_{\text{D},i})$, both subroutines make use of the backward
induction algorithm \cite{puterman1994markov} (see Appendix \ref{apx:
subgradients lagrangian}).

Finally, recall from Section \ref{sec: robustness} that we are also interested
in quantifying the worst-case loss $\cI_{\cF,g}$ that might be incurred when
solving problem $(\cF,g)$. This is computed as $\widehat{\cI}_K$ by \rob at Line
\ref{line: uninf} based on the following result:
\begin{restatable}{prop}{uninf}\label{prop: uninf}%
  It holds that: $\widehat{\cI}_{\cF,g}\coloneqq
    \max_{r\in\widehat{\cR}_\cF}\widehat{\cL}_g(r,\widehat{x}_{\cF,g}) =
    (\widehat{M}-\widehat{m})/2$.
\end{restatable}

\subsection{Theoretical Analysis}

We now show that \rob is both computationally and sample efficient.
To this aim, we make the assumption that the feasible set $\cR_\cF$ contains a
\emph{strictly} feasible reward $\overline{r}$, which is common in both the
optimization \citep{nedic2009subgradient} and the RL \citep{ding2020npg}
literature:
\begin{ass}[Slater's condition]\label{ass: slater condition}%
  There exist $\xi>0$ and $\overline{r}\in\fR$ such that:
  \begin{align*}
    \begin{cases}
      \max_\pi J^\pi(\overline{r};p_{\text{D},i})-\dotp{d^{\pi_{\text{D},i}},\overline{r}}
  -t_i\le-\xi & \forall i\in\dsb{m_{\text{D}}}\\
  \dotp{d^{\omega^1_{\text{TC,i}}}-d^{\omega^2_{\text{TC,i}}}, \overline{r}}\le-\xi & \forall i\in\dsb{m_{\text{TC}}}\\
  \dotp{d^{\pi^1_{\text{PC},i}}
  -d^{\pi^2_{\text{PC},i}}, \overline{r}}\le -\xi & \forall i\in\dsb{m_{\text{PC}}}
    \end{cases}.
  \end{align*}
\end{ass}
Then, we can prove the following result for \rob:
\begin{restatable}{thr}{thrcaty}
  \label{thr: guarantees caty}
Let $\tuple{\cF,g}$ be a ReL problem as described in Section \ref{sec: family
problems} for which Assumption \ref{ass: slater condition} holds.
Let $\epsilon\in(0,2H]$ and $\delta\in(0,1)$.
If we set $s=4H/\xi+\sqrt{(4H/\xi)^2+SAH/4}$ and $\alpha=\epsilon/(16H
(1+s\sqrt{m_{\text{D}}+m_{\text{TC}}+m_{\text{PC}}})^2)$, then, with probability
$1-\delta$, \rob satisfies:
\begin{align*}
  &\scalebox{0.96}{$  \displaystyle \cL_g(r^\star, \widehat{x}_K)\le
  \cI_{\cF,g}+\epsilon
  \quad \text{and}\quad
  \big|\cI_{\cF,g}-\widehat{\cI}_K\big|\le \epsilon,$}
\end{align*}
with a number of samples:
\begin{align*}
  & \scalebox{0.96}{$  \displaystyle n^1,n^2,n_{\text{D},i},n^1_{\text{PC},i}, n^2_{\text{PC},i}\le \widetilde{\cO}\Big(
    \frac{SAH^5}{\epsilon^2\xi^2} \log\frac{1}{\delta}  
  \Big),$}\\
  & \scalebox{0.96}{$  \displaystyle
  N_{\text{D},i}\le \widetilde{\cO}\Big(
    \frac{SAH^5}{\epsilon^2\xi^2}\Big(S+\log\frac{1}{\delta}\Big)
    \Big),\;  N,N_{\text{TC},i},N_{\text{PC},i}=0,$}
\end{align*}
and a number of iterations:
\begin{align*}
  \scalebox{0.93}{$  \displaystyle
  K\le\cO\Biggr{
    \frac{H^{5/2}}{\xi\epsilon^2}\biggr{\sqrt{SA}+\frac{\sqrt{H}}{\xi}}
    \biggr{1+\sqrt{H(m_{\text{D}}+m_{\text{TC}}+m_{\text{PC}})\Bigr{\frac{H}{\xi^2}+SA}}}^2
  }.$}
\end{align*}
  \end{restatable}
Simply put, Theorem \ref{thr: guarantees caty} tells us that \rob enjoys sample
and iteration complexities that are polynomial in the quantities of interest
$S,A,H,\frac{1}{\epsilon},\log\frac{1}{\delta}, \frac{1}{\xi},
m_{\text{D}},m_{\text{TC}}, m_{\text{PC}}$.
In Appendix \ref{apx: other kinds feedback}, we show how to extend \rob to other
forms of applications and feedback.
Finally, we observe that \rob can also be used for estimating the worst-case
loss $\max_{r\in\cR_\cF}\cL_g(r,x)$ incurred by deploying an arbitrary object
$x\in\cX_g$ (see Appendix \ref{apx: bounding worst case error with rob} for
details).
\begin{proofsketch}
Define $M\coloneqq \max_{r\in\cR_\cF}\dotp{d^{\pi^1}-d^{\pi^2},r}$ and
$m\coloneqq \min_{r\in\cR_\cF}\dotp{d^{\pi^1}-d^{\pi^2},r}$. Then, after having
shown that $x_{\cF,g}=(M+m)/2$ and $\cI_{\cF,g}=(M-m)/2$, the result can be
proved by upper bounding the \emph{estimation error}
$|M-\widehat{M}|+|m-\widehat{m}|$ and the \emph{iteration error}
$|\widehat{M}-\widehat{M}_K|+|\widehat{m}-\widehat{m}_K|$.
We upper bound the estimation error in Lemma \ref{lemma: estimation error}, by
first bounding the error in estimating the visitation distribution of the
policies (using Hoeffding's inequality) and the transition models (using the
results in \cite{menard2021fast}), and then showing that, under Assumption
\ref{ass: slater condition}, all $M,\widehat{M},m,\widehat{m}$ are saddle points
with bounded optimal Lagrange multipliers (using Lemma 3 of
\cite{nedic2009subgradient}).
Regarding the iteration error (see Lemma \ref{lemma: approximation error}), we
exploit the theoretical guarantees of the PDSM (Proposition 2 of
\cite{nedic2009subgradient}).
%
\end{proofsketch}

\section{Numerical Simulations}\label{sec: numerical simulations}

\begin{figure}[!t]
  \centering
  \begin{minipage}[t!]{0.47\textwidth}
    \centering
    \includegraphics[width=0.95\linewidth]{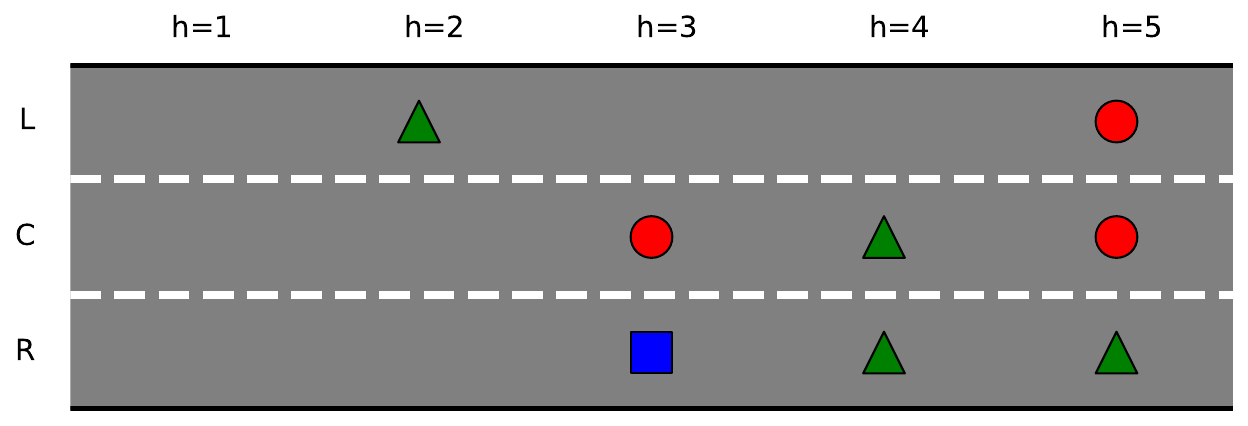}
\end{minipage}
\begin{minipage}[t!]{0.47\textwidth}
    \centering
    \includegraphics[width=0.95\linewidth]{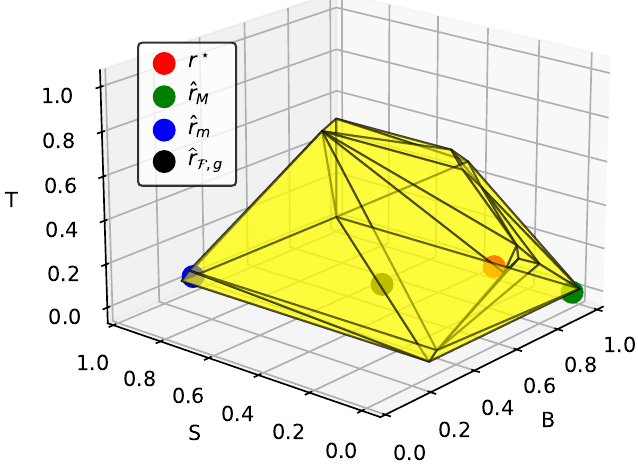}
\end{minipage}
\caption{(Left) The target environment considered in the
experiment.
(Right) The feasible set with $r^\star$ and
$\widehat{r},\widehat{r}_{M}, \widehat{r}_{m}$. Axis $B,S,T$ refer to
the reward values of B, S and T.}
\label{fig: road and fs}
\end{figure}

In this section, we present an illustrative experiment aimed at $(i)$
exemplifying the class of problems solvable by \rob and $(ii)$ providing a
\emph{graphical} intuition of the robust approach.
%
For these reasons, we focus on a low-dimensional problem.
More details are in Appendix \ref{apx: experimental details}.
  
\paragraph{Environment.}

We use the following state-action space $\tuple{\cS,\cA,H}$ to model the road
depicted on the left of Figure \ref{fig: road and fs}.
There are 12 states $s\in\cS$ obtained by combining three lanes (left (L),
center (C), right (R)) with four items (a ball (B), a square (S), a triangle
(T), or nothing (N)).
The action space contains three actions $\cA=\{a_L,a_C,a_R\}$, that aim to bring
an agent to, respectively, the left, keep the current lane, or go to the right,
respectively.
The horizon is $H=5$.

\paragraph{ReL problem.}

We consider the ReL problem $\tuple{\cF,g}$ in which an agent, \texttt{Alice}
(e.g., a human), has a preference over the items B, S and T, that can be modeled
through a three-dimensional reward $r^\star=[r^\star_B, r^\star_S, r^\star_T]$,
where each component represents the value that \texttt{Alice} associates with
visiting states containing items B, S and T ($r^\star$ is zero for N).
$r^\star$ is unknown to our learner, and for the experiment we set
$r^\star=[0.7,0.1,0.2]$.
The application $g$ consists of assessing how much \texttt{Alice} prefers the
policy $\pi^1$, that always takes action $a_R$, w.r.t. policy $\pi^2$, that
always selects $a_L$, in an environment constructed using $\tuple{\cS,\cA,H}$
and some dynamics $s_0$, $p$ (see Appendix \ref{apx: experimental details}).
This value is $\Delta J(r^\star)\coloneqq J^{\pi^1}(r^\star;p)
-J^{\pi^2}(r^\star;p)=0.39$, i.e., \texttt{Alice}'s preference for $\pi^1$ over
$\pi^2$ has an ``intensity'' of 0.39.
To simulate some feedback $\cF$ from \texttt{Alice}, we randomly generated some
demonstrations, policy comparisons and trajectory comparisons feedback
consistent with $r^\star$ ($1+2+3=6$ feedback in total).
The resulting feasible set $\cR_\cF$, along with $r^\star$, is plotted in Figure
\ref{fig: road and fs} (right).

\paragraph{Simulation and results.}

Our goal is to output a scalar $x\in\RR$ that is as close as possible to $\Delta
J(r^\star)=0.39$ without knowing $r^\star$, but only that $r^\star$ belongs to
the yellow set $\cR_\cF$ in Figure \ref{fig: road and fs}.
We have executed \rob for $K=1200$ iterations with a step size $\alpha=0.01$
using for simplicity the exact values of policies and transition models,
obtaining $\widehat{m}_K=-0.62,\widehat{M}_K=1.02$, corresponding to the rewards
in $\cR_\cF$ that provide the smallest and largest values of $\Delta J$ (see the
blue $\widehat{r}_{m}$ and green $\widehat{r}_{M}$ rewards in Figure \ref{fig:
road and fs}).
The output of \rob is, therefore, $\widehat{x}_{\cF,g}=0.2$ (corresponding to the
black reward $\widehat{r}_{\cF,g}$ in Figure \ref{fig: road and fs}) and
$\widehat{\cI}_{\cF,g}=0.82$.
Thus, we know that the distance between $\widehat{x}_{\cF,g}=0.2$ and $\Delta
J(r^\star)=0.39$ is at most $\widehat{\cI}_{\cF,g}=0.82$ in the worst case, and
that it can be reduced by collecting more feedback.
Moreover, looking at these numbers, we realize that, if our robust approach had
not been adopted, i.e., if an arbitrary reward $r$ in $\cR_\cF$ had been used
for prediction, then \emph{the worst-case error might have been doubled} to
approximately $1.64$ (e.g., if $r=\widehat{r}_{M}$ and
$r^\star=\widehat{r}_{m}$).

\section{Conclusion}\label{sec: conclusion}

In this paper, we presented a unifying and quantitative framework for studying
ReL problems. We then introduced a principled approach for solving ReL problems
in general and we described \rob, an algorithm tailored to a specific subset of
ReL problems that uses our approach.

\paragraph{Limitations and future work.}

The main limitation of this work is that the proposed algorithm, \rob, does not
address all ReL problems. Therefore, we believe that future research should
focus on developing additional algorithms that adopt the robust approach to
tackle settings not covered by \rob. We hope this will facilitate solving ReL
problems in real-world scenarios.

\bibliographystyle{plain}
\bibliography{refs.bib}

\begin{thebibliography}{10}

\bibitem{abbeel2004apprenticeship}
Pieter Abbeel and Andrew~Y. Ng.
\newblock Apprenticeship learning via inverse reinforcement learning.
\newblock In {\em International Conference on Machine Learning 21 (ICML)}, 2004.

\bibitem{alimov2019chebyshev}
A.~R. Alimov and I.~G. Tsar'kov.
\newblock Chebyshev centres, jung constants, and their applications.
\newblock {\em Russian Mathematical Surveys}, 74, 2019.

\bibitem{amin2016resolving}
Kareem Amin and Satinder Singh.
\newblock Towards resolving unidentifiability in inverse reinforcement learning, 2016.

\bibitem{arora2020survey}
Saurabh Arora and Prashant Doshi.
\newblock A survey of inverse reinforcement learning: Challenges, methods and progress.
\newblock {\em Artificial Intelligence}, 297:103500, 2021.

\bibitem{boularias2011relative}
Abdeslam Boularias, Jens Kober, and Jan Peters.
\newblock Relative entropy inverse reinforcement learning.
\newblock In {\em International Conference on Artificial Intelligence and Statistics 14 (AISTATS 2011)}, pages 182--189, 2011.

\bibitem{boyd2022subgradients}
Stephen Boyd, Jhon Duchi, Mert Pilanci, and Lieven Vandenberghe.
\newblock Subgradients, 2022.
\newblock Notes for EE364b, Stanford University.

\bibitem{boyd2004convex}
Stephen Boyd and Lieven Vandenberghe.
\newblock {\em Convex optimization}.
\newblock Cambridge university press, 2004.

\bibitem{brown2020bayesianrobust}
Daniel Brown, Scott Niekum, and Marek Petrik.
\newblock Bayesian robust optimization for imitation learning.
\newblock In {\em Advances in Neural Information Processing Systems 33 (NeurIPS)}, pages 2479--2491, 2020.

\bibitem{cao2021identifiability}
Haoyang Cao, Samuel Cohen, and Lukasz Szpruch.
\newblock Identifiability in inverse reinforcement learning.
\newblock In {\em Advances in Neural Information Processing Systems 34 (NeurIPS)}, pages 12362--12373, 2021.

\bibitem{cheng2024rime}
Jie Cheng, Gang Xiong, Xingyuan Dai, Qinghai Miao, Yisheng Lv, and Fei-Yue Wang.
\newblock Rime: Robust preference-based reinforcement learning with noisy preferences.
\newblock In {\em International Conference on Machine Learning 41 (ICML)}, 2024.

\bibitem{christiano2017deep}
Paul~F Christiano, Jan Leike, Tom Brown, Miljan Martic, Shane Legg, and Dario Amodei.
\newblock Deep reinforcement learning from human preferences.
\newblock In {\em Advances in Neural Information Processing Systems 30 (NeurIPS)}, 2017.

\bibitem{dabbene2014probabilistic}
Fabrizio Dabbene, Mario Sznaier, and Roberto Tempo.
\newblock Probabilistic optimal estimation with uniformly distributed noise.
\newblock {\em IEEE Transactions on Automatic Control}, 59:2113--2127, 2014.

\bibitem{dann2015episodic}
Christoph Dann and Emma Brunskill.
\newblock Sample complexity of episodic fixed-horizon reinforcement learning.
\newblock In {\em Advances in Neural Information Processing Systems 28 (NeurIPS)}, 2015.

\bibitem{danzer1963helly}
L.~Danzer, B.~Gr{\"u}nbaum, and V.~Klee.
\newblock {\em Helly's Theorem and Its Relatives}.
\newblock Proceedings of symposia in pure mathematics: Convexity. American Mathematical Society, 1963.

\bibitem{ding2020npg}
Dongsheng Ding, Kaiqing Zhang, Tamer Basar, and Mihailo Jovanovic.
\newblock Natural policy gradient primal-dual method for constrained markov decision processes.
\newblock In {\em Advances in Neural Information Processing Systems 33 (NeurIPS)}, pages 8378--8390, 2020.

\bibitem{finn2016guided}
Chelsea Finn, Sergey Levine, and Pieter Abbeel.
\newblock Guided cost learning: Deep inverse optimal control via policy optimization.
\newblock In {\em International Conference on Machine Learning 33 (ICML)}, pages 49--58, 2016.

\bibitem{Fu2017LearningRR}
Justin Fu, Katie Luo, and Sergey Levine.
\newblock Learning robust rewards with adversarial inverse reinforcement learning.
\newblock In {\em International Conference on Learning Representations 5 (ICLR)}, 2017.

\bibitem{haarnoja2017rldeepenergypolicies}
Tuomas Haarnoja, Haoran Tang, Pieter Abbeel, and Sergey Levine.
\newblock Reinforcement learning with deep energy-based policies.
\newblock In {\em International Conference on Machine Learning 34 (ICML)}, volume~70, pages 1352--1361, 2017.

\bibitem{hadfield2017inverse}
Dylan Hadfield-Menell, Smitha Milli, Pieter Abbeel, Stuart~J Russell, and Anca Dragan.
\newblock Inverse reward design.
\newblock In {\em Advances in Neural Information Processing Systems 30 (NeurIPS)}, 2017.

\bibitem{hadfield2016cirl}
Dylan Hadfield-Menell, Stuart~J Russell, Pieter Abbeel, and Anca Dragan.
\newblock Cooperative inverse reinforcement learning.
\newblock In {\em Advances in Neural Information Processing Systems 29 (NeurIPS)}, 2016.

\bibitem{ho2016generativeadversarialimitationlearning}
Jonathan Ho and Stefano Ermon.
\newblock Generative adversarial imitation learning.
\newblock In {\em Advances in Neural Information Processing Systems 29 (NeurIPS)}, 2016.

\bibitem{ho2016modelfreeimitationlearningpolicy}
Jonathan Ho, Jayesh~K. Gupta, and Stefano Ermon.
\newblock Model-free imitation learning with policy optimization, 2016.

\bibitem{huang2018learning}
Jessie Huang, Fa~Wu, Doina Precup, and Yang Cai.
\newblock Learning safe policies with expert guidance.
\newblock In {\em Advances in Neural Information Processing Systems 31 (NeurIPS)}, 2018.

\bibitem{ibarz2018rewardlearning}
Borja Ibarz, Jan Leike, Tobias Pohlen, Geoffrey Irving, Shane Legg, and Dario Amodei.
\newblock Reward learning from human preferences and demonstrations in atari.
\newblock In {\em Advances in Neural Information Processing Systems 32 (NeurIPS)}, pages 8022--8034, 2018.

\bibitem{jeon2020rewardrational}
Hong~Jun Jeon, Smitha Milli, and Anca Dragan.
\newblock Reward-rational (implicit) choice: A unifying formalism for reward learning.
\newblock In {\em Advances in Neural Information Processing Systems 33 (NeurIPS)}, pages 4415--4426, 2020.

\bibitem{jin2020RFE}
Chi Jin, Akshay Krishnamurthy, Max Simchowitz, and Tiancheng Yu.
\newblock Reward-free exploration for reinforcement learning.
\newblock In {\em International Conference on Machine Learning 37 (ICML)}, volume 119, pages 4870--4879, 2020.

\bibitem{jin2020provablyefficient}
Chi Jin, Zhuoran Yang, Zhaoran Wang, and Michael~I Jordan.
\newblock Provably efficient reinforcement learning with linear function approximation.
\newblock In {\em Conference on Learning Theory 33 (COLT)}, volume 125, pages 2137--2143, 2020.

\bibitem{jung1901ueber}
Heinrich Jung.
\newblock Ueber die kleinste kugel, die eine räumliche figur einschliesst.
\newblock {\em Journal für die reine und angewandte Mathematik}, 123:241--257, 1901.

\bibitem{kakade2003sample}
Sham~Machandranath Kakade.
\newblock {\em On the sample complexity of reinforcement learning}.
\newblock University of London, University College London (United Kingdom), 2003.

\bibitem{kaufmann2024surveyreinforcementlearninghuman}
Timo Kaufmann, Paul Weng, Viktor Bengs, and Eyke Hüllermeier.
\newblock A survey of reinforcement learning from human feedback, 2024.

\bibitem{kim2021rewardidentification}
Kuno Kim, Shivam Garg, Kirankumar Shiragur, and Stefano Ermon.
\newblock Reward identification in inverse reinforcement learning.
\newblock In {\em International Conference on Machine Learning 38 (ICML)}, pages 5496--5505, 2021.

\bibitem{lacotte2019riskil}
Jonathan Lacotte, Mohammad Ghavamzadeh, Yinlam Chow, and Marco Pavone.
\newblock Risk-sensitive generative adversarial imitation learning.
\newblock In {\em International Conference on Artificial Intelligence and Statistics 22 (AISTATS)}, volume~89, pages 2154--2163, 2019.

\bibitem{lang2024aideceive}
Leon Lang, Davis Foote, Stuart Russell, Anca Dragan, Erik Jenner, and Scott Emmons.
\newblock When your ais deceive you: Challenges of partial observability in reinforcement learning from human feedback.
\newblock In {\em Advances in Neural Information Processing Systems 37 (NeurIPS)}, pages 93240--93299, 2024.

\bibitem{lazzati2024utility}
Filippo Lazzati and Alberto~Maria Metelli.
\newblock Learning utilities from demonstrations in markov decision processes, 2024.

\bibitem{lazzati2024compatibility}
Filippo Lazzati, Mirco Mutti, and Alberto~Maria Metelli.
\newblock How does inverse rl scale to large state spaces? a provably efficient approach.
\newblock In {\em Advances in Neural Information Processing Systems 37 (NeurIPS)}, pages 54820--54871, 2024.

\bibitem{lazzati2024offline}
Filippo Lazzati, Mirco Mutti, and Alberto~Maria Metelli.
\newblock Offline inverse rl: New solution concepts and provably efficient algorithms.
\newblock In {\em International Conference on Machine Learning 41 (ICML)}, 2024.

\bibitem{lindner2022active}
David Lindner, Andreas Krause, and Giorgia Ramponi.
\newblock Active exploration for inverse reinforcement learning.
\newblock In {\em Advances in Neural Information Processing Systems 35 (NeurIPS)}, pages 5843--5853, 2022.

\bibitem{lopes2009activelearning}
Manuel Lopes, Francisco Melo, and Luis Montesano.
\newblock Active learning for reward estimation in inverse reinforcement learning.
\newblock In {\em Machine Learning and Knowledge Discovery in Databases (ECML PKDD)}, pages 31--46, 2009.

\bibitem{menard2021fast}
Pierre Menard, Omar~Darwiche Domingues, Anders Jonsson, Emilie Kaufmann, Edouard Leurent, and Michal Valko.
\newblock Fast active learning for pure exploration in reinforcement learning.
\newblock In {\em International Conference on Machine Learning 38 (ICML)}, volume 139, pages 7599--7608, 2021.

\bibitem{metelli2023towards}
Alberto~Maria Metelli, Filippo Lazzati, and Marcello Restelli.
\newblock Towards theoretical understanding of inverse reinforcement learning.
\newblock In {\em International Conference on Machine Learning 40 (ICML)}, pages 24555--24591, 2023.

\bibitem{metelli2021provably}
Alberto~Maria Metelli, Giorgia Ramponi, Alessandro Concetti, and Marcello Restelli.
\newblock Provably efficient learning of transferable rewards.
\newblock In {\em International Conference on Machine Learning 38 (ICML)}, volume 139, pages 7665--7676, 2021.

\bibitem{nedic2009subgradient}
A.~Nedi\'{c} and A.~Ozdaglar.
\newblock Subgradient methods for saddle-point problems.
\newblock {\em Journal of Optimization Theory and Applications}, 142:205--228, 2009.

\bibitem{ng2000algorithms}
Andrew~Y. Ng and Stuart~J. Russell.
\newblock Algorithms for inverse reinforcement learning.
\newblock In {\em International Conference on Machine Learning 17 (ICML 2000)}, pages 663--670, 2000.

\bibitem{osa2018IL}
Takayuki Osa, Joni Pajarinen, Gerhard Neumann, J.~Andrew Bagnell, Pieter Abbeel, and Jan Peters.
\newblock An algorithmic perspective on imitation learning.
\newblock {\em Foundations and Trends® in Robotics}, 7(1-2):1--179, 2018.

\bibitem{poiani2024inversereinforcementlearningsuboptimal}
Riccardo Poiani, Gabriele Curti, Alberto~Maria Metelli, and Marcello Restelli.
\newblock Inverse reinforcement learning with sub-optimal experts, 2024.

\bibitem{puterman1994markov}
Martin~Lee Puterman.
\newblock {\em {M}arkov Decision Processes: Discrete Stochastic Dynamic Programming}.
\newblock John Wiley \& Sons, Inc., 1994.

\bibitem{Ramachandran2007birl}
Deepak Ramachandran and Eyal Amir.
\newblock Bayesian inverse reinforcement learning.
\newblock In {\em International Joint Conference on Artifical Intelligence 20 (IJCAI)}, pages 2586--2591, 2007.

\bibitem{ratliff2006maximum}
Nathan~D. Ratliff, J.~Andrew Bagnell, and Martin~A. Zinkevich.
\newblock Maximum margin planning.
\newblock In {\em International Conference on Machine Learning 23 (ICML 2006)}, pages 729--736, 2006.

\bibitem{rolland2022identifiability}
Paul Rolland, Luca Viano, Norman Sch\"{u}rhoff, Boris Nikolov, and Volkan Cevher.
\newblock Identifiability and generalizability from multiple experts in inverse reinforcement learning.
\newblock In {\em Advances in Neural Information Processing Systems 35 (NeurIPS)}, pages 550--564, 2022.

\bibitem{russell1998learning}
Stuart Russell.
\newblock Learning agents for uncertain environments (extended abstract).
\newblock In {\em Proceedings of the Eleventh Annual Conference on Computational Learning Theory 11 (COLT)}, pages 101--103, 1998.

\bibitem{schlaginhaufen2023identifiability}
Andreas Schlaginhaufen and Maryam Kamgarpour.
\newblock Identifiability and generalizability in constrained inverse reinforcement learning.
\newblock In {\em International Conference on Machine Learning 40 (ICML)}, volume 202, pages 30224--30251, 2023.

\bibitem{schlaginhaufen2024transferability}
Andreas Schlaginhaufen and Maryam Kamgarpour.
\newblock Towards the transferability of rewards recovered via regularized inverse reinforcement learning, 2024.

\bibitem{scott1991extension}
P.R. Scott.
\newblock An extension of jung's theorem.
\newblock {\em The Quarterly Journal of Mathematics}, 42:209--212, 1991.

\bibitem{Shani2021OnlineAL}
Lior Shani, Tom Zahavy, and Shie Mannor.
\newblock Online apprenticeship learning.
\newblock In {\em AAAI Conference on Artificial Intelligence 36 (AAAI)}, pages 8240--8248, 2022.

\bibitem{skalse2024quantifyingsensitivityinversereinforcement}
Joar Skalse and Alessandro Abate.
\newblock Quantifying the sensitivity of inverse reinforcement learning to misspecification.
\newblock In {\em International Conference on Learning Representations 12 (ICLR)}, 2024.

\bibitem{skalse2023invariancepolicyoptimisationpartial}
Joar Skalse, Matthew Farrugia-Roberts, Stuart Russell, Alessandro Abate, and Adam Gleave.
\newblock Invariance in policy optimisation and partial identifiability in reward learning.
\newblock In {\em International Conference on Machine Learning 40 (ICML)}, volume 202, pages 32033--32058, 2023.

\bibitem{syed2007game}
Umar Syed and Robert~E Schapire.
\newblock A game-theoretic approach to apprenticeship learning.
\newblock In {\em Advances in Neural Information Processing System 20 (NeurIPS)}, 2007.

\bibitem{viano2021robust}
Luca Viano, Yu-Ting Huang, Parameswaran Kamalaruban, Adrian Weller, and Volkan Cevher.
\newblock Robust inverse reinforcement learning under transition dynamics mismatch.
\newblock In {\em Advances in Neural Information Processing Systems 34 (NeurIPS)}, pages 25917--25931, 2021.

\bibitem{wiesemann2013robust}
Wolfram Wiesemann, Daniel Kuhn, and Breç Rustem.
\newblock Robust markov decision processes.
\newblock {\em Mathematics of Operations Research}, 38(1):153--183, 2013.

\bibitem{wirth2017pbrl}
Christian Wirth, Riad Akrour, Gerhard Neumann, and Johannes F\"{u}rnkranz.
\newblock A survey of preference-based reinforcement learning methods.
\newblock {\em Journal of Machine Learning Research}, 18:4945--4990, 2017.

\bibitem{wulfmeier2016maximumentropydeepirl}
Markus Wulfmeier, Peter Ondruska, and Ingmar Posner.
\newblock Maximum entropy deep inverse reinforcement learning, 2016.

\bibitem{zhao2023inverse}
Lei Zhao, Mengdi Wang, and Yu~Bai.
\newblock Is inverse reinforcement learning harder than standard reinforcement learning?
\newblock In {\em International Conference on Machine Learning 41 (ICML)}, 2024.

\bibitem{zhu2023principledRLHF}
Banghua Zhu, Michael Jordan, and Jiantao Jiao.
\newblock Principled reinforcement learning with human feedback from pairwise or k-wise comparisons.
\newblock In {\em International Conference on Machine Learning 40 (ICML)}, volume 202, pages 43037--43067, 2023.

\bibitem{Ziebart2010ModelingPA}
Brian~D. Ziebart.
\newblock Modeling purposeful adaptive behavior with the principle of maximum causal entropy, 2010.

\bibitem{ziebart2008maxent}
Brian~D. Ziebart, Andrew Maas, J.~Andrew Bagnell, and Anind~K. Dey.
\newblock Maximum entropy inverse reinforcement learning.
\newblock In {\em AAAI Conference on Artificial Intelligence 23 (AAAI)}, volume~3, pages 1433--1438, 2008.

\end{thebibliography}

\appendix

\newpage

\section{Additional Related Work}\label{apx: additional rel work}

In this appendix, we provide a comprehensive presentation of the main related
work of this paper.
We group the related work into four categories: papers that ``ignore'' partial
identifiability because it does not create issues in the applications that they
consider, papers that explicitly address partial identifiability by looking for
sufficient conditions that guarantee that it does not create issues, papers that
aim to estimate the feasible set in a provably-efficient manner, and a
miscellaneous of other papers.
Next, in Appendix \ref{apx: note IL}, we mention that some IL algorithms can be
seen as adopting our robust approach.
Finally, in Appendix \ref{apx: comparison with skalse}, we provide a thorough
comparison of our \emph{quantitative} framework introduced in Section \ref{sec:
rel problem formulation} with the \emph{qualitative} framework provided by
\cite{skalse2023invariancepolicyoptimisationpartial}.

\paragraph{Works that ``ignore'' partial identifiability.}

The partial identifiability in ReL is a topic that dates back to the seminal
works of \cite{russell1998learning,ng2000algorithms}. However, in many existing
ReL works
\cite{ng2000algorithms,ratliff2006maximum,ziebart2008maxent,Ziebart2010ModelingPA,boularias2011relative,wulfmeier2016maximumentropydeepirl,finn2016guided,christiano2017deep},
partial identifiability is not considered. The reason is that these works
consider ReL problems $\tuple{\cF,g}$ where partial identifiability does not
create issues, because all the rewards in the feasible set $\cR_\cF$ behave as
the unknown target reward $r^\star$ w.r.t. application $g$. 
Many of the mentioned papers are IRL works where the application $g$ is IL, and
the set of feedback $\cF=\{f\}$ contains a single demonstrations feedback $f$
able to provide enough information for IL.
In particular, \cite{ng2000algorithms,ratliff2006maximum} introduce the
heuristic of ``margin maximization'' to extract a reward from the feasible set
that facilitates the IL task, and then perform planning on it to find the
imitating policy to deploy. \cite{ziebart2008maxent,Ziebart2010ModelingPA}
introduce new ReL feedback consisting in demonstrations collected from a maximum
(causal) entropy policy and present an algorithm that recovers an arbitrary
reward from the corresponding feasible set (the authors let the inner
optimization algorithm break the ties).
\cite{boularias2011relative,wulfmeier2016maximumentropydeepirl,finn2016guided}
adopt the model of \cite{ziebart2008maxent} and extend its algorithm by,
respectively, adopting a model-free approach, using neural networks for
parameterizing the reward function to learn, and approximating the objective to
speed-up the learning process under unknown dynamics in high-dimensional
continuous systems.
Finally, we mention \cite{christiano2017deep} that, in the context of PbRL,
adopts a maximum likelihood approach to recover an arbitrary reward from the
corresponding feasible set (also \cite{christiano2017deep} let the inner
optimization algorithm break the ties) from finite data.

\paragraph{Works on partial identifiability.}

When the aforementioned algorithms have been applied to other applications
(i.e., by using the corresponding recovered reward function), then partial
identifiability turned out to be a problem, as observed empirically for instance
by \cite{finn2016guided}, who tried to transfer the recovered reward to a new
environment. Consequently, many works addressing the identifiability problem
have appeared
\cite{amin2016resolving,Fu2017LearningRR,cao2021identifiability,kim2021rewardidentification,viano2021robust,rolland2022identifiability,skalse2023invariancepolicyoptimisationpartial,schlaginhaufen2023identifiability,schlaginhaufen2024transferability,lang2024aideceive}.
In particular, \cite{amin2016resolving} consider access to demonstrations
feedback from multiple environments to reduce the size of the feasible set.
Similarly, \cite{cao2021identifiability} consider demonstrations feedback from
multiple environments or using multiple discount factors to improve the
identifiability of the target reward, while \cite{kim2021rewardidentification}
focuses on the properties of the considered MDPs.
\cite{Fu2017LearningRR} assume that the target reward is state-only and that the
environment satisfies certain additional properties to guarantee that the
feasible set contains only rewards that can be transferred to new environments
in the same way as the true target reward. \cite{viano2021robust} propose a
robust approach for the reward transfer application that is rather different
from ours. Specifically, the authors propose to deploy the policy that minimizes
the loss w.r.t. the worst possible transition model in a certain rectangular
uncertainty set \cite{wiesemann2013robust} centered in the transition model of
the target environment. \cite{rolland2022identifiability} studies the
identifiability problem using as feedback demonstrations from multiple experts,
and focuses on the application of reward transfer.
\cite{schlaginhaufen2023identifiability,schlaginhaufen2024transferability}
provide a similar study but use, respectively, different kinds of feedback and a
different method for studying the similarity of the environments.
\cite{skalse2023invariancepolicyoptimisationpartial} present a general study of
the identifiability problem in both IRL and PbRL, and we provide a detailed
analysis of this work in Appendix \ref{apx: comparison with skalse}. Finally,
\cite{lang2024aideceive} study the identifiability problem in PbRL assuming
that the feedback is based only on partial observations of the environment.

\paragraph{Works that estimate the feasible set.}

A recent line of research has focused on studying the sample complexity of
estimating the feasible set
\cite{metelli2021provably,lindner2022active,metelli2023towards,zhao2023inverse,lazzati2024offline,lazzati2024compatibility,poiani2024inversereinforcementlearningsuboptimal}.
All these papers consider the IRL problem settings where the feedback consists
in demonstrations from an optimal or $\epsilon$-optimal
\cite{poiani2024inversereinforcementlearningsuboptimal} expert's policy.
Specifically, \cite{metelli2021provably} is the seminal work in this context. It
assumes availability of a generative sampling model of the environment in the
tabular setting, and provides an upper bound to the sample complexity for
estimating the feasible set.
This result is completed by \cite{metelli2023towards} who present a lower bound
in a similar setting.
\cite{lindner2022active} and \cite{zhao2023inverse} focus
on a forward sampling model in tabular MDPs, and provide upper bounds to the
sample complexity.
\cite{lazzati2024compatibility} extend these theoretical results to the Linear
MDPs \cite{jin2020provablyefficient} setting.
Instead, \cite{lazzati2024compatibility} and \cite{zhao2023inverse} consider the
offline setting where only batch datasets of demonstrations are available, and
provide results under some concentrability assumptions.
\cite{poiani2024inversereinforcementlearningsuboptimal} present both lower and
upper bounds for estimating the feasible set in tabular settings assuming a
suboptimal expert.

\paragraph{Others.}

\cite{Ramachandran2007birl} adopt a Bayesian approach in the context of IRL. It
assumes a prior on the set of rewards $\fR$, and assumes that the feedback
provide a likelihood. The authors focus on two specific applications, i.e., IL
and learning $r^\star$, and propose to minimize the \emph{expected} loss. In
some way, \cite{Ramachandran2007birl} can be seen as deploying the object $x$
that minimizes the \emph{average} loss w.r.t. some known distribution over
$r^\star$ (instead of our \emph{worst-case} loss), and
\cite{brown2020bayesianrobust} that generalizes this approach using risk
measures.
We mention also \cite{zhu2023principledRLHF}, that in the context of PbRL make a
``pessimistic'' choice of policy given a ``confidence'' set of rewards. From a
high-level perspective, their pessimistic approach is very close to our robust
approach. However, their proposal aims to address estimation issues, while ours
concerns identifiability issues.
\cite{jeon2020rewardrational} introduce a framework for combining multiple and
various ReL feedback, that are all modeled using a Boltzmann distribution
constructed indirectly through the target reward.
\cite{cheng2024rime} adopt a robust approach to PbRL, but the
considered robustness aims to address a noisy feedback, and not partial
identifiability.
Finally, we mention \cite{huang2018learning} that adopt a robust approach for
policy optimization using multiple demonstrations.

\subsection{A Note on IL Works}\label{apx: note IL}

In the IL literature \cite{osa2018IL}, there are some algorithms that can be
interpreted as adopting the robust approach presented in Section \ref{sec:
robustness}. Specifically, such algorithms aim to solve the following
optimization problem:
\begin{align}\label{eq: il related works}
  \pi\in \argmin\limits_{\pi'\in\Pi}\max\limits_{r\in\cR}
  \Bigr{
  J^{\overline{\pi}}(r;p) - J^{\pi'}(r;p)},
\end{align}
where $\overline{\pi}$ is the expert's policy and $\cR$ is some set of reward
functions \cite{ho2016modelfreeimitationlearningpolicy}, e.g., rewards linear
\cite{abbeel2004apprenticeship} or convex \cite{syed2007game} in some known
feature map $\phi:\SA\to\RR^d$.
If we look at Eq. \eqref{eq: il related works} through the lens of our
framework, then $\pi$ can be seen as the \emph{robust} choice of policy that
minimizes the worst-case loss $\cL_g=J^{\overline{\pi}}(r;p) - J^{\pi}(r;p)$
over the feasible set $\cR_\cF=\cR$, and $\cX_g=\Pi$.

\subsection{Comparison with \cite{skalse2023invariancepolicyoptimisationpartial}}
\label{apx: comparison with skalse}

The work of \cite{skalse2023invariancepolicyoptimisationpartial} deserves its
own section because it introduces a ReL framework that shares similarities with
ours. Thus, it is important to remark the differences.

\emph{Similarly} to our framework, also the framework of
\cite{skalse2023invariancepolicyoptimisationpartial} provides a model of
feedback $f$ that always corresponds to a feasible set $\cR_f\subseteq\fR$.
The \emph{difference} lies in how they model the applications $g$. Specifically,
\cite{skalse2023invariancepolicyoptimisationpartial} model an application $g$ as
a function $g:\fR\to\cX_g$, and, for every reward $r\in\fR$, they let $g(r)$ be
the object that should be deployed in case reward $r$ were the target reward
$r^\star$. Crucially, this model is \emph{qualitative}, namely, it does not
contemplate the possibility that another object $x\neq g(r^\star)$ can be
deployed, and so the authors require that, for every pair of rewards
$r,r'\in\cR_f$ in the feasible set, it must hold that $g(r)=g(r')$ (since
$r^\star\in\cR_f$, this guarantees that no object other than $g(r^\star)$ can be
deployed). Instead, in our \emph{quantitative} framework, we allow for this to
happen as long as the loss $\cL_g(r^\star,x)$ of $x$ is sufficiently small.

In other words, the framework of
\cite{skalse2023invariancepolicyoptimisationpartial} and ours give birth to
different \emph{sufficient conditions} on whether a set of feedback $\cF$ can be
used for carrying out an application $g$. In particular, the framework of
\cite{skalse2023invariancepolicyoptimisationpartial} requires that:
\begin{align}\label{eq: suff cond skalse}
  \cR_\cF\subseteq g^{-1}\bigr{g(r^\star)},
\end{align}
i.e., that all the rewards in the feasible set $\cR_\cF$ prescribe the same
object $x=g(r^\star)$, the object prescribed also by $r^\star$. Clearly, this is
qualitative, and if we add a single reward $r'$ with $g(r')\neq g(r^\star)$ to
$\cR_\cF$, then the framework of
\cite{skalse2023invariancepolicyoptimisationpartial} concludes that we cannot
use $\cF$ for $g$ anymore.
Our framework, instead, provides a \emph{more general} quantitative sufficient
condition. Specifically, for some threshold $\Delta\ge0$, our framework requires
that there exists at least an item $x\in\cX_g$ such that:
\begin{align}\label{eq: suff cond ours}
  \cL_g(r,x)\le\Delta \qquad \forall r\in\cR_\cF.
\end{align}
Intuitively, if we set $\Delta=0$, then we recover Eq. \eqref{eq: suff cond
skalse}, because we would be requiring that all the rewards $r$ in the feasible
set have the same minima of $\cL_g(r,\cdot)$, and so, that they all prescribe
the same item. Instead, by enforcing Eq. \eqref{eq: suff cond ours}, we are
basically asking that it is possible to find at least one item (i.e., $x$ in Eq.
\ref{eq: suff cond ours}), that suffers from a true loss upper bounded by
$\Delta$, i.e., Eq. \eqref{eq: suff cond ours} guarantees that:
\begin{align*}
  \cL_g(r^\star,x)\le\Delta,
\end{align*}
since $r^\star\in\cR_\cF$.

A simple example can help in recognizing the advantages of our framework.
Consider an MDP with a single state $s$, three actions $a_1,a_2,a_3$, and
horizon $H=1$. Assume that we have received some feedback $\cF$ that results in
the feasible set $\cR_\cF=\{r_1,r_2\}$ containing only two rewards, and that the
application $g$ consists in outputting the optimal policy under $r^\star$.
Then, \cite{skalse2023invariancepolicyoptimisationpartial} tell us that we can
solve this problem \emph{if} the optimal policies induced by $r_1$ and $r_2$
coincide (in other words, if it is irrelevant whether $r^\star$ is $r_1$ or
$r_2$). However, what about the situation in which, e.g.:
\begin{align*}
  r_1(s,a)=\begin{cases}
    1 &\text{if }a=a_1\\
    0 &\text{if }a=a_2\\
    0 &\text{if }a=a_3\\
  \end{cases},
  \qquad 
  r_2(s,a)=\begin{cases}
    1-\Delta &\text{if }a=a_1\\
    1 &\text{if }a=a_2\\
    0 &\text{if }a=a_3\\
  \end{cases}\qquad?
\end{align*}
Clearly, $r_1$ makes $a_1$ optimal, while $r_2$ makes $a_2$ optimal. However, as
long as $\Delta$ is small, $a_1$ is an almost-optimal policy also for $r_2$,
thus, intuitively, it ``will not be a problem for $g$ if we deploy $a_1$ even if
$r^\star=r_2$''.
Note that our framework allows for this situation, by telling us that, as long
as $\Delta$ is sufficiently small, then partial identifiability is not an issue
for this problem, and we can solve it (by deploying the policy that plays $a_1$,
or better, based on the robust approach presented in Section \ref{sec:
robustness}, by deploying the policy that plays a mixture of $a_1,a_2$ with
equal probabilities).
In addition, note that $\Delta$ corresponds to the uninformativeness
$\cI_{\cF,g}$, and, thus, can be computed.

To sum up, our framework has the advantage of allowing for quantitative
considerations, while the framework of
\cite{skalse2023invariancepolicyoptimisationpartial} cannot.

\section{More on the \emph{Policy Comparison} Feedback}\label{apx: examples of
feedback} 

In this appendix, we provide additional details on the new \emph{policy
comparison} feedback (see Table \ref{table: feedback}) introduced in Section
\ref{sec: feedback}.

In the IRL literature, we are given a dataset of demonstrations, i.e.,
state-action trajectories $\cD=\{\omega^i\}_i\sim d^{\pi^E}$ collected by
executing some (expert) policy $\pi^E$. This setting models the situation in
which we observe an agent doing a task many times, and the assumption is that
the agent behavior is guided by the reward $r^\star$.

In the PbRL literature, we are given two trajectories and a preference signal
between them. This setting models the situation in which an agent expresses a
preference between two trajectories, i.e., the agent observes two trajectories
and says which one it prefers. In doing so, the choice is guided by the target
reward $r^\star$.

The \emph{policy comparison} feedback introduced in Section \ref{sec: feedback}
concerns with the scenario in which we have an expert agent (i.e., the agent
with $r^\star$ in mind) that expresses a preference (or, more generally, any
statement) between two datasets of demonstrations collected by (but not
necessarily) two other agents. Simply put, it can be seen as a mix of the
demonstrations and the trajectory comparisons feedback types.

For example, assume that we observe two agents $A_1,A_2$ demonstrating the task
of driving a car, and they provide two datasets of demonstrations
$\cD_1=\{\omega^1_i\}_i\sim d^{\pi^1},\cD_2= \{\omega^2_i\}_i\sim d^{\pi^2}$
where $\pi^1$ is the policy of $A_1$ and $\pi^2$ is the policy of $A_2$. Assume
that we do not want to learn the reward function that guides the behavior of
$A_1$ nor $A_2$, but we aim to learn the reward of a third agent $E$ (i.e.,
$r^\star$ is the reward of $E$). Thus, we can ask $E$ to provide us with a
preference signal (or, more generally, any statement between the demonstrated
olicies $\pi^1,\pi^2$) between the behavior of $A_1,A_2$. For instance, we can
show to $E$ the video of how $A_1,A_2$ drive, and we can ask him who drives
better. Then, we can use this feedback to infer $r^\star$ and carry out any
downstream application $g$.

\section{When the Application is to Deploy a Reward}
\label{apx: object deploy reward}

In this appendix, we provide additional insights on the class of ReL problems in
which the application $g$ requires the deployment of a reward function, i.e.,
$\cX_g = \fR$. Before that, we need some additional notation.

\paragraph{Additional notation.}

Let $\cX$ be a set, and let $d:\cX\times\cX\to\RR_+$ be a premetric in $\cX$.
If, in addition, $d$ satisfies $(i)$ $d(x,y)=0$ if and only if $x=y$ (identity
of indiscernibles), $(ii)$ $d(x,y)=d(y,x) \;\forall x,y\in\cX$ (simmetry),
$(iii)$ $d(x,y)\le d(x,z)+d(z,y)\;\forall x,y,z\in\cX$ (triangle inequality),
then we say that $d$ is a \emph{metric}.
Let $d$ be a premetric in a set $\cX$. The \emph{Chebyshev center}
\cite{alimov2019chebyshev} of a set $\cY\subseteq\cX$ is any of the points in
$\argmin_{x\in\cX}\max_{y\in\cY}d(x,y)$.
The \emph{Chebyshev radius} of $\cY$ is defined as
$\min_{x\in\cX}\max_{y\in\cY}d(x,y)$, while the diameter of $\cY$ is
$\max_{x,y\in\cY}d(x,y)$.
Moreover, given a reward $r$ and an environment with dynamics $p$, we denote by
$\Pi^*(r;p)\coloneqq\argmax_\pi J^\pi(r;p)$ the set of optimal policies under
$r$ in $p$.

\paragraph{Setting.}

There are situations in which we are interested in using some given feedback
$\cF$ for learning a reward function $\cX_g=\fR$. In these cases, the loss
$\cL_g:\fR\times\fR\to\RR_+$ can be seen as a premetric in the set of rewards
$\fR$. These settings can happen for various reasons, like:
\begin{itemize}
  \item The ultimate goal is to learn $r^\star$. Then, we can use as loss some
  ``standard'' metric in the set of rewards, like that induced by the 2-norm
  (see \cite{Ramachandran2007birl}):
  \begin{align}\label{eq: norm 2 rewards Lg}
    \cL_2(r,r')\coloneqq \|r-r'\|_2,
\end{align}
or by the max norm:
\begin{align}\label{eq: norm max rewards Lg}
  \cL_\infty(r,r')\coloneqq \|r-r'\|_\infty\coloneqq \max\limits_{s,a,h}|r_h(s,a)-r_h'(s,a)|.
\end{align}
\item The true application $g$ is unknown or revealed a posteriori, and we just
know that it belongs to a given set of applications. For instance, in case we
know that we want to use $r^\star$ to assess the return of some trajectory
$\omega'$, but we do not know $\omega'$ yet, then solving all the ReL problems
corresponding to every possible trajectory may be inefficient (there is an
exponential number of trajectories). Thus, we can simply learn a reward and then
use it at a later time when we will observe $\omega'$. In such setting, we can
use as loss:
\begin{align}\label{eq: max trajectories Lg}
  \cL_{\text{TR}}(r,r')\coloneqq \max_{\omega\in\Omega}
  \big|G(\omega;r)-G(\omega;r')\big|,
\end{align}
which quantifies the worst-case error among all possible trajectories.
\item For some reason, we aim to pass through a reward function to solve the
true application. For instance, in case the application is planning in some MDP
with transition model $p$ (see Table \ref{table: application}), then we can add
an intermediate step of passing through a reward function by using as loss:
\begin{align}\label{eq: loss planning}
  \cL_{\text{PL},p}(r,r')\coloneqq
  J^*(r';p)-\min\limits_{\pi\in\Pi^*(r;p)}J^{\pi}(r';p),
\end{align}
which considers, among all the optimal policies of $r$, the one that maximizes
the suboptimality under $r'$. Another similar example is the application of
computing the greedy policy
$\pi^{\text{gr}}(\cdot;r)\in\argmax_{a\in\cA}r(\cdot,a)$ in a stationary
environment \cite{zhu2023principledRLHF}. Here, for some distribution
$\rho\in\Delta^\cS$, we can use:
\begin{align}\label{eq: loss greedy}
  \cL_{\text{GR},\rho}(r,r')\coloneqq \E_{s\sim\rho}\big[
    \max_{a\in\cA} r'(s,a)-r'(s,
  \pi^{\text{gr}}(s;r))\big].
\end{align}
\end{itemize}

Note that we introduced the notion of premetric because not all of the
introduced dissimilarity functions are metrics:
\begin{restatable}{prop}{applicationpremetric}\label{prop: application pre
  metric} For any $p,\rho$, then
  $\cL_{\text{PL},p},\cL_{\text{CO},p},\cL_{\text{GR},\rho}$ are premetrics. Moreover,
  there are some $p,\rho$ such that:
  \begin{itemize}[noitemsep, leftmargin=*, topsep=-2pt]
    \item {\thinmuskip=1mu \medmuskip=1mu \thickmuskip=1mu
    $\cL_{\text{PL},p},\cL_{\text{CO},p},\cL_{\text{GR},\rho}$} lack the identity
    of indiscernibles;
    \item $\cL_{\text{PL},p},\cL_{\text{GR},\rho}$ are not simmetric;
    \item $\cL_{\text{PL},p},\cL_{\text{GR},\rho}$ lack the triangle
    inequality.
  \end{itemize}
  \end{restatable}
  \begin{proof}
    It is immediate that, for any transition model $p\in\Delta_\SAH^\cS$ and
    distribution $\rho\in\Delta^\cS$, the distances
    $\cL_{\text{PL},p},\cL_{\text{CO},p},\cL_{\text{GR},\rho}$ are non-negative for any
    pair of rewards $r,r'\in\fR$, and also that they are all 0 when $r=r'$. Thus, we
    conclude that they are premetrics.
    
    Now, consider the \textbf{identity of indiscernibles} property defined
    earlier. Given distance $\cL_g$, it does not hold if there exist $r\neq r'$
    s.t. $\cL_g(r,r')=0$. Concerning $\cL_{\text{PL},p}$, we see that any pair
    of rewards $r,r'$ such that $\Pi^*(r;p)=\Pi^*(r';p)$ satisfies
    $\cL_{\text{PL},p}(r,r')=0$. Thus, we can take $r'$ to be, e.g., a multiple
    of $r$ to get $\cL_{\text{PL},p}(r,r')=0$. Concerning $\cL_{\text{CO},p}$,
    simply consider as $p$ a transition model for which there exists at least a
    $(s,a,h)\in\SAH$ s.t., for all $\pi\in\Delta_{\SH}^\cS$,
    $d^{\pi}_h(s,a)=0$. Then, such triple does not contribute to the
    performance of any policy, and therefore any pair of rewards $r,r'$ that
    coincide everywhere except for $s,a,h$ satisfy $\cL_{\text{CO},p}(r,r')=0$.
    Concerning $\cL_{\text{GR},\rho}$, we can take $r,r'$ for which there exists
    a state $s\in\cS$ where $\max_{a\in\cA}r(s,a)=\max_{a\in\cA}r'(s,a)$, but
    the maximum is achieved by different actions
    $\argmax_{a\in\cA}r(s,a)\neq\argmax_{a\in\cA}r'(s,a)$. Clearly, $r\neq r'$,
    but $\cL_{\text{GR},\rho}=0$.
    
    Consider now the \textbf{simmetry} property. For $\cL_{\text{PL},p}$, take two rewards
    $r,r'$ such that $\Pi^*(r;p)=\{\pi^1\},\Pi^*(r';p)=\{\pi^1,\pi^2\}$ for some
    policies $\pi^1,\pi^2$. Then:
    \begin{align*}
      &\cL_{\text{PL},p}(r,r')\coloneqq J^*(r';p)-\min\limits_{\pi\in\Pi^*(r;p)}J^{\pi}(r';p)
      =J^{\pi^1}(r';p)-J^{\pi^1}(r';p)=0,\\
      &\cL_{\text{PL},p}(r',r)\coloneqq J^*(r;p)-\min\limits_{\pi\in\Pi^*(r';p)}J^{\pi}(r;p)
      =J^{\pi^1}(r;p)-J^{\pi^2}(r;p)\neq 0.
    \end{align*}
    Thus, $\cL_{\text{PL},p}(r,r')\neq \cL_{\text{PL},p}(r',r)$, so simmetry does not
    hold.
    For $\cL_{\text{CO},p}$ the simmetry property holds for any pair of rewards $r,r'\in\fR$:
    \begin{align*}
      \cL_{\text{CO},p}(r,r')&\coloneqq \max_\pi
        \big|J^{\pi}(r';p)-J^{\pi}(r;p)\big|\\
        &=\max_\pi
        \big|-\bigr{J^{\pi}(r;p)-J^{\pi}(r';p)}\big|\\
        &=\max_\pi
        \big|J^{\pi}(r;p)-J^{\pi}(r';p)\big|\\
        &= \cL_{\text{CO},p}(r',r).
    \end{align*}
    Concerning $\cL_{\text{GR},\rho}$, the simmetry property does not hold. To see it,
    consider a problem with a single state $s$ and two actions $a_1,a_2$, and let
    $r,r'$ be two rewards such that:
    \begin{align*}
      r(s,a)=\begin{cases}
        1 & \text{if }a=a_1,\\
        0 & \text{if }a=a_2,
      \end{cases},\qquad 
      r'(s,a)=\begin{cases}
        0 & \text{if }a=a_1,\\
        0.5 & \text{if }a=a_2,
      \end{cases}.
    \end{align*}
    Then, we have that:
    \begin{align*}
      &\cL_{\text{GR},\rho}(r,r')\coloneqq \E_{s\sim\rho}\big[
        \max_{a\in\cA} r'(s,a)-r'(s,
      \pi^{\text{gr}}(s;r))\big]=r'(s,a_2)-r'(s,a_1)=0.5,\\
    &\cL_{\text{GR},\rho}(r',r)\coloneqq \E_{s\sim\rho}\big[
      \max_{a\in\cA} r(s,a)-r(s,
    \pi^{\text{gr}}(s;r'))\big]=r(s,a_1)-r(s,a_2)=1.
    \end{align*}
    
    Finally, let us consider the \textbf{triangle inequality} property. First, we
    consider
    $\cL_{\text{PL},p}$. Let $r,r',r''\in\fR$ be three rewards such that $\Pi^*(r;p)=\Pi^*(r'';p)=\{\pi_1\},
    \Pi^*(r';p)=\{\pi^2\}$, and take $r''=0.5 r$. Then:
    \begin{align*}
      &\cL_{\text{PL},p}(r',r)= J^{\pi^1}(r;p)-J^{\pi^2}(r;p),\\
      &\cL_{\text{PL},p}(r',r'')= J^{\pi^1}(r'';p)-J^{\pi^2}(r'';p)=0.5J^{\pi^1}(r;p)-0.5
      J^{\pi^2}(r;p)=0.5 \cL_{\text{PL},p}(r',r),\\
      &\cL_{\text{PL},p}(r'',r)=J^{\pi^1}(r;p)-J^{\pi^1}(r;p)=0.
    \end{align*}
    Therefore, we have that:
    \begin{align*}
      \cL_{\text{PL},p}(r',r)> \cL_{\text{PL},p}(r',r'')+\cL_{\text{PL},p}(r'',r),
    \end{align*} 
    which proves that triangle inequality does not hold.
    Concerning $\cL_{\text{CO},p}$, triangle inequality holds for any $p$ and
    $r,r',r''\in\fR$:
    \begin{align*}
      \cL_{\text{CO},p}(r',r)&\coloneqq \max_\pi
        \big|J^{\pi}(r;p)-J^{\pi}(r';p)\popblue{\pm J^{\pi}(r'';p)}\big|\\
        &\le \max_\pi
        \big|J^{\pi}(r'';p)-J^{\pi}(r';p)|+\max_\pi
        \big|J^{\pi}(r;p)-J^{\pi}(r'';p)|\\
        &=\cL_{\text{CO},p}(r',r'')+\cL_{\text{CO},p}(r'',r),
    \end{align*}
    where we have applied the triangle inequality of the absolute value and the fact
    that the maximum of a sum is smaller than the sum of maxima. As far as
    $\cL_{\text{GR},\rho}$ is concerned, we note that it lacks the triangle inequality
    property with a counterexample analogous to that for $\cL_{\text{PL},p}$. Consider
    a problem with a single state $s$ (or $\rho$ supported only on it) and two
    actions $a_1,a_2$, and consider the rewards $r,r',r''\in\fR$ such that:
    \begin{align*}
      r(s,a)=\begin{cases}
        1 & \text{if }a=a_1,\\
        0 & \text{if }a=a_2,
      \end{cases},\qquad 
      r'(s,a)=\begin{cases}
        0 & \text{if }a=a_1,\\
        1 & \text{if }a=a_2,
      \end{cases},\qquad 
      r''(s,a)=\begin{cases}
        0.5 & \text{if }a=a_1,\\
        0 & \text{if }a=a_2,
      \end{cases}.
    \end{align*}
    Then, we have that:
    \begin{align*}
      &\cL_{\text{GR},\rho}(r',r)= r(s,a_1)-r(s,a_1)=1,\\
      &\cL_{\text{GR},\rho}(r',r'')= r''(s,a_1)-r''(s,a_1)=0.5=0.5 \cL_{\text{GR},\rho}(r',r),\\
      &\cL_{\text{GR},\rho}(r',r)=r(s,a_1)-r(s,a_1)=0.
    \end{align*}
    Thus:
    \begin{align*}
      \cL_{\text{GR},\rho}(r',r)=1> \cL_{\text{GR},\rho}(r',r'')+\cL_{\text{GR},\rho}(r'',r)=0.5+0.
    \end{align*}
    This concludes the proof.
    
    \end{proof}

\paragraph{Robust approach and Chebyshev center.}

In these scenarios, the robust choice $x_{\cF,g}$ presented in Eq. \eqref{eq:
minimax robust} corresponds to the Chebyshev center of the feasible set
$\cR_\cF$ in the premetric space $\tuple{\fR,\cL_g}$. Since $x_{\cF,g}\in\fR$,
then we denote it by $r_{\cF,g}$. Moreover, observe that the informativeness
$\cI_{\cF,g}$ (Eq. \ref{eq: informativeness}) can be interpreted as the
Chebyshev radius of the feasible set $\cR_\cF$, and that, in the worst-case, the
ReL procedure carried out by most algorithms in literature (see Section
\ref{sec: robustness}, Eq. \ref{eq: approach existing literature}) can be seen
as the diameter of the feasible set, that we denote by $D_{\cF,g}$. Finally,
given an arbitrary reward $r\in\fR=\cX_g$, we can define its worst-case loss as:
\begin{align}\label{eq: def comp}
  \cC_{\cF,g}(r)\coloneqq \max\limits_{r'\in\cR_\cF}\cL_g(r,r').
\end{align}
See Figure \ref{fig: fs center radius diameter} for a simple graphical intuition
of all these quantities.

\begin{figure}[t!]
  \centering
  \begin{tikzpicture}
      \filldraw[black,fill=white,very thick] (-2.5,-1.6) rectangle (2.5,1.6);
      \filldraw[vibrantBlue,fill=vibrantBlue!40, very thick, fill opacity=0.5]
      (0,0) ellipse (2 and 1.2);
      \draw[mygreen, very thick] (0,0.05) -- (2,0.05);
      \draw[vibrantOrange, very thick] (-1.2, 0.4) -- (1.5,-0.8);
      \filldraw[black] (0, 0) circle (2pt);
      \draw[red, very thick] (-2,0) -- (2,0);
      \filldraw[black] (-1.2, 0.4) circle (2pt);
      \node[text width=2cm] at (-2.1,1.3) {\large$\fR$};
      \node at (1.4,+1.2) {\textcolor{vibrantBlue}{$\cR_{\cF}$}};
      \node at (-1,-0.4) {\textcolor{red}{$D_{\cF,g}$}};
      \node at (1,0.4) {\textcolor{mygreen}{$\cI_{\cF,g}$}};
      \node at (0.5,-0.7) {\small\textcolor{vibrantOrange}{$\cC_{\cF,g}(r)$}};
      \node at (-0.1,+0.3) {$r_{\cF,g}$};
      \node at (-1,+0.5) {$r$};
  \end{tikzpicture}
      \caption{
        Illustration of the quantities of interest. $r$ is any reward.
      }
  \label{fig: fs center radius diameter}
\end{figure}

In the following, we provide some interesting results on these quantities.

\paragraph{Some results.}

One of the most interesting results is that, depending on $\cF,g$, the Chebyshev
center might lie \emph{outside} of the feasible set $\cR_\cF$. This is
interesting because no ReL work in literature has looked for a reward function
outside $\cR_\cF$ as far as we know. Formally, we show that this can happen for
$\cL_{\text{PL},p}$ and $\cL_\infty$:

\begin{restatable}{prop}{considerationrewardchoicenofs}
    \label{prop: considerations reward choice not in fs}
    If the loss is $\cL_g=\cL_{\text{PL},p}$, then there exists a (convex) set
    $\cR_\cF$ for which $r_{\cF,g}\notin\cR_\cF$.
\end{restatable}
\begin{proof}
  Consider a simple MDP without reward with a single state $s$, three actions
  $a_1,a_2,a_3$, and horizon $H=1$. Let $\pi^1,\pi^2,\pi^3$ be, respectively,
  the deterministic policies that play actions $a_1,a_2,a_3$. In this context,
  take $\cR_\cF=\{r,r'\}$, where:
  \begin{align*}
    r(s,a)=\begin{cases}
      1 & \text{if }a=a_1,\\
      0 & \text{if }a=a_2,\\
      0.5 & \text{if }a=a_3,
    \end{cases},\qquad 
    r'(s,a)=\begin{cases}
      0 & \text{if }a=a_1,\\
      1 & \text{if }a=a_2,\\
      0.5 & \text{if }a=a_3,
    \end{cases}.
  \end{align*}
  Since $\Pi^*(r;p)=\{\pi^1\}$ and $\Pi^*(r';p)=\{\pi^2\}$, then it is immediate
  that:
  \begin{align*}
    \cL_{\text{PL},p}(r,r')=\cL_{\text{PL},p}(r',r)=1.
  \end{align*}
  The robust reward choice is any reward $r_{\cF,g}$ s.t.
  $\Pi^*(r_{\cF,g};p)=\{\pi^3\}$. Indeed, in this manner:
  \begin{align*}
    \cL_{\text{PL},p}(r_{\cF,g},r)=\cL_{\text{PL},p}(r_{\cF,g},r')=0.5.
  \end{align*}
  Thus, $r_{\cF,g}\notin\cR_\cF$.

  Note that we can provide a counterexample with a convex feasible set by using
  as $\cR_\cF$ the convex hull of $r,r'$. In this way, it is simple to see that
  no reward in this new feasible set can make $\pi^3$ be the unique optimal
  policy. As such, the Chebyshev center is still external to $\cR_\cF$.

  This concludes the proof.
\end{proof}

\begin{prop}
Let the feasible set $\cR_\cF$ be the 3-dimensional convex hull of rewards
$[1,1,1]$, $[-1,1,1]$, $[1,-1,1]$, $[1,1,-1]$. Then, if we take
$\cL_g=\cL_\infty$, the Chebyshev center is $r_{\cF,g}=[0,0,0]$ which lies
outside $\cR_\cF$.
\end{prop}
\begin{proof}
  See footnote 1 in \cite{dabbene2014probabilistic}.
\end{proof}

To quantify how convenient is our reward choice $r_{\cF,g}$ w.r.t. what is done
in literature (see Eq. \eqref{eq: approach existing literature}), we have to
understand how small is the ratio $\cI_{\cF,g}/D_{\cF,g}$.
However, since it coincides with the ratio between the Chebyshev radius and the
diameter of the feasible set $\cR_\cF$ in premetric space $\tuple{\fR,\cL_g}$,
then the answer \emph{depends} on the specific $\cR_\cF$ and $\cL_g$ at stake.

If $\cL_g$ is a metric, then the error reduction is at most $50\%$:
\begin{restatable}{prop}{erroratmostfiftypercent}\label{prop: error at most 50
percent} If $\cL_g$ satisfies simmetry and triangle inequality, then:
\begin{align*}
 \cI_{\cF,g}\ge D_{\cF,g}/2.
\end{align*}
\end{restatable}
\begin{proof}
  Let $(r_1,r_2)\in\argmax_{r,r'}\cL_g(r,r')$, i.e., be the points in the diameter
  of the feasible set $\cR_\cF$, thus $D_{\cF,g}=\cL_g(r_1,r_2)$. Because of
  triangle inequality, for any $r\in\fR$ including the Chebyshev center
  $r_{\cF,g}$, it holds that:
  \begin{align*}
    D_{\cF,g}=\cL_g(r_1,r_2)&\le \cL_g(r_1,r) + \cL_g(r,r_2).
  \end{align*}
  If we take $r=r_{\cF,g}$ and apply simmetry, then we see that
  $\cL_g(r_1,r_{\cF,g})=\cL_g(r_{\cF,g},r_1)\le \max_{r\in\cR_\cF}\cL_g(r_{\cF,g},r)=
  \cI_{\cF,g}$, and also that $\cL_g(r_{\cF,g},r_2)\le
  \max_{r\in\cR_\cF}\cL_g(r_{\cF,g},r)= \cI_{\cF,g} $, from which:
  \begin{align*}
    D_{\cF,g}=\cL_g(r_1,r_2)&\le 2\cI_{\cF,g}.
  \end{align*}
\end{proof}

If $\cL_g$ is induced by the $2$-norm, we get an error reduction of at least
$1/\sqrt{2}$:
%
\begin{restatable}{prop}{erroratleastjung}\label{prop: error at least jung}
  %
  If $\cL_g=\cL_2$, then:
  \begin{align*}
   \cI_{\cF,g}\le \sqrt{1-1/(SAH+1)}D_{\cF,g}/\sqrt{2}.
  \end{align*}
\end{restatable}
As an example of why it is not exactly $50\%$, think to the equilateral triangle
or to the simplex.
\begin{proof}
  The result is immediate by an application of the Jung's theorem
  \citep{jung1901ueber,danzer1963helly,scott1991extension}, that states that in
  every Euclidean space with $n$ dimensions, the Chebyshev radius $r$ of a set
  with diameter $D$ satisfies:
  \begin{align*}
    r\le \sqrt{\frac{n}{2(n+1)}}D.
  \end{align*}
\end{proof}

It should be noted that, in absence of the triangle inequality, Proposition
\ref{prop: error at most 50 percent} does not hold, and the radius $\cI_{\cF,g}$
can be zero when the diameter $D_{\cF,g}$ is not, providing an ``infinite''
increase in performance:
\begin{restatable}{prop}{errornotriangleinequality}\label{prop: error no
triangle inequality} There exist $p,\rho$ and feedback $\cF$ such that the
premetrics $\cL_{\text{PL},p},\cL_{\text{GR},\rho}$ satisfy:
  \begin{align*}
    & \cI_{\cF,(\text{PL},p)}=0 \text{ and }D_{\cF,(\text{PL},p)}=H,\\
    & \cI_{\cF,(\text{GR},\rho)}=0 \text{ and }D_{\cF,(\text{GR},\rho)}=1.
  \end{align*}
  \end{restatable}
\begin{proof}
  Let us begin with $\cL_{\text{PL},p}$. Consider a problem with horizon $H$ where
  $\cR_\cF=\{r\in\fR\,|\,\pi_1\in\Pi^*(r;p)\}$ the feasible set contains all the
  rewards that make at least $\pi_1$ as optimal policy. Consider a transition
  model $p$ for which there exists a reward $\overline{r}\in\cR_\cF$ s.t.
  $J^*(\overline{r};p)=J^{\pi^1}(\overline{r};p)=H$, and for some other policy
  $\pi^2$ it holds that: $J^{\pi^2}(\overline{r};p)=0$ (for instance, if we
  construct $p$ to be deterministic, then, this is possible). Then, since
  $\cR_\cF$ contains also the reward $r'$ that makes all the policies optimal
  (in particular $\pi^2$), we have that:
  \begin{align*}
    D_{\cF,(\text{PL},p)}\ge \cL_{\text{PL},p}(r',\overline{r})=J^*(\overline{r};p)-J^{\pi^2}(\overline{r};p)=H.
  \end{align*}
  Instead, if we consider any reward $r''$ s.t. $\Pi^*(r'';p)=\{\pi^1\}$ (as our
  robust reward choice), then it is clear that:
  \begin{align*}
    \cI_{\cF,(\text{PL},p)}\le \max\limits_{r\in\cR_\cF}\cL_{\text{PL},p}(r'',r)=J^*(r;p)-J^{\pi^1}(r;p)=0.
  \end{align*}
  The other claim of the proposition can be proved with an analagous construction.
\end{proof}

\section{Some Insights on Model Selection and Active Learning}
\label{apx: misspecification and active learning}

In this appendix, we provide some considerations on model selection and active
learning that make use of the framework and robust approach introduced in the
main paper. Here, for simplicity of presentation, we introduce a new symbol
$\bA_f$: we will refer to a feedback $f$ as a pair $(Q_f,\bA_f)$, where $\bA_f$
denotes the assertion/statement that describes the relation between $Q_f$ and
$r^\star$, and, so, gives birth to the feasible set $\cR_f$. It is convenient to
introduce $\bA_f$ because it permits to reason on different feedback $f,f'$ of
the same type, i.e., sharing the same quantity $Q_f=Q_{f'}$, in a simple and
intuitive way. We will call $\bA_f$ an ``assertion''.

\subsection{Discussion on Model Selection}

For simplicity of presentation, we consider a single feedback $\cF=\{f\}$, with
$f=\tuple{Q_f,\bA}$, and we assume that infinite data are available.
 
\textbf{Modelling feedback.}~~%
In practical applications, we observe the data $Q_f$, and we have to select the
specific assertion $\bA_f$ based on our knowledge on how $Q_f$ has been
generated.
%
\begin{example}[Feedback of driving safely]\label{ex: driving safely}
  Let the \emph{unknown} target reward $r^\star$ represent the task of driving
  safely. Let $Q_f=\{\omega_i\}_i$ be a dataset of demonstrations collected by
  an agent $\pi^E$ that ``aims to demonstrate how to drive safely in a certain
  environment''.
  What assertion do we adopt?
 \end{example}
 We can decide to model the problem in different ways, i.e., there are always
 multiple assertions $\bA_f$ that can be associated to the given data $Q_f$ to
 \emph{try} to capture the ``true'' relationship between $Q_f$ and the
 underlying unknown $r^\star$.
  \begin{continueexample}{ex: driving safely}
  We can model $\pi^E$ as an optimal policy for the ``driving safely'' task by
  using assertion $\bA^{\text{OPT}}$, corresponding to the ``optimal expert''
  entry in Table \ref{table: feedback}. If we think that the demonstrated
  $\pi^E$ sometimes takes suboptimal actions, then we might prefer using
  $\bA^{\text{MCE}}$ (``$\beta$-MCE expert'' in Table \ref{table: feedback}).
  Alternatively, we can simply assume that $\pi^E$ is at least
  $\epsilon$-optimal for some $\epsilon>0$, i.e., $J^{\pi^E}(r^\star;p)\ge
  J^*(r^\star;p)-\epsilon$. We call this assertion $\bA^{\text{SUB}}$
  (``$t$-suboptimal expert'' in Table \ref{table: feedback}). 
  \end{continueexample}
 We incur in \emph{misspecification}
 \emph{if $\bA_f$ does not correctly describe the relationship between $r^\star$
 and $Q_f$} (see \cite{skalse2024quantifyingsensitivityinversereinforcement} for
 an analysis of misspecification in IRL).
 \begin{continueexample}{ex: driving safely}
    If we model data $Q_f$ using assertion $\bA^{\text{MCE}}$, but, actually,
    $\pi^E$ is optimal (i.e., it follows $\bA^{\text{OPT}}$), then our feedback
    is misspecified.
  \end{continueexample}

\textbf{Choosing the correct modelling assertion.}~~%
The crucial question is: \emph{what is the} best \emph{modelling assertion for
the given data $Q_f$?} Intuitively, \emph{any} assertion that permits to carry
out the downstream ReL application $g$ \emph{effectively} is satisfactory. In
other words, we can tolerate some misspecification error as long as the final
outcome is acceptable.\\
%
Thanks to our framework, we can make these considerations more quantitative. As
explained in Section \ref{sec: robustness}, the uninformativeness
$\cI_{\{f\},g}$ represents the \emph{minimum} error that can be achieved in the
worst-case for doing $g$ using $f=(Q_f,\bA_f)$. Under the model of feedback $f$
(i.e., the assertion $\bA_f$) considered, the worst-case error cannot be smaller
than $\cI_{\{f\},g}$. Therefore, this is unsatisfactory if we aim to carry out
$g$ with an error $\Delta < \cI_{\{f\},g}$.\\
To solve this issue, we have to reduce the size of the feasible set $\cR_f$ so
that its Chebyshev radius $\cI_{\{f\},g}$ reduces too. There are two ways for
doing this. The best one $(i)$ consists in collecting additional feedback $f'$
to obtain $\cR_f\cap\cR_{f'}\subseteq\cR_f$ (e.g., through active learning, see
Appendix \ref{apx: active learning}). However, additional feedback might not
be available in practice. The other way $(ii)$ consists in imposing ``more
structure'' to the problem by changing the assertion $\bA_f$ to $\bA_{f'}$ to
restrict the feasible set $\cR_{\{Q_f,\bA_f\}}$ to $\cR_{\{Q_f,\bA_{f'}\}}$.
%
\begin{example}
  If we model $Q_f$ using $\bA^{\text{OPT}}$, we obtain a feasible set that is
  \emph{strictly} contained into the feasible set obtained using assertion
  $\bA^{\text{SUB}}$. Thus, whatever the application $g$ at stake,
  $\cI_{\{Q_f,\bA^{\text{OPT}}\},g}\le \cI_{\{Q_f,\bA^{\text{SUB}}\},g}$.
\end{example}
%
%
By adopting a more restrictive (potentially less realistic) model $\bA_{f'}$, we
reduce the ``measurable'' error $\cI_{\{Q_f,\bA_{f'}\},g}$ at the price of a
larger unknown misspecification error. This is fine as long as $\bA_{f'}$ is
perceived as sufficiently realistic.
\emph{If there is no realistic assertion with a small enough value of
informativeness}, then we conclude that the application $g$ cannot be carried
out effectively with the only data $Q_f$. In other words, \emph{the ReL problem
cannot be solved with the desired accuracy}.

\subsection{Active Learning}
\label{apx: active learning}

In the Active Learning setting \citep{lopes2009activelearning}, in addition to a
given set of feedback $\cF$, we can choose to receive a new feedback $f'$ from a
set $\cF'=\{f_i'\}_i$. Crucially, for any $f'\in\cF'$,
$f'=\tuple{Q_{f'},\bA_{f'}}$, we consider the assumption $\bA_{f'}$ to be known,
while the actual data $Q_{f'}$ is revealed \emph{after} our choice.
\begin{example}
  We might choose between feedback $f_1=\tuple{Q_1,\bA^{\text{OPT}}}$,
  consisting of demonstrations in environment $\cM_1$ (i.e., $Q_1$) from an
  optimal expert (i.e., $\bA^{\text{OPT}}$, see Example \ref{ex: driving
  safely}), or $f_2=\tuple{Q_2,\bA^{\text{MCE}}}$, made of demonstrations in
  environment $\cM_2$ (i.e., $Q_2$) from a maximal causal entropy expert (i.e.,
  $\bA^{\text{MCE}}$). After our choice, the expert demonstrates only one policy
  (i.e., either $Q_1$ or $Q_2$).
\end{example}
What is the ``best'' choice of feedback $f$ to receive without knowing its data
$Q_f$? Thanks to our framework, and specifically to the notion of
uninformativeness defined in Section \ref{sec: robustness}, the immediate choice
is the feedback that, whatever its true data, is the most ``informative'':
%
%
\begin{defi}[Information gain]\label{def: information gain}
  We define the \emph{information
gain} $\text{IG}_{\cF,g}(f)$ of a feedback $f$ w.r.t. $g$ given $\cF$ as:
\begin{align*}
  \text{IG}_{\cF,g}(f)\coloneqq \cI_{\cF,g}-\cI_{\cF\cup\{f\},g}.
\end{align*}
\end{defi}
In words, $\text{IG}_{\cF,g}(f)$ represents the reduction of
uninformativeness (worst-case error) of the ReL problem. Observe that
$\text{IG}_{\cF,g}(f)\ge 0$, since $\cI_{\cF\cup\{f\},g}\le \cI_{\cF,g}$,
i.e., the more feedback the less error.

Formally, the feedback that reduces the most the uninformativeness is the
feedback $f_j'$ from set $\cF'=\{f_i'\}_i,f_i'=\tuple{Q_i,\bA_i}$, that
maximizes the \emph{worst-case} information gain w.r.t. the data $Q_i$:
\begin{align*}
 j\in\argmax\limits_i\min\limits_{Q_i}
  \text{IG}_{\cF,g}(\{Q_i,\bA_i\}).
\end{align*}
Simply put, we select the feedback that, whatever the true data we will
receive, we know that it will bring the larger reduction of informativeness,
i.e., of error, for the ReL problem.

\section{Additional Results on Section \ref{sec: use case}}\label{apx: proofs
rob}

In this appendix, we provide additional results for Section \ref{sec: use case}.
We begin by reporting the subroutine \texttt{PDSM-MAX} in Appendix \ref{apx:
pdsm max}, which is not present in the main paper. Then, we provide explicit
computation of the subgradients of the Lagrangian function $\widehat{L}$ for the
estimated problem (Appendix \ref{apx: subgradients lagrangian}). We prove
Theorem \ref{thr: guarantees caty} in Appendix \ref{apx: proof thr rob}. we
provide a miscellaneous of other results in Appendix \ref{apx: other proofs and
results}, and we conclude by showing that \rob can be easily extended to other
ReL problems (Appendix \ref{apx: other kinds feedback}).

\subsection{PDSM-MAX}\label{apx: pdsm max}

The pseudocode of the \texttt{PDSM-MAX} subroutine is reported in Algorithm
\ref{alg: pdsm max}. Note that it is analogous to that of \texttt{PDSM-MIN}
(Algorithm \ref{alg: pdsm min}), with the difference that set $\fD_-$ is
mirrored w.r.t. $\fD_+$, and that the signs of the subgradients update are
reversed.

We remark that operators $\Pi_{\fR},\Pi_{\fD_-},\Pi_{\fD_+}$ can be implemented
in $\cO(d)$ time for a vector in $\RR^d$. Indeed, as explained in
\cite{boyd2004convex}, given a vector $x\in\RR^d$, we can implement the
projection onto set $\cY\coloneqq[k_1,k_2]^d$ (with arbitrary $k_1\le k_2$) as:
\begin{align*}
  [\Pi_{\cY}(x)]_j=\begin{cases}
    k_1&\text{if }x_j\le k_1\\
    k_2&\text{if }x_j\ge k_2\\
    x_j&\text{otherwise }
  \end{cases},
\end{align*}
where $[\cdot]_j$ denotes the $j$-th component.

Finally, we remark also that the explicit computation of the subgradients is
explained in Appendix \ref{apx: subgradients lagrangian}.

\let\oldnl\nl

\RestyleAlgo{ruled}
\SetInd{0.5em}{0.5em}
\LinesNumbered
 \begin{algorithm}[t!]
    \caption{\texttt{PDSM-MAX}}\label{alg: pdsm max}
    \SetKwInOut{Input}{Input}
    \small
    \Input{{\thinmuskip=1mu \medmuskip=1mu \thickmuskip=1mu objective
    $\widehat{L}$, iterations $K$}}

    $\lambda_0 \gets 0$

    $r_0\gets 0$

    $\fD_-\gets \{\lambda\le 0:\, \|\lambda\|_2\le s \}$

    \For{$k=0,1,\dotsc,K$}{
        $r_{k+1}\gets \Pi_{\fR}\bigr{r_k+\alpha \partial_r
        \widehat{L}(r_k,\lambda_k)}$
        
        $\lambda_{k+1}\gets \Pi_{\fD_-}\bigr{\lambda_k-\alpha \partial_\lambda
    \widehat{L}(r_k,\lambda_k)}$ }

    $\widehat{r}_K \gets \frac{1}{K} \sum_{k=0}^K r_k$

  \textbf{Return} $\dotp{\widehat{d}^{\pi^1}-\widehat{d}^{\pi^2},\widehat{r}_K}$

 \end{algorithm}

\subsection{Subgradients of the Lagrangian}\label{apx: subgradients lagrangian}

Both subroutines \texttt{PDSM-MIN} and \texttt{PDSM-MAX} require the computation
of the subgradients $\partial_r \widehat{L}$, $\partial_\lambda \widehat{L}$ of
the estimated Lagrangian $\widehat{L}$ (Eq. \eqref{eq: estimated lagrangian})
w.r.t. $r,\lambda$. We report Eq. \eqref{eq: estimated lagrangian} below for
arbitrary $r,\lambda$:
\begin{align*}
  \widehat{L}(r,\lambda)&
 =\dotp{\widehat{d}^{\pi^1}-\widehat{d}^{\pi^2},r}
 +\sum\nolimits_i\lambda_{\text{D}}^i
 \bigr{\max\nolimits_\pi J^\pi(r;\widehat{p}_{\text{D},i})-\dotp{\widehat{d}^{\pi_{\text{D},i}},r}
 -t_i}\\
 & \qquad+\sum\nolimits_i\lambda_{\text{TC}}^i\dotp{d^{\omega^1_{\text{TC,i}}}-d^{\omega^2_{\text{TC,i}}}, r}
 +\sum\nolimits_i\lambda_{\text{PC}}^i \dotp{\widehat{d}^{\pi^1_{\text{PC},i}}
 -\widehat{d}^{\pi^2_{\text{PC},i}}, r}.
\end{align*}
Then, the following quantities are subgradients of $\widehat{L}$ evaluated at
any $r\in\fR$ and $\lambda\coloneqq [\lambda_{\text{TC}},
\lambda_{\text{C}}, \lambda_{\text{D}}]$:
\begin{align}\label{eq: subgradients estimated problem}
  &\partial_r \widehat{L}(r,\lambda)=(\widehat{d}^{\pi^1}-\widehat{d}^{\pi^2})+
  \sum\nolimits_i\lambda_{\text{D}}^i
 \bigr{d^{\widehat{\pi}_{\text{D},i}^{r,*}}-\widehat{d}^{\pi_{\text{D},i}}}\\
 &\qquad\qquad+
 \sum\nolimits_i\lambda_{\text{TC}}^i\bigr{d^{\omega^1_{\text{TC,i}}}-d^{\omega^2_{\text{TC,i}}}}
 +\sum\nolimits_i\lambda_{\text{PC}}^i \bigr{\widehat{d}^{\pi^1_{\text{PC},i}}
 -\widehat{d}^{\pi^2_{\text{PC},i}}},\nonumber
 \\
 &\partial_{\lambda^i_{\text{D}}}\widehat{L}(r',\lambda')=\max\nolimits_\pi J^\pi(r;\widehat{p}_{\text{D},i})
 -\dotp{\widehat{d}^{\pi_{\text{D},i}},r}
 -t_i,\qquad\forall i\in\dsb{m_{\text{D}}}\nonumber\\
&\partial_{\lambda^i_{\text{TC}}}\widehat{L}(r',\lambda')=\dotp{d^{\omega^1_{\text{TC,i}}}-d^{\omega^2_{\text{TC,i}}}, r},
\qquad\forall i\in\dsb{m_{\text{TC}}}\nonumber\\
&\partial_{\lambda^i_{\text{PC}}}\widehat{L}(r',\lambda')=\dotp{\widehat{d}^{\pi^1_{\text{PC},i}}
-\widehat{d}^{\pi^2_{\text{PC},i}}, r}, \qquad\forall i\in\dsb{m_{\text{PC}}}\nonumber
\end{align}
where $\widehat{\pi}^{r,*}_{\text{D},i}\in\argmax_\pi J^\pi(r;\widehat{p}_{\text{D},i})$
represents an arbitrary optimal policy in the \emph{estimated} MDP
$\tuple{\cS,\cA,H,s_{0,\text{D},i},\widehat{p}_{\text{D},i},r}$, for all
$i\in\dsb{m_{\text{D}}}$, and $d^{\widehat{\pi}_{\text{D},i}^{r,*}}$ represents
the visit distribution of $\widehat{\pi}^{r,*}_{\text{D},i}$ in the corresponding
\emph{estimated} MDP without reward
$\tuple{\cS,\cA,H,s_{0,\text{D},i},\widehat{p}_{\text{D},i}}$.

To see that this is actually the subgradient of $\widehat{L}$, simply note that
$\widehat{L}$ is linear in $\lambda$, thus the subgradients simply correspond to
the gradients.
Concerning $\partial_r \widehat{L}$, observe that $\widehat{L}$ is not
differentiable in $r$ because it contains the pointwise maximum of
differentiable functions $\max\nolimits_\pi J^\pi(r;\widehat{p}_{\text{D},i})$.
Thus, the subgradient $\partial_r \widehat{L}(r,\lambda)$ can be obtained as any
convex combination of the gradients $d^{\widehat{\pi}_{\text{D},i}^{r,*}}$ of
the functions that attain the maximum \citep{boyd2022subgradients}.

We remark that the computation of $\max\nolimits_\pi
J^\pi(r;\widehat{p}_{\text{D},i})$ and
$\widehat{\pi}^{r,*}_{\text{D},i}\in\argmax_\pi
J^\pi(r;\widehat{p}_{\text{D},i})$ can be done exactly using the \emph{backward
induction} algorithm \cite{puterman1994markov}. Then, also
$d^{\widehat{\pi}_{\text{D},i}^{r,*}}$ can be obtained exactly given
$\widehat{\pi}^{r,*}_{\text{D},i}$ and
$\tuple{\cS,\cA,H,s_{0,\text{D},i},\widehat{p}_{\text{D},i}}$ by simply starting
from:
\begin{align*}
  d^{\widehat{\pi}_{\text{D},i}^{r,*}}_{1}(s,a)=\widehat{\pi}^{r,*}_{\text{D},i,1}(a|s)\indic{s=s_{0,\text{D},i}},
\end{align*}
and then constructing the visit distribution for $h=2,\dotsc,H$ recursively
using the relation \cite{puterman1994markov}:
\begin{align*}
  d^{\widehat{\pi}_{\text{D},i}^{r,*}}_{h+1}(s,a)=\widehat{\pi}^{r,*}_{\text{D},i,h+1}(a|s)\sum\limits_{s',a'}
  d^{\widehat{\pi}_{\text{D},i}^{r,*}}_{h}(s',a')\widehat{p}_{\text{D},i,h}(s|s',a').
\end{align*}


\subsection{Proof of Theorem \ref{thr: guarantees caty}}
\label{apx: proof thr rob}

\thrcaty*
\begin{proof}
  Define $M,m$ as:
\begin{align}\label{eq: def M m}
    &M\coloneqq \max\limits_{r\in\cR_\cF} \dotp{d^{\pi^1}-d^{\pi^2},r},
    \qquad m\coloneqq \min\limits_{r\in\cR_\cF} \dotp{d^{\pi^1}-d^{\pi^2},r}.
\end{align}
  Then, we can write:
  \begin{align*}
    \cL_g(r^\star,\widehat{x}_K)-\cI_{\cF,g}&\markref{(1)}{\le}
    \max\limits_{r\in\cR_\cF}\cL_g(r,\widehat{x}_K)-\cI_{\cF,g}\\
    &\markref{(2)}{=}\max\Bigc{\widehat{x}_K-m,M-\widehat{x}_K}- \frac{M-m}{2}\\
    &=\max\Bigc{\widehat{x}_K-\frac{M+m}{2},\frac{M+m}{2}-\widehat{x}_K}\\
    &\markref{(3)}{=}\max\Bigc{\frac{\widehat{M}_K+\widehat{m}_K}{2}-\frac{M+m}{2},
    \frac{M+m}{2}-\frac{\widehat{M}_K+\widehat{m}_K}{2}}\\
    &\markref{(4)}{\le} \frac{1}{2}\Bigr{\big|\widehat{M}_K-M\big|+\big|\widehat{m}_K-m\big|}\\
    &\markref{(5)}{\le} \frac{1}{2}\Bigr{\big|\widehat{M}_K-\widehat{M}\big|
    +\big|\widehat{M}-M\big|+
    \big|\widehat{m}_K-\widehat{m}\big|+
    \big|\widehat{m}-m\big|},
  \end{align*}
  where at (1) we used that $r^\star\in\cR_\cF$ by hypothesis, at (2) we apply
  Lemma \ref{lemma: change objective} twice using the definitions of
  $\cI_{\cF,g}$ and $\cL_g$, at (3) we insert the definitions of
  $\widehat{M}_K$, $\widehat{m}_K$ as computed by \rob, at (4) we apply
  triangle's inequality, and at (5) we add and subtract
  $\widehat{M},\widehat{m}$ and apply again triangle's inequality.

  Similarly, note that:
  \begin{align*}
    \Big|\cI_{\cF,g}-\widehat{\cI}_{\cF,g}\Big|&=
    \Big|\frac{M-m}{2}- \frac{\widehat{M}_K-\widehat{m}_K}{2}\Big|\\
    &\le\frac{1}{2}\Bigr{\big|\widehat{M}_K-\widehat{M}\big|
    +\big|\widehat{M}-M\big|+
    \big|\widehat{m}_K-\widehat{m}\big|+
    \big|\widehat{m}-m\big|}.
  \end{align*}

  Observe that $\big|\widehat{M}_K-\widehat{M}\big| +
  \big|\widehat{m}_K-\widehat{m}\big|$ represents the iteration error, while
  $\big|\widehat{M}-M\big| + \big|\widehat{m}-m\big|$ the estimation error.
  Let $\epsilon_{\text{EST}}\in(0,1)$, and define the good event $\cE$ as the
  intersection of the following events:
  \begin{align}\label{eq: good events}
    \begin{split}
    &\cE_{g}\coloneqq \Big\{ 
      \sup\limits_{r\in\fR} \Big|\dotp{\widehat{d}^{\pi^i},r}-
      J^{\pi^i}(r;p)\Big|\le\epsilon_{\text{EST}}/2\quad\forall i\in\{1,2\}
    \Big\},\\
    &\cE_{\text{D}}\coloneqq\Big\{ 
      \sup\limits_{r\in\fR} \Big|\dotp{\widehat{d}^{\pi_{\text{D},i}},r}-
      J^{\pi_{\text{D},i}}(r;p_{\text{D},i})\Big|\le\epsilon_{\text{EST}}/2\quad\forall i\in\dsb{m_{\text{D}}}
    \Big\},\\ 
    &\cE_{\text{D},*}\coloneqq\Big\{ 
      \sup\limits_{r\in\fR} \Big|\max\nolimits_\pi J^\pi(r;\widehat{p}_{\text{D},i})-
      J^{*}(r;p_{\text{D},i})\Big|\le\epsilon_{\text{EST}}/2\quad\forall i\in\dsb{m_{\text{D}}}
    \Big\},\\
    &\cE_{\text{PC},1}\coloneqq\Big\{ 
      \sup\limits_{r\in\fR} \Big|\dotp{\widehat{d}^{\pi^1_{\text{PC},i}}, r}-
      J^{\pi^1_{\text{PC},i}}(r;p_{\text{PC},i})\Big|\le\epsilon_{\text{EST}}/2\quad\forall i\in\dsb{m_{\text{PC}}}
    \Big\},\\ 
    &\cE_{\text{PC},2}\coloneqq\Big\{ 
      \sup\limits_{r\in\fR} \Big|\dotp{\widehat{d}^{\pi^2_{\text{PC},i}}, r}-
      J^{\pi^2_{\text{PC},i}}(r;p_{\text{PC},i})\Big|\le\epsilon_{\text{EST}}/2\quad\forall i\in\dsb{m_{\text{PC}}}
    \Big\}.
  \end{split}
  \end{align} 

  Then, conditioning on $\cE$, Lemma \ref{lemma: estimation error} allows to
  bound the estimation error, while Lemma \ref{lemma: approximation error}
  permits to bound the iteration error.
  
  The result follows by applying Lemma \ref{lemma: concentration} with
  $\epsilon_{\text{EST}}< \frac{\xi\epsilon}{10H}$.
\end{proof}

\subsubsection{Concentration}

\begin{lemma}[Concentration]\label{lemma: concentration} Let
  $\epsilon_{\text{EST}}\in(0,1),\delta\in(0,1)$. Then, the good event
  $\cE\coloneqq \cE_{g}\cap\cE_{\text{D}}\cap\cE_{\text{D},*}\cap
  \cE_{\text{PC},1}\cap\cE_{\text{PC},2}$,
  where these events are defined in Eq. \eqref{eq: good events}, holds with
  probability at least $1-\delta$, with at most:
  \begin{align}
    \begin{split}
    \label{eq: sample complexity lemma concentration}
    &n^1,n^2,n_{\text{D},i},n^1_{\text{PC},i}, n^2_{\text{PC},i}\le \widetilde{\cO}\Big(
      \frac{SAH^3}{\epsilon^2_{\text{EST}}} \log\frac{m_{\text{PC}}+m_{\text{D}}+2}{\delta}  
    \Big),\\
    &N_{\text{D},i}\le \widetilde{\cO}\Big(
      \frac{SAH^3}{\epsilon^2_{\text{EST}}}\Big(S+\log\frac{m_{\text{PC}}+m_{\text{D}}+2}{\delta}\Big)
      \Big).
    \end{split}
  \end{align}
\end{lemma} 
\begin{proof}
  \rob estimates the visit distribution of the considered policies through their
  empirical estimates, defined in Eq. \eqref{eq: empirical estimates visit
  distributions}. Thus, events $\cE_{g}, \cE_{\text{D}},
  \cE_{\text{PC},1}\cap\cE_{\text{PC},2}$ can be shown to hold using the same
  proof as in Theorem 5.1 of \cite{lazzati2024compatibility} (see also
  \cite{Shani2021OnlineAL}), and a union bound.
  
  Regarding event $\cE_{\text{D},*}$, we can use the guarantees in
  \cite{menard2021fast} (see also \cite{lazzati2024compatibility}).
  
  The result follows by an application of the union bound.

  We remark that, of course, in case $m_{\text{PC}}=m_{\text{D}}=0$, then the
  good event holds with a number of samples:
  \begin{align*}
    &n_{\text{D},i},n^1_{\text{PC},i}, n^2_{\text{PC},i},N_{\text{D},i}=0,\\
    &n^1,n^2\le \widetilde{\cO}\Big(
      \frac{SAH^3}{\epsilon^2_{\text{EST}}} \log\frac{1}{\delta}  
    \Big).
  \end{align*}
  Nevertheless, for simplicity of presentation, we will mention only the bound
  in Eq. \eqref{eq: sample complexity lemma concentration}. 
\end{proof}

\subsubsection{Bounding the \emph{Estimation Error}}

\begin{restatable}[Estimation error]{lemma}{lemmaestimationerror}\label{lemma:
estimation error}
  Let $\epsilon\in(0,2H],\delta\in(0,1)$.
  Make Assumption \ref{ass: slater condition}, and assume that event $\cE$ holds
  with $\epsilon_{\text{EST}}< \frac{\xi\epsilon}{10H}$.
  Then, we have that:
  \begin{align*}
    |M-\widehat{M}|+|m-\widehat{m}|\le\epsilon.
  \end{align*} 
%
%
\end{restatable}
\begin{proof}
  We proceed for $M,\widehat{M}$ (the proof for $m,\widehat{m}$ is completely
  analogous):
  \begin{align*}
    \Big|M-\widehat{M}\Big|
    &\markref{(1)}{=}\Big|\min\limits_{\lambda\in\fD_-}\max\limits_{r\in\fR}L(r,\lambda)-
    \min\limits_{\lambda\in\fD_-}\max\limits_{r\in\fR}\widehat{L}(r,\lambda)\Big|\\
    &\markref{(2)}{\le}\popblue{\max\limits_{\lambda\in\fD_-}}\Big|\max\limits_{r\in\fR}L(r,\lambda)-
    \max\limits_{r\in\fR}\widehat{L}(r,\lambda)\Big|\\
    &\markref{(3)}{\le}\max\limits_{\lambda\in\fD_-}\popblue{\max\limits_{r\in\fR}}\Big|L(r,\lambda)-
    \widehat{L}(r,\lambda)\Big|\\
    &=\max\limits_{\lambda\in\fD_-}\max\limits_{r\in\fR}\Big|
    \dotp{\widehat{d}^{\pi^1}-\widehat{d}^{\pi^2},r}-
    \dotp{d^{\pi^1}-d^{\pi^2},r}\\
    &\qquad +
    \sum\nolimits_i\lambda_{\text{D}}^i
 \Bigr{\bigr{\max\nolimits_\pi J^\pi(r;\widehat{p}_{\text{D},i})-\dotp{\widehat{d}^{\pi_{\text{D},i}},r}}
 -\max_\pi J^\pi(r;p_{\text{D},i})-\dotp{d^{\pi_{\text{D},i}},r}}\\
 &\qquad
 +\sum\nolimits_i\lambda_{\text{PC}}^i \Bigr{\dotp{\widehat{d}^{\pi^1_{\text{PC},i}}
 -\widehat{d}^{\pi^2_{\text{PC},i}}, r}-
 \dotp{d^{\pi^1_{\text{PC},i}}-d^{\pi^2_{\text{PC},i}}, r}}\\
  &\markref{(4)}{\le}\popblue{\epsilon_{\text{EST}}}\Bigr{1+\max\limits_{\lambda\in\fD_-}
  \|{\lambda}\|_1}\\
  &\markref{(5)}{\le}\epsilon_{\text{EST}}\popblue{\frac{5H}{\xi}},
  \end{align*}
  where at (1) we have applied Lemma \ref{lemma: strong duality true problem}
  and Lemma \ref{lemma: strong duality estimated problem}, at (2) and at (3) the
  Lipschitzianity of the maximum operator. At (4) we use triangle's inequality,
  that event $\cE$ holds and we upper bound with the 1-norm of $\lambda$, and at
  (5) we use the size of set $\fD_-$ provided by Lemma \ref{lemma: strong
  duality estimated problem}.

  By choosing:
  \begin{align*}
    \epsilon_{\text{EST}}< \frac{\xi\epsilon}{10H},
  \end{align*}
  which is smaller than $\xi/2$ for $\epsilon\in(0,2H]$, as required by Lemma
  \ref{lemma: strong duality estimated problem}, we get the result.
\end{proof}

\begin{lemma}\label{lemma: strong duality true problem}
  Let $L$ be the Lagrangian of the problem:
  \begin{align}\label{eq: true lagrangian}
    L(r,\lambda)&
   =\dotp{d^{\pi^1}-d^{\pi^2},r}
   +\sum\nolimits_i\lambda_{\text{D}}^i
   \bigr{\max\nolimits_\pi J^\pi(r;p_{\text{D},i})-\dotp{d^{\pi_{\text{D},i}},r}
   -t_i}\\
   & \qquad+\sum\nolimits_i\lambda_{\text{TC}}^i\dotp{d^{\omega^1_{\text{TC,i}}}-d^{\omega^2_{\text{TC,i}}}, r}
   +\sum\nolimits_i\lambda_{\text{PC}}^i \dotp{d^{\pi^1_{\text{PC},i}}
   -d^{\pi^2_{\text{PC},i}}, r}.
   \nonumber
 \end{align}
  Then, under Assumption \ref{ass: slater condition}, it holds that:
  \begin{align*}
    &M=\min\limits_{\lambda\in\fD_-}\max\limits_{r\in\fR}L(r,\lambda),\\
    &m=\max\limits_{\lambda\in\fD_+}\min\limits_{r\in\fR}L(r,\lambda),
  \end{align*}
  where $\fD_+,\fD_-$ correspond to $s=4H/\xi$ in Algorithms \ref{alg: pdsm min}
  and \ref{alg: pdsm max}, namely, $\fD_+=\{\lambda\ge0\,|\,
  \|\lambda\|_1\le 4H/\xi\}$ and $\fD_-=\{\lambda\le0\,|\,
  \|\lambda\|_1\le 4H/\xi\}$.
\end{lemma}
\begin{proof}
  Under Assumption \ref{ass: slater condition}, we have that Slater's constraint
  qualification holds. Thanks also to the linearity of the objective function
  and to the convexity of $\cR_\cF$, we have that strong duality holds
  \cite{boyd2004convex}, thus:
  \begin{align*}
    &M=\min\limits_{\lambda\le0}\max\limits_{r\in\fR}L(r,\lambda),\\
    &m=\max\limits_{\lambda\ge0}\min\limits_{r\in\fR}L(r,\lambda).
  \end{align*}
  To prove the boundedness of the Lagrange multipliers corresponding to the
  saddle point, let $(r^*,\lambda^*)$ be any saddle point for problem $M$ (the
  proof for $m$ is analogous). Under Assumption \ref{ass: slater condition}, we
  can apply Lemma 3 of \cite{nedic2009subgradient} to obtain that, for the
  reward $\overline{r}$ in Assumption \ref{ass: slater condition}:
  \begin{align*}
    \|\lambda^*\|_1&\le \frac{
    \dotp{d^{\pi^1}-d^{\pi^2},\overline{r}} -M}{\xi}\\
    &\le \frac{
      \dotp{d^{\pi^1}-d^{\pi^2},\overline{r}}+H}{\xi}\\
    &\le \frac{
      2H}{\xi}.
  \end{align*}
  Thus, the values of the Lagrange multipliers in saddle points can be found in
  bounded sets $\fD_+,\fD_-$.
  Note that we upper bound again with $4H/\xi$ to comply with the bound for the
  estimated problem in Lemma \ref{lemma: strong duality estimated problem}.
\end{proof}

\begin{lemma}\label{lemma: strong duality estimated problem}
  Make Assumption \ref{ass: slater condition}. Under the good event $\cE$ with
  $\epsilon_{\text{EST}}\in(0,\xi/2)$, we have:
  \begin{align*}
    &\widehat{M}=\min\limits_{\lambda\in\fD_-}\max\limits_{r\in\fR}\widehat{L}(r,\lambda),\\
    &\widehat{m}=\max\limits_{\lambda\in\fD_+}\min\limits_{r\in\fR}\widehat{L}(r,\lambda),
  \end{align*}
  where $\fD_-=\{\lambda\le 0\,|\, \|\lambda\|_1\le 4H/\xi\}$ and
  $\fD_+=\{\lambda\ge 0\,|\, \|\lambda\|_1\le 4H/\xi\}$.
\end{lemma}
\begin{proof}
  The proof is analogous to that of Lemma \ref{lemma: strong duality true
  problem} as long as we use Lemma \ref{lemma: slater condition estimated
  problem}.
\end{proof}

\begin{lemma}[Slater's Condition]\label{lemma: slater condition estimated problem}%
  Make Assumption \ref{ass: slater condition}. Under the good event $\cE$ with
  $\epsilon_{\text{EST}}\in(0,\xi/2)$, for the reward $\overline{r}$ in
  Assumption \ref{ass: slater condition}, we have:
  \begin{align*}
    \begin{cases}
      \max_\pi J^\pi(\overline{r};\widehat{p}_{\text{D},i})-
      \dotp{\widehat{d}^{\pi_{\text{D},i}},\overline{r}}
  -t_i\le-\frac{\xi}{2} & \forall i\in\dsb{m_{\text{D}}}\\
  \dotp{d^{\omega^1_{\text{TC,i}}}-d^{\omega^2_{\text{TC,i}}}, \overline{r}}\le-\frac{\xi}{2} & \forall i\in\dsb{m_{\text{PC}}}\\
  \dotp{\widehat{d}^{\pi^1_{\text{PC},i}}
  -\widehat{d}^{\pi^2_{\text{PC},i}}, \overline{r}}\le -\frac{\xi}{2} & \forall i\in\dsb{m_{\text{TC}}}
    \end{cases}.
  \end{align*}
\end{lemma}
\begin{proof}
  The result follows directly by using the conditions in the good event $\cE$.
  For instance, concerning the policy comparison feedback, we have:
  \begin{align*}
    \dotp{\widehat{d}^{\pi^1_{\text{PC},i}}
  -\widehat{d}^{\pi^2_{\text{PC},i}}, \overline{r}}&=
  \dotp{\widehat{d}^{\pi^1_{\text{PC},i}}
  -\widehat{d}^{\pi^2_{\text{PC},i}}, \overline{r}} \popblue{\pm
  \bigr{\dotp{d^{\pi^1_{\text{PC},i}}
  -d^{\pi^2_{\text{PC},i}}, \overline{r}}}}\\
      &\le \popblue{\epsilon_{\text{EST}}}+ \bigr{\dotp{d^{\pi^1_{\text{PC},i}}
      -d^{\pi^2_{\text{PC},i}}, \overline{r}}}\\
      &\le \epsilon_{\text{EST}}-\xi.
  \end{align*}
  By using that $\epsilon_{\text{EST}}\in(0,\xi/2)$ we get the result.
\end{proof}

\subsubsection{Bounding the \emph{Iteration Error}}

\begin{restatable}[Iteration error]{lemma}{lemmaapproximationerror} 
  \label{lemma: approximation error}%
  Let $\epsilon\in(0,2H]$, make Assumption \ref{ass: slater condition} and
  assume that event $\cE$ with $\epsilon_{\text{EST}}\in(0,\xi/2)$ holds. Then,
  if we choose:
  \begin{align*}
    &s=4H/\xi+\sqrt{(4H/\xi)^2+SAH/4},\\
    &\alpha=\epsilon/(16H (1+s\sqrt{m_{\text{D}}+m_{\text{TC}}+m_{\text{PC}}})^2),
  \end{align*}
   we have that:
  \begin{align*}
    |\widehat{M}-\widehat{M}_K|+
    |\widehat{m}-\widehat{m}_K|\le\epsilon,
  \end{align*}
  with a number of iterations:
  \begin{align*}
    K\le\cO\Bigr{
      \frac{H^{5/2}}{\xi\epsilon^2}\Bigr{\sqrt{SA}+\frac{\sqrt{H}}{\xi}}
      \Bigr{1+\sqrt{H(m_{\text{D}}+m_{\text{TC}}+m_{\text{PC}})(H/\xi^2+SA)}}^2
    }.
  \end{align*}
\end{restatable}
\begin{proof} 
  We prove the result by imposing that both the terms
  $|\widehat{M}-\widehat{M}_K|$ and $|\widehat{m}-\widehat{m}_K|$ are smaller
  than $\epsilon/2$.
  We present the proof only for $|\widehat{M}-\widehat{M}_K|$, because the proof
  for $|\widehat{m}-\widehat{m}_K|$ is completely analogous. For simplicity,
  define $\ell$ as:
  \begin{align*}
    \ell\coloneqq \sqrt{\Big(\frac{4H}{\xi}\Big)^2+\Bigr{\frac{SAH}{4}}},
  \end{align*}
  and note that, in this way, for the choice of $s$ in the statement, we have:
  \begin{align*}
    s=\frac{4H}{\xi}+\ell.
  \end{align*}

  Thanks to Lemma \ref{lemma: bound norm subgradients}, using this choice of
  $s$, we obtain the following upper bound to the norm of the subgradients, that
  we call $U$:
  \begin{align*} 
    U\coloneqq 2\sqrt{H}\bigr{1+s\sqrt{|\cF|}}=
    2\sqrt{H} +2\Big(\frac{4H}{\xi}+\ell\Big)\sqrt{H|\cF|}.
 \end{align*}
 Observe that the choice of $\alpha$ corresponds to:
 \begin{align*}
  \alpha=\frac{\epsilon}{4U^2}.
 \end{align*}

Using Lemma \ref{lemma: slater condition estimated problem} and Lemma
\ref{lemma: bound norm subgradients}, we can apply Proposition 2 of
\cite{nedic2009subgradient} to obtain (recall that we start with
$r_0=0,\lambda_0=0$):
\begin{align}
  &\widehat{M}_K-\widehat{M}
  \le \frac{\|\widehat{r}_K\|_2^2}{2K\alpha}
  +\alpha U^2,\label{eq: term1}\\
  &\widehat{M}-\widehat{M}_K\le 
  \frac{4H}{\xi}\Big(
    \frac{2}{K\alpha\ell}\Big(\frac{4H}{\xi}+\ell\Big)^2
    +\frac{\|\widehat{r}_K\|^2_2}{2K\alpha\ell}+ \frac{\alpha U^2}{2\ell}
  \Big)\label{eq: term2}.
\end{align} 
We now bound each of these terms.

Concerning the first term, i.e., Eq. \eqref{eq: term1}, using that
$\widehat{r}_K\in\fR$ and the choice $\alpha=\epsilon/(4U^2)$ we have:
\begin{align*}
  \widehat{M}_K-\widehat{M}
  &\le\frac{SAH}{2K\alpha}
  +\alpha U^2,\\ 
  &\le\frac{2SAHU^2}{\epsilon K}
  +\frac{\epsilon}{4},
\end{align*}   
which is smaller than $\epsilon/4$ if: 
\begin{align*}
  K\ge \frac{8SAHU^2}{\epsilon^2}.
\end{align*} 
Regarding the second term, namely, Eq. \eqref{eq: term2}, we write: 
\begin{align*}
  \widehat{M}-\widehat{M}_K&\le 
  \frac{4H}{\xi}
    \frac{2}{K\alpha\ell}\Big(\frac{4H}{\xi}+\ell\Big)^2
    +\frac{4H}{\xi}\frac{SAH}{2K\alpha\ell}+ \frac{4H}{\xi}\frac{\alpha U^2}{2\ell}\\ 
    &\markref{(1)}{\le}\frac{4H}{\xi}
    \frac{2}{K\alpha\ell}\Big(\frac{4H}{\xi}+\ell\Big)^2
    +\frac{4H}{\xi}\frac{SAH}{2K\alpha\ell}+\popblue{\frac{\epsilon}{4}}\\ 
    &\markref{(2)}{\le}\frac{4H}{\xi}
    \frac{2}{K\alpha\ell}\Big(\frac{4H}{\xi}+\ell\Big)^2
    +\frac{4H}{\xi}\popblue{\frac{\sqrt{SAH}}{K\alpha}}+{\frac{\epsilon}{4}}\\ 
    &\markref{(3)}{\le}\frac{4H}{\xi}\frac{2}{K\alpha}\popblue{\Bigr{
      \frac{16H}{\xi}+\frac{\sqrt{SAH}}{2}}}
    +\frac{4H}{\xi}{\frac{\sqrt{SAH}}{K\alpha}}+{\frac{\epsilon}{4}}\\ 
    &\le \frac{4H}{\xi}\frac{2}{K\alpha}\Bigr{\frac{16H}{\xi}+\sqrt{SAH}} +\frac{\epsilon}{4}\\
    &=\frac{8H^{3/2}}{\xi K\alpha}\Bigr{\frac{16\sqrt{H}}{\xi}+\sqrt{SA}} +\frac{\epsilon}{4}\\
    &\markref{(4)}{=}
    \frac{32U^2H^{3/2}}{\xi K\epsilon}\Bigr{\frac{16\sqrt{H}}{\xi}+\sqrt{SA}} +\frac{\epsilon}{4},
\end{align*}  
where at (1) we use that $\ell\ge 2H/\xi$ and the choice of
$\alpha=\epsilon/(4U^2)$, at (2) we use that $\ell\ge \sqrt{\frac{SAH}{4}}$, at
(3) we use that $\Big(\frac{4H}{\xi}+\ell\Big)^2/\ell\le
\frac{16H}{\xi}+\frac{\sqrt{SAH}}{2}$ and at (4) we use that
$\alpha=\epsilon/(4U^2)$.

The resulting quantity is smaller than $\epsilon/2$ if:
\begin{align*}
  \frac{32U^2H^{3/2}}{\xi K\epsilon}\Bigr{\frac{16\sqrt{H}}{\xi}+\sqrt{SA}}
  \le\frac{\epsilon}{4}\quad\iff\quad 
     K\ge \frac{128H^{3/2}U^2}{\xi \epsilon^2}\Bigr{\sqrt{SA}  
     + \frac{16\sqrt{H}}{\xi}} .
\end{align*}  
By inserting the definition of $U$, we get:
\begin{align*}
  K&\ge \frac{128H^{3/2}U^2}{\xi \epsilon^2}\Bigr{\sqrt{SA}  
     + \frac{16\sqrt{H}}{\xi}}\\
    &=\frac{128H^{3/2}}{\xi \epsilon^2}\Bigr{\sqrt{SA}  
    + \frac{16\sqrt{H}}{\xi}}\biggr{2\sqrt{H} +2\Big(\frac{4H}{\xi}+\ell\Big)\sqrt{H|\cF|}}^2\\
    &\le \cO\biggr{
      \frac{H^{5/2}}{\xi\epsilon^2}\Bigr{\sqrt{SA}+\frac{\sqrt{H}}{\xi}}
      \Bigr{1+\sqrt{H|\cF|(H/\xi^2+SA)}}^2
    }.
\end{align*}

The result follows by noting that this quantity is larger than the upper bound
derived for Eq. \eqref{eq: term1}.

\end{proof}

\begin{lemma}\label{lemma: bound norm subgradients}
  If we execute \texttt{PDSM-MIN} or \texttt{PDSM-MAX} for $K\ge 0$ iterations
  using any choice of $s\ge 0$, then it holds that:
  \begin{align*}
    \max\limits_{k\le K}\max\Bigc{\|\partial_r\widehat{L}(r_k,
    \lambda_k)\|_2,
    \|\partial_\lambda\widehat{L}(r_k,\lambda_k)\|_2}\le
    2\max\Bigc{\sqrt{H}\bigr{1+s\sqrt{|\cF|}},
    H\sqrt{|\cF|}
    },
  \end{align*}
  where we denoted $|\cF|\coloneqq m_{\text{D}}+m_{\text{TC}}+m_{\text{PC}}.$
\end{lemma}
\begin{proof}
  We prove the result for \texttt{PDSM-MAX} only, because for \texttt{PDSM-MIN}
  the proof is analogous. Consider the subgradient expressions in Eq. \eqref{eq:
  subgradients estimated problem}. We have (the dependence on $s$ is implicit in
  $\fD_-$):
  \begin{align*}
    \max\limits_{k\le K}\|\partial_r \widehat{L}(r_k,\lambda_k)\|_2
    &\markref{(1)}{\le}\popblue{\max\limits_{r\in\fR}\max\limits_{\lambda\in\fD_{-}}}
    \|\partial_r \widehat{L}(r,\lambda)\|_2\\
    &=\max\limits_{r\in\fR}\max\limits_{\lambda\in\fD_{-}}
\Big\|(\widehat{d}^{\pi^1}-\widehat{d}^{\pi^2})+
\sum\nolimits_i\lambda_{\text{D}}^i
\bigr{d^{\widehat{\pi}_{\text{D},i}^{r,*}}-\widehat{d}^{\pi_{\text{D},i}}}\\
&\qquad+
\sum\nolimits_i\lambda_{\text{TC}}^i\bigr{d^{\omega^1_{\text{TC,i}}}-d^{\omega^2_{\text{TC,i}}}}
+\sum\nolimits_i\lambda_{\text{PC}}^i \bigr{\widehat{d}^{\pi^1_{\text{PC},i}}
-\widehat{d}^{\pi^2_{\text{PC},i}}}
\Big\|_2\\
&\le\max\limits_{r\in\fR}\max\limits_{\lambda\in\fD_{-}}
\Big\|\widehat{d}^{\pi^1}-\widehat{d}^{\pi^2}\Big\|_2+
    \sum\nolimits_i |\lambda_{\text{D}}^i|
    \Big\|d^{\widehat{\pi}_{\text{D},i}^{r,*}}-\widehat{d}^{\pi_{\text{D},i}}\Big\|_2\\
  &\qquad
    + \sum\nolimits_i|\lambda_{\text{TC}}^i|\Big\|d^{\omega^1_{\text{TC,i}}}-
    d^{\omega^2_{\text{TC,i}}}\Big\|_2
    +\sum\nolimits_i|\lambda_{\text{PC}}^i|
    \Big\|\widehat{d}^{\pi^1_{\text{PC},i}}
    -\widehat{d}^{\pi^2_{\text{PC},i}}
\Big\|_2\\
&\markref{(2)}{\le}
\popblue{2\sqrt{H}}\max\limits_{\lambda\in\fD_{-}}\bigr{1+\|\lambda\|_1}\\
&\markref{(3)}{\le} 2\sqrt{H}\Bigr{1+\popblue{\sqrt{m_{\text{D}}+m_{\text{TC}}+m_{\text{PC}}}}
\max\limits_{\lambda\in\fD_{-}}\|\lambda\|_{\popblue{2}}}\\
&=2\sqrt{H}\biggs{1+\popblue{s}\sqrt{m_{\text{D}}+m_{\text{TC}}+m_{\text{PC}}}},
  \end{align*}
  where at (1) we use that the PDSM projects onto those sets, thus
  $r_k,\lambda_k$ cannot lie outside, at (2) we use the fact that the 2-norm of
  every visit distribution is at most $\sqrt{H}$, at (3) we use the relationship
  between the 1-norm and the 2-norm, recalling that $\lambda$ is
  $(m_{\text{D}}+m_{\text{TC}}+m_{\text{PC}})$-dimensional.

  Regarding the subgradients w.r.t. $\lambda$, we can write:
  \begin{align*}
    \max\limits_{k\le K}\|\partial_\lambda \widehat{L}(r_k,\lambda_k)\|_2
    &\markref{(4)}{\le}\popblue{\max\limits_{r\in\fR}\max\limits_{\lambda\in\fD_{-}}}
    \|\partial_\lambda \widehat{L}(r,\lambda)\|_2\\
    &\markref{(5)}{=} \max\limits_{r\in\fR}
    \Big\|\Bigs{
    \partial_{\lambda^i_{\text{D}}}\widehat{L}(r,\lambda),\;
    \dotsc,
    \partial_{\lambda^i_{\text{TC}}}\widehat{L}(r,\lambda),\;
    \dotsc,
    \partial_{\lambda^i_{\text{PC}}}\widehat{L}(r,\lambda)
    }
    \Big\|_2\\
    &\markref{(6)}{\le}
    \Big\|\Bigs{
    2H,\;\dotsc,
    H,\;\dotsc,
    H
    }
    \Big\|_2\\
    &\le
    2H\sqrt{m_{\text{D}}+m_{\text{TC}}+m_{\text{PC}}},
  \end{align*}
  where at (4) we use that the PDSM projects onto those sets, thus
  $r_k,\lambda_k$ cannot lie outside, at (5) we note that the subgradient vector
  is $(m_{\text{D}}+m_{\text{TC}}+m_{\text{PC}})$-dimensional, and write down
  its components, noting also that there is no dependence on $\lambda$. At (6)
  we use the formulas in Eq. \eqref{eq: subgradients estimated problem} and note
  that the difference between returns and expected returns is bounded in
  $[-H,+H]$ because $r\in\fR$, and also that $t_i\in[0,H]$ for all
  $i\in\dsb{m_{\text{D}}}$ by hypothesis.
  
  The result follows by joining the two upper bounds.
\end{proof}

\subsection{Other Proofs and Results}
\label{apx: other proofs and results}

\begin{lemma}\label{lemma: change objective}
  Let $d\in\Nat$, and let $\cY$ be a subset of $\RR^d$. Let $h: \cY \to\RR$ be
  any function that attains both minimum and maximum in $\cY$, and define
  $M_h\coloneqq \max_{y\in\cY}h(y)$ and $m_h\coloneqq \min_{y\in\cY}h(y)$. Then:
  \begin{align*}
    & \frac{M_h-m_h}{2}=\min\limits_{k\in\RR}\max\limits_{y\in\cY}\Big|k-h(y)\Big|,\\
    &\frac{M_h+m_h}{2}=\argmin\limits_{k\in\RR}\max\limits_{y\in\cY}\Big|k-h(y)\Big|,\\
    &\max\limits_{y\in\cY}\Big|k-h(y)\Big|= \max\Bigc{k-m_h,
    M_h-k}\quad\forall k\in\RR.
  \end{align*}  
\end{lemma}
\begin{proof}
  For any $k\in\RR$, we can write:
  \begin{align*}
    \max\limits_{y\in\cY}\Big|k-h(y)\Big|
    &=\max\limits_{y\in\cY}\max\Bigc{k-h(y),
    h(y)-k}\\
    &=\max\Bigc{\max\limits_{y\in\cY}k-h(y),
    \max\limits_{y\in\cY}h(y)-k}\\
    &=\max\Bigc{k-\min\limits_{y\in\cY}h(y),
    \max\limits_{y\in\cY}h(y)-k}\\
    &=\max\Bigc{k-m_h, M_h-k}.
  \end{align*}

  Then, we know that the maximum between two items is minimized when the two
  items coincide. Thus, the argmin must satisfy:
  \begin{align*}
    k-m_h = M_h-k\quad\iff\quad k=\frac{M_h+m_h}{2}.
  \end{align*}

  Substituting this value into the previous expression, we get:
  \begin{align*}
    \max\Bigc{\frac{M_h+m_h}{2}-m_h, M_h-\frac{M_h+m_h}{2}}=\frac{M_h-m_h}{2}.
  \end{align*}

  This concludes the proof.
\end{proof}



\rewriteobjectivesusecase*
\begin{proof}
  Apply Lemma \ref{lemma: change objective}.
\end{proof}

\uninf*
\begin{proof}
  Apply Lemma \ref{lemma: change objective}.
\end{proof}

\subsubsection{Using \rob for Estimating the Worst-Case Loss}
\label{apx: bounding worst case error with rob}

Thanks to Lemma \ref{lemma: change objective}, we realize that the worst-case
loss suffered by any object $x\in\cX_g$ can be equivalently computed as:
\begin{align*}
  \max\limits_{r\in\cR_\cF}\cL_g(r,x)=\max\Bigc{x-m,M-x}.
\end{align*}
This suggests that we can use \rob to quantify the error suffered in the worst
case by arbitrary $x\in\cX_g$, potentially the output provided by other ReL
algorithms. To do so, we can use the estimate:
\begin{align*}
  \max\Bigc{x-\widehat{m}_K,\widehat{M}_K-x},
\end{align*}
and note that:
\begin{align*}
  \Big|\max\Bigc{x-\widehat{m}_K,\widehat{M}_K-x}&-\max\Bigc{x-m,M-x}\Big|\le
  \max\Bigc{|m-\widehat{m}_K|,|M-\widehat{M}_K|}\\
  &\le |m-\widehat{m}_K|+|M-\widehat{M}_K|\\
  &\le (|m-\widehat{m}|+|M-\widehat{M}|)
  +(|\widehat{m}_K-\widehat{m}|+|\widehat{M}_K-\widehat{M}|),
\end{align*}
namely, the error can be upper bounded by the estimation and iteration errors
considered by Theorem \ref{thr: guarantees caty}. Therefore, we have the
following corollary:
\begin{coroll}
  Let $x\in\cX_g$ be arbitrary. Under the setting of Theorem \ref{thr:
guarantees caty}, with the number of samples and iterations specified in the
statement of the theorem, with probability $1-\delta$, it holds that:
\begin{align*}
  \Big|\max\limits_{r\in\cR_\cF}\cL_g(r,x)-
  \max\Bigc{x-\widehat{m}_K,\widehat{M}_K-x}\Big|\le 2\epsilon.
\end{align*}
\end{coroll}

\subsection{Other ReL Problems}\label{apx: other kinds feedback}

The algorithm presented in Section \ref{sec: use case} can be straightforwardly
extended, along with its theoretical guarantees, to consider other kinds of
feedback and applications that preserve the convexity of the problem.

For instance, if we replace the application $g$ with that of assessing a
preference between two trajectories (see Table \ref{table: application}), then
it is clear that the scheme of the algorithm does not change, but we just have
to modify the computation of the subgradients.

As another example, we might consider demonstrations from ``bad'' policies $\pi$,
i.e., demonstrations from policies whose performance is almost the worst
possible:
\begin{align*}
  J^\pi(r^\star;p)\le\min_{\pi'}J^{\pi'}(r^\star;p)+t,
\end{align*}
for some $t>0$. Note that the reward-free exploration algorithm of
\cite{menard2021fast} works also in this setting, and its theoretical guarantees
can be easily adapted to this setting.

We mention that we could also consider as feedback a \emph{fractional}
comparison of policies $\pi^1,\pi^2$ (or trajectories $\omega^1,\omega^2$):
\begin{align*}
  J^{\pi^1}(r^\star;p)\ge\alpha J^{\pi^2}(r^\star;p),
\end{align*}
for some $\alpha\in(0,1]$.

\section{Experimental Details}\label{apx: experimental details}

We provide here additional details on the simulations conducted. Note that the
experiment has been conducted on a laptop with processor ``AMD Ryzen 5 5500U
with Radeon Graphics 2.10 GHz'' with 8GB of RAM, and required a few seconds for
execution.

\subsection{Target Environment, Target Reward $r^\star$, Application $g$}

For the application $g$, we considered the environment reported in Figure
\ref{fig: road} (left) and described in Section \ref{sec: use case}, with
initial state the C lane, and the stationary transition model $p$ described
below. Note that $p$ depends only on the lane and the action played, thus, with
abuse of notation, we write $p$ as:
\begin{align*}
  &p_h(\cdot|L,a_L)=\begin{cases}
    1, \text{if }\cdot = L\\
    0, \text{if }\cdot = C\\
    0, \text{if }\cdot = R\\
  \end{cases}
  &&p_h(\cdot|C,a_L)=\begin{cases}
    0.6, \text{if }\cdot = L\\
    0.4, \text{if }\cdot = C\\
    0, \text{if }\cdot = R\\
  \end{cases}
  p_h(\cdot|R,a_L)=\begin{cases}
    0, \text{if }\cdot = L\\
    0.6, \text{if }\cdot = C\\
    0.4, \text{if }\cdot = R\\
  \end{cases}\\
  &p_h(\cdot|L,a_C)=\begin{cases}
    \scalebox{0.85}{0.55}, \text{if }\cdot = L\\
    \scalebox{0.85}{0.45}, \text{if }\cdot = C\\
    0, \text{if }\cdot = R\\
  \end{cases}
  &&p_h(\cdot|C,a_C)=\begin{cases}
    \scalebox{0.85}{0.3}, \text{if }\cdot = L\\
    \scalebox{0.85}{0.4}, \text{if }\cdot = C\\
    \scalebox{0.85}{0.3}, \text{if }\cdot = R\\
  \end{cases}
  p_h(\cdot|R,a_C)=\begin{cases}
    0, \text{if }\cdot = L\\
    \scalebox{0.85}{0.45}, \text{if }\cdot = C\\
    \scalebox{0.85}{0.55}, \text{if }\cdot = R\\
  \end{cases}\\
  &p_h(\cdot|L,a_R)=\begin{cases}
    0.3, \text{if }\cdot = L\\
    0.7, \text{if }\cdot = C\\
    0, \text{if }\cdot = R\\
  \end{cases}
  &&p_h(\cdot|C,a_R)=\begin{cases}
    0, \text{if }\cdot = L\\
    0.3, \text{if }\cdot = C\\
    0.7, \text{if }\cdot = R\\
  \end{cases}
  p_h(\cdot|R,a_R)=\begin{cases}
    0, \text{if }\cdot = L\\
    0, \text{if }\cdot = C\\
    1, \text{if }\cdot = R\\
  \end{cases}
\end{align*}
Intuitively, action $a_L$ moves to the left w.p. $0.6$, and keeps the lane w.p.
$0.4$, except when it is on the left lane, where it keeps the lane. Action $a_R$
is analogous but with a different bias. Instead, action $a_C$ keeps the lane
w.p. $0.4$, and moves to the left or to the right w.p. $0.3$. When it is on the
borders, it cannot move in a certain direction, thus the remaining probability
is splitted equally in the other two lanes.
The target reward $r^\star$ considered is described in Section \ref{sec: use
case}, and is shown in Figure \ref{fig: road}, on the right.
\begin{figure}[!t]
  \centering
  \begin{minipage}[t!]{0.49\textwidth}
    \centering
    \includegraphics[width=0.95\linewidth]{map.pdf}
\end{minipage}
\begin{minipage}[t!]{0.49\textwidth}
    \centering
    \includegraphics[width=0.95\linewidth]{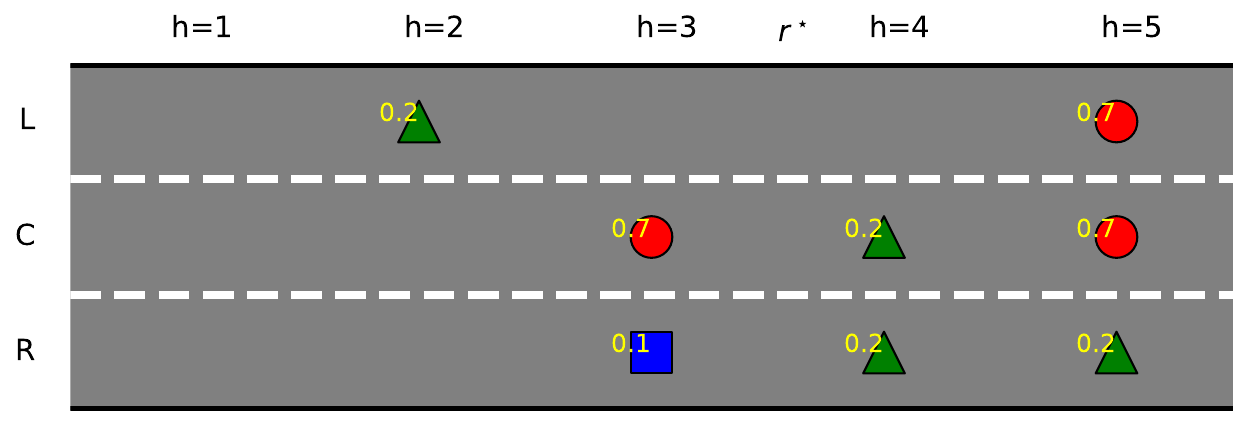}
\end{minipage}
\caption{ (Left) The target environment considered in the
experiment.
(Right) Representation of $r^\star$ for the target environment.}
\label{fig: road}
 \end{figure}
Finally, the occupancy measures $d^{\pi^1},d^{\pi^2}$ describing the
application $g$ arise from the policies described in Section \ref{sec: use case}
and the transition model $p$ presented earlier, and are shown in Figure \ref{fig: appl g}.
\begin{figure}[!t]
  \centering
  \begin{minipage}[t!]{0.49\textwidth}
    \centering
    \includegraphics[width=0.95\linewidth]{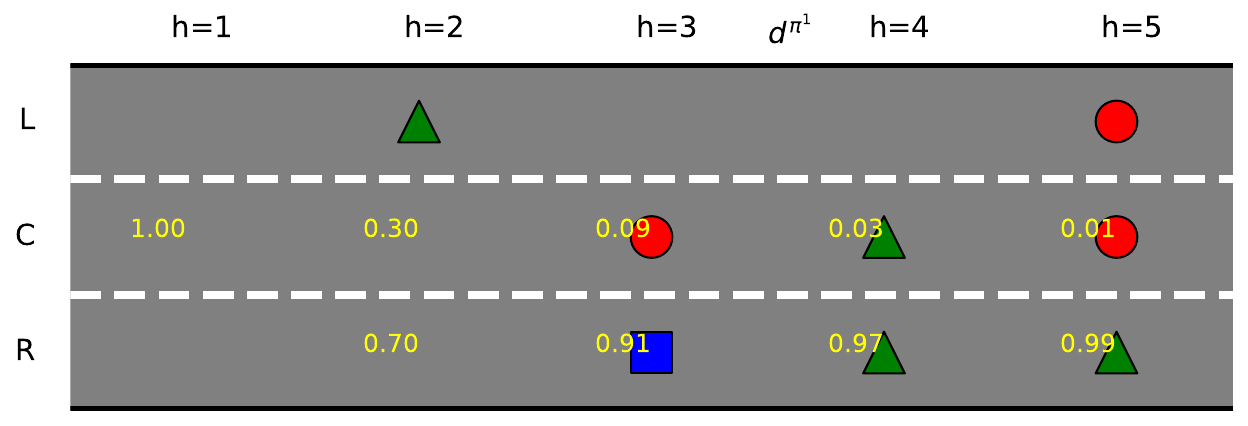}
\end{minipage}
\begin{minipage}[t!]{0.49\textwidth}
    \centering
    \includegraphics[width=0.95\linewidth]{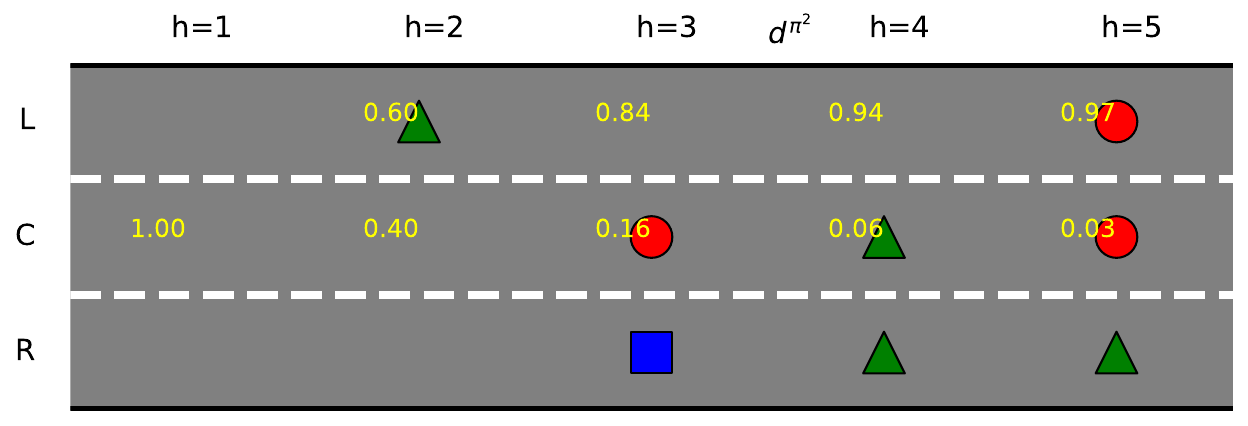}
\end{minipage}
\caption{ (Left) Plot of $d^{\pi^1}$.
(Right) Plot of $d^{\pi^2}$.}
\label{fig: appl g}
 \end{figure}

\subsection{Feedback $\cF$}

\paragraph{Trajectory comparisons.}
We construct three trajectory comparison feedback. The first pair is in Figure
\ref{fig: TC 1}, and we associated $t_1=0.3$ to it. The second pair of
trajectories is in Figure \ref{fig: TC 2}, and we set $t_2=1$. Finally, the
third pair is in Figure \ref{fig: TC 3} and has $t_3=-0.5$.

\paragraph{Comparisons.}
Concerning the comparisons feedback, we considered the new environment shown in
Figure \ref{fig: new maps C D} on the left, keeping the same transition model
$p$ described earlier but using the left $L$ lane as initial state. We compared
the two occupancy measures in Figure \ref{fig: C 1} using $t_1=0$, and the two
occupancy measures in Figure \ref{fig: C 2} using $t_2=0.5$.

\paragraph{Demonstrations.}
For the demonstrations feedback, we adopted the map in Figure \ref{fig: new maps
C D} on the right, preserving the transition model $p$, but using lane R as
initial state. We considered only one feedback, whose policy has the occupancy
measure in Figure \ref{fig: D 1}, to which we associated $t_1=1$.

\begin{figure}[!t]
  \centering
  \begin{minipage}[t!]{0.49\textwidth}
    \centering
    \includegraphics[width=0.95\linewidth]{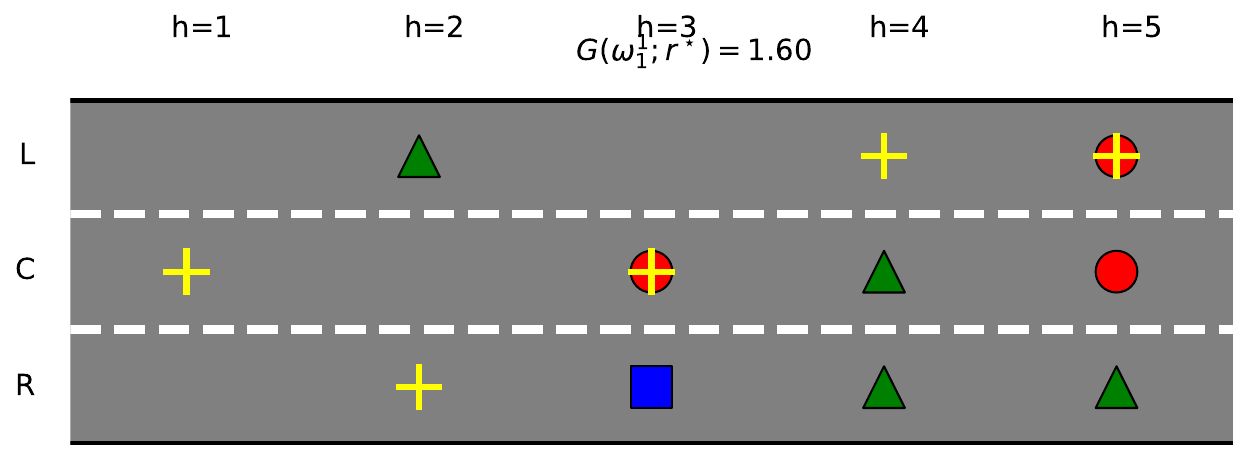}
\end{minipage}
\begin{minipage}[t!]{0.49\textwidth}
    \centering
    \includegraphics[width=0.95\linewidth]{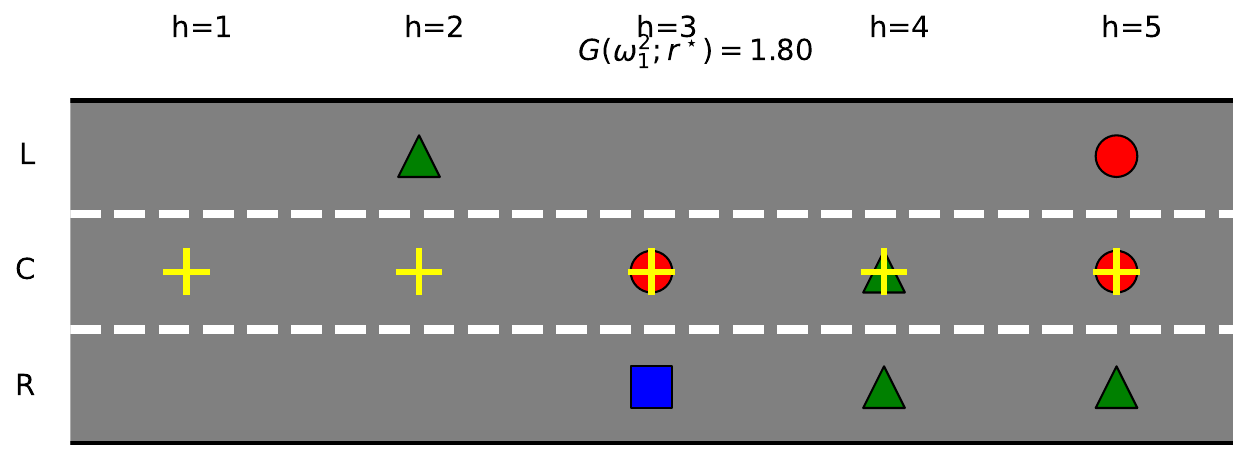}
\end{minipage}
\caption{ The trajectories compared in the first feedback.
$\omega^1_1$ is on the left, and $\omega^2_1$ on the right.}
\label{fig: TC 1}
 \end{figure}

 \begin{figure}[!t]
  \centering
  \begin{minipage}[t!]{0.49\textwidth}
    \centering
    \includegraphics[width=0.95\linewidth]{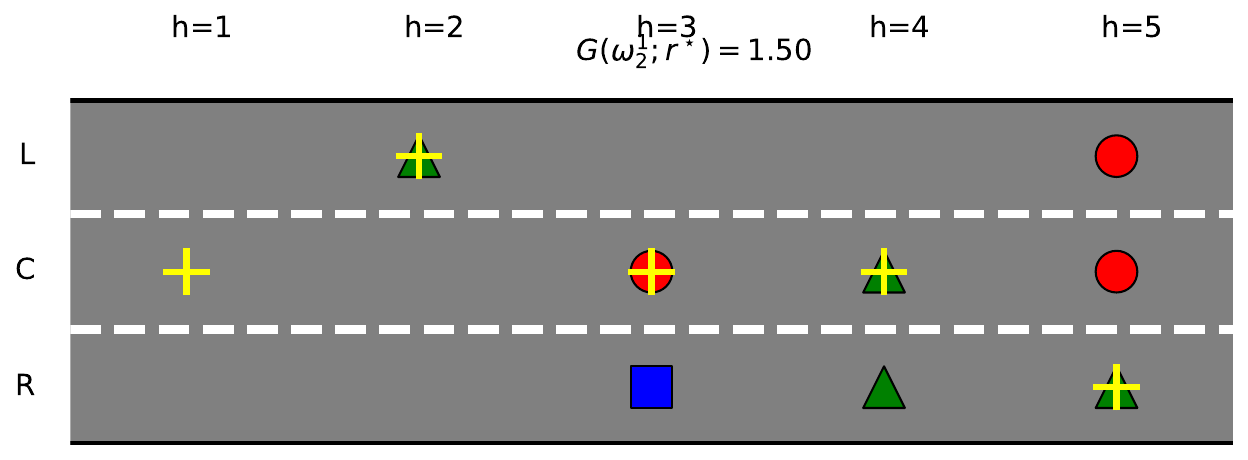}
\end{minipage}
\begin{minipage}[t!]{0.49\textwidth}
    \centering
    \includegraphics[width=0.95\linewidth]{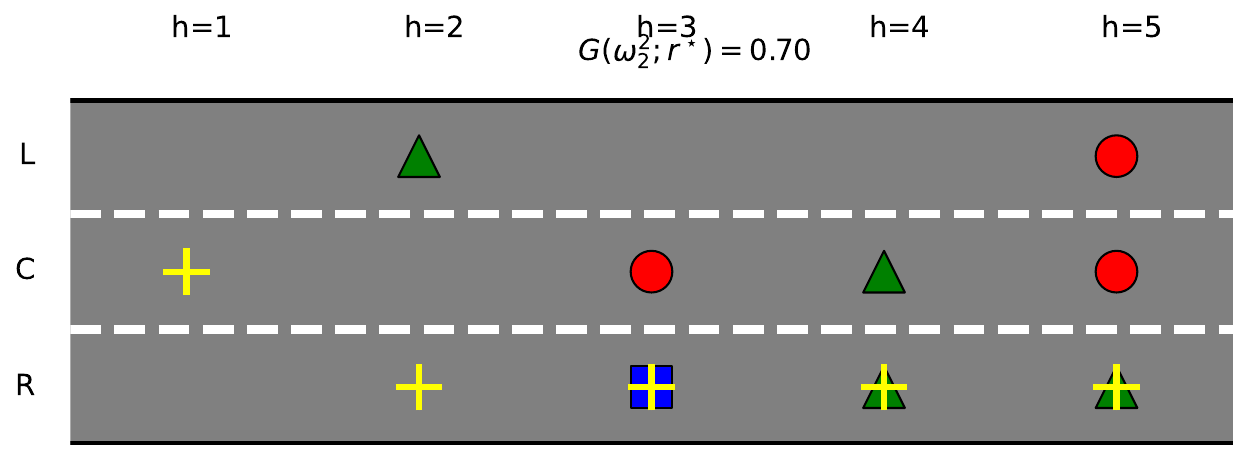}
\end{minipage}
\caption{ The trajectories compared in the second feedback.
$\omega^1_2$ is on the left, and $\omega^2_2$ on the right.}
\label{fig: TC 2}
 \end{figure}

 \begin{figure}[!t]
  \centering
  \begin{minipage}[t!]{0.49\textwidth}
    \centering
    \includegraphics[width=0.95\linewidth]{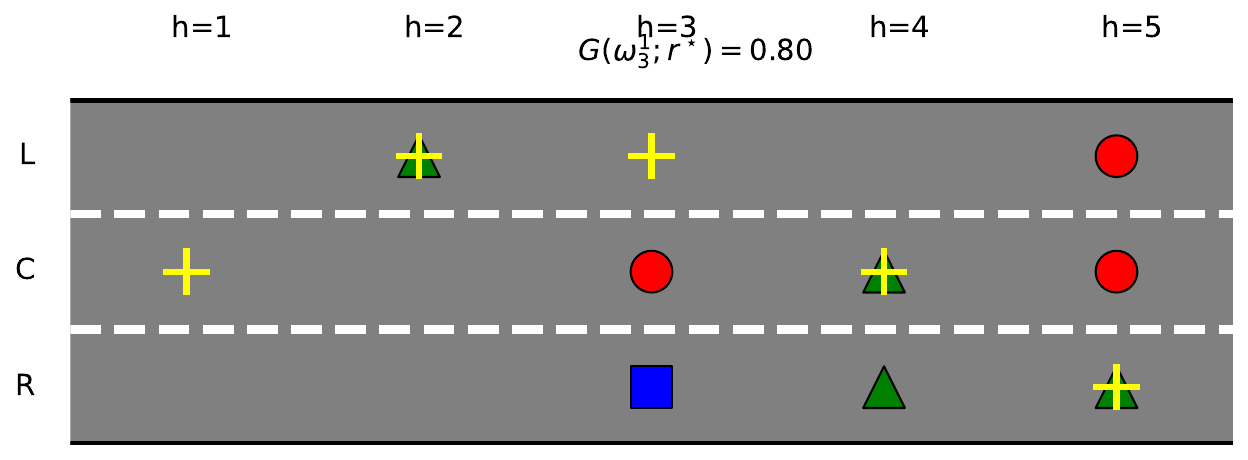}
\end{minipage}
\begin{minipage}[t!]{0.49\textwidth}
    \centering
    \includegraphics[width=0.95\linewidth]{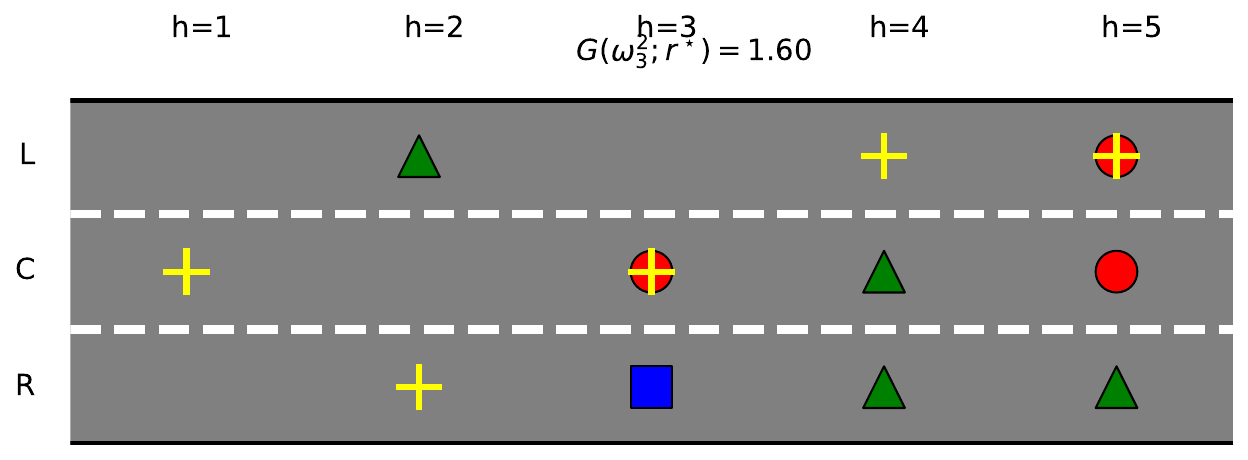}
\end{minipage}
\caption{ The trajectories compared in the third feedback.
$\omega^1_3$ is on the left, and $\omega^2_3$ on the right.}
\label{fig: TC 3}
 \end{figure}

 \begin{figure}[!t]
  \centering
  \begin{minipage}[t!]{0.49\textwidth}
    \centering
    \includegraphics[width=0.95\linewidth]{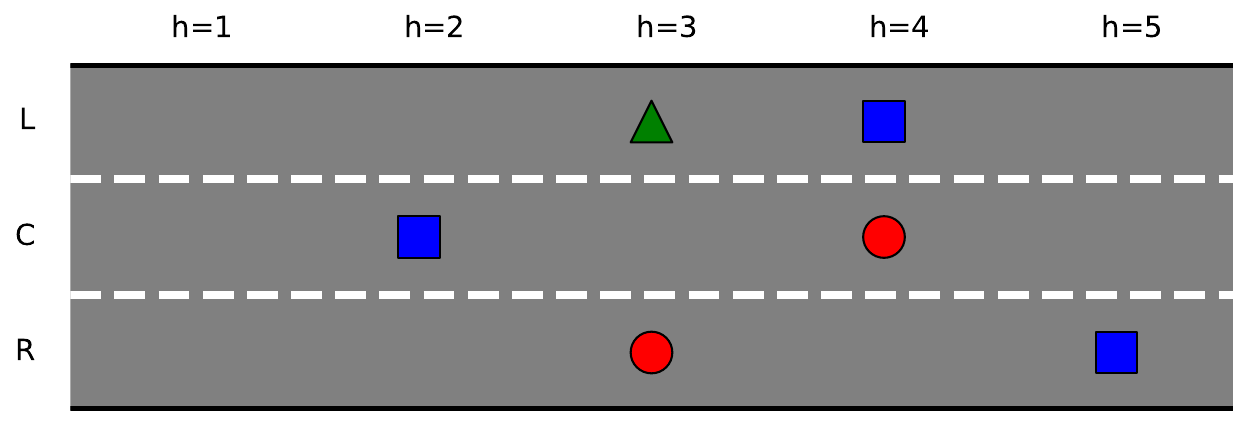}
\end{minipage}
\begin{minipage}[t!]{0.49\textwidth}
    \centering
    \includegraphics[width=0.95\linewidth]{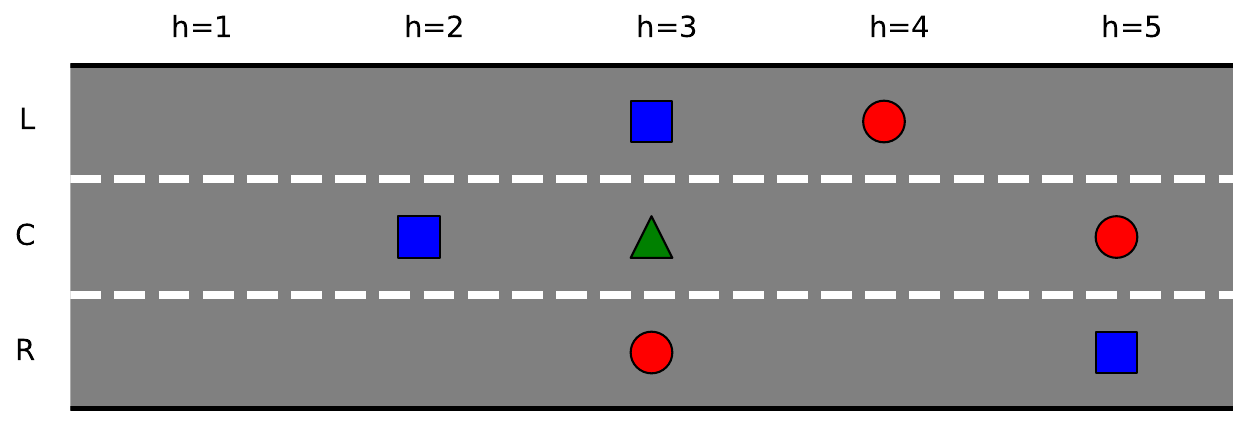}
\end{minipage}
\caption{ (Left) The new map considered for the comparisons
feedback. (Right) The new map considered for the demonstrations feedback.}
\label{fig: new maps C D}
 \end{figure}

 \begin{figure}[!t]
  \centering
  \begin{minipage}[t!]{0.49\textwidth}
    \centering
    \includegraphics[width=0.95\linewidth]{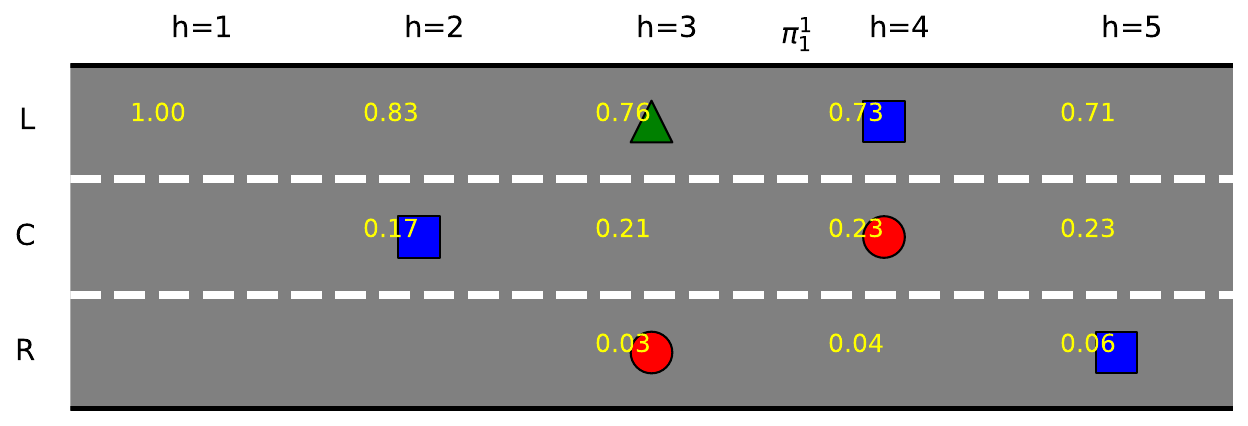}
\end{minipage}
\begin{minipage}[t!]{0.49\textwidth}
    \centering
    \includegraphics[width=0.95\linewidth]{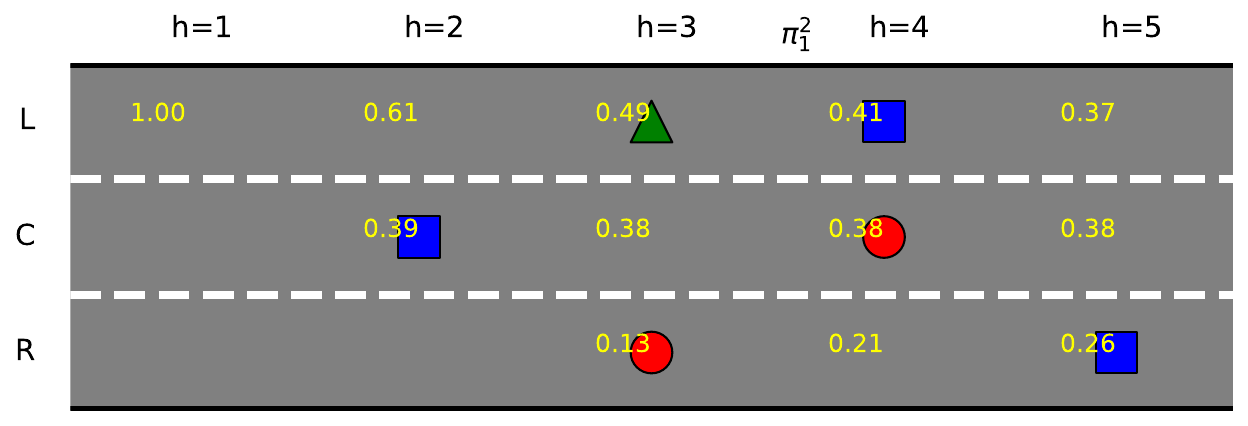}
\end{minipage}
\caption{ The occupancy measures compared in the first comparisons
feedback. $\pi^1_1$ is on the left, and $\pi^2_1$ on the right.}
\label{fig: C 1}
 \end{figure}

 \begin{figure}[!t]
  \centering
  \begin{minipage}[t!]{0.49\textwidth}
    \centering
    \includegraphics[width=0.95\linewidth]{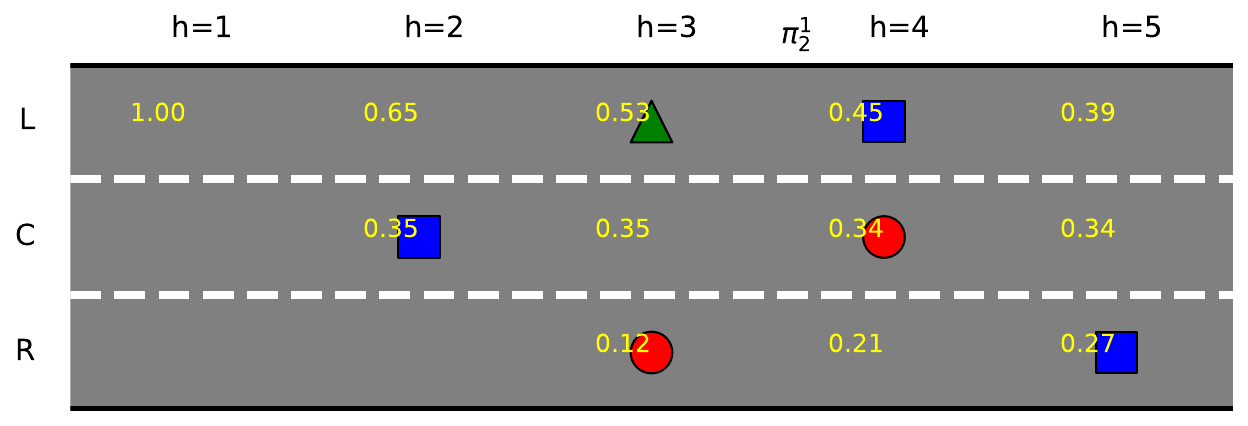}
\end{minipage}
\begin{minipage}[t!]{0.49\textwidth}
    \centering
    \includegraphics[width=0.95\linewidth]{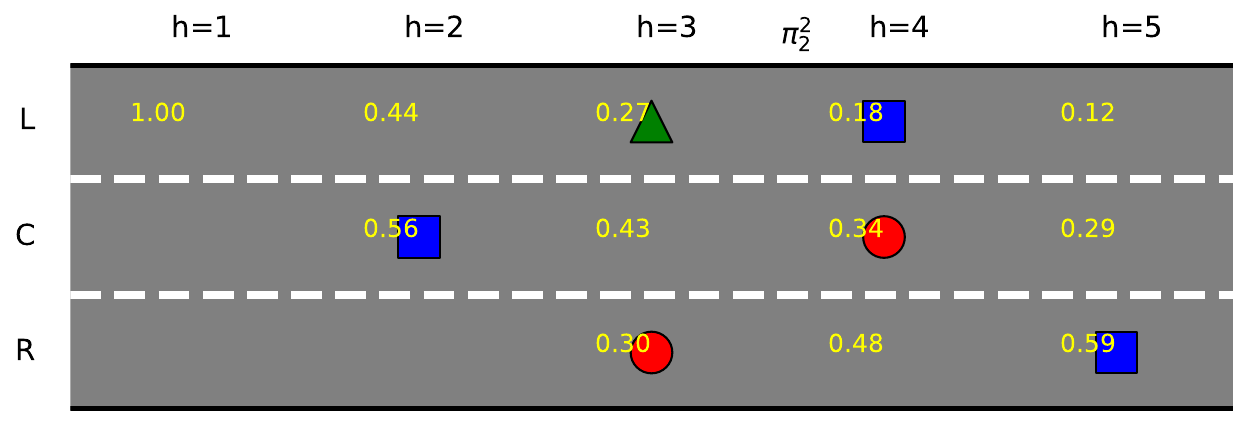}
\end{minipage}
\caption{ The occupancy measures compared in the second comparisons
feedback. $\pi^1_2$ is on the left, and $\pi^2_2$ on the right.}
\label{fig: C 2}
 \end{figure}

 \begin{figure}[!t]
  \centering
  \includegraphics[width=0.46\linewidth]{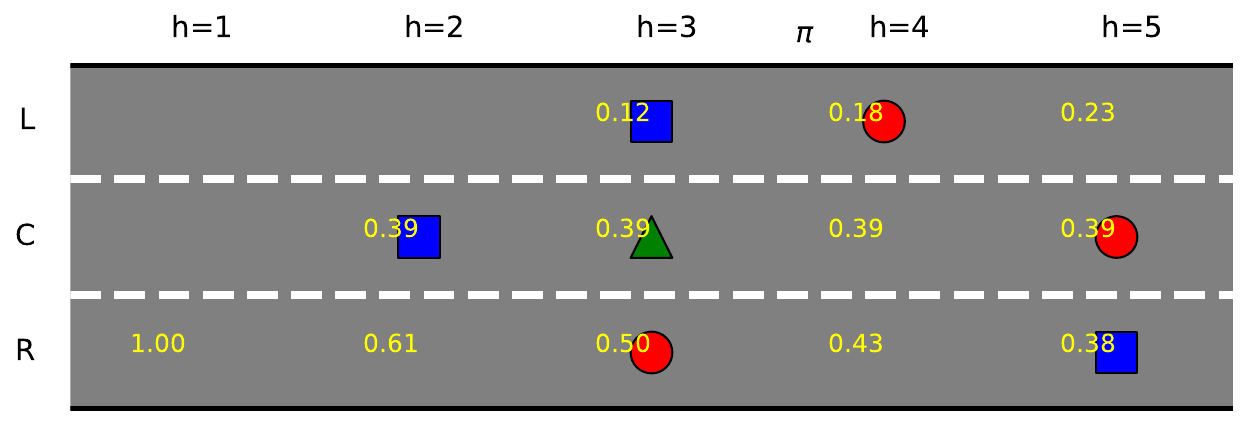}
\caption{ The occupancy measure of the expert's policy for the
demonstrations feedback.}
\label{fig: D 1}
 \end{figure}

 \subsection{Simulation}

 The execution of \rob generated the sequence of reward functions
 $\widehat{r}_{m,k}$ for finding $m$ in Figure \ref{fig: seq r m} on
 the left, while on the right we plotted the corresponding value of the
 objective function $\Delta J (\widehat{r}_{m,k})$.

 The analogous plots for $M$ are in Figure \ref{fig: seq r M}.

 \begin{figure}[!t]
  \centering
  \begin{minipage}[t!]{0.49\textwidth}
    \centering
    \includegraphics[width=0.95\linewidth]{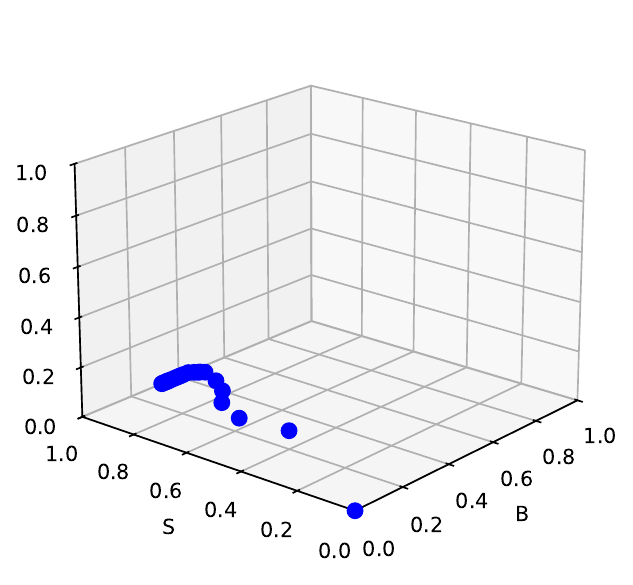}
\end{minipage}
\begin{minipage}[t!]{0.49\textwidth}
    \centering
    \includegraphics[width=0.95\linewidth]{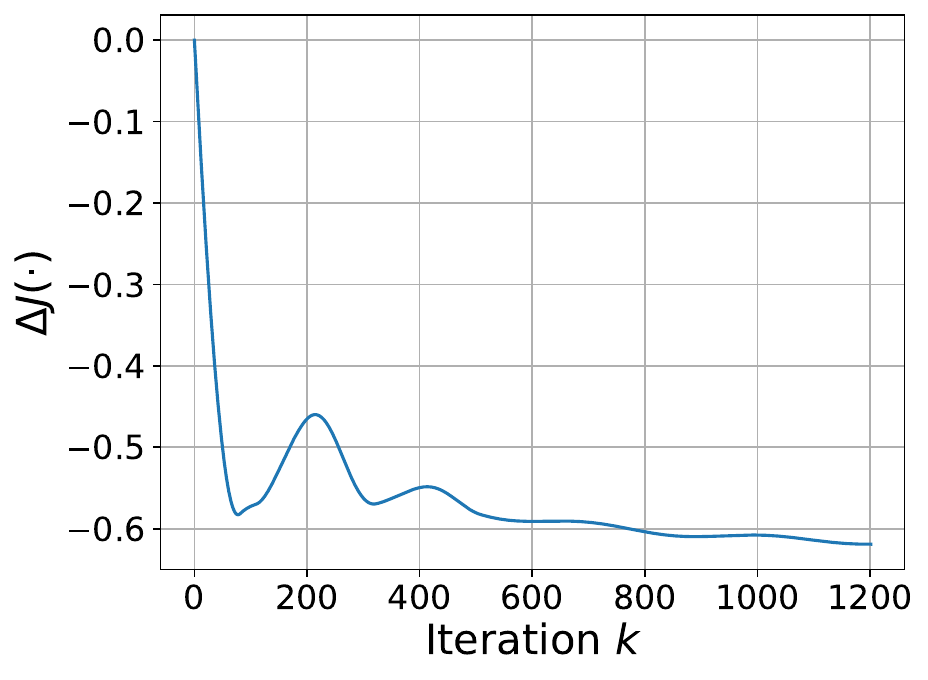}
\end{minipage}
\caption{\footnotesize (Left) The sequence of rewards $\widehat{r}_{m,k}$
computed by our algorithm. (Right) The corresponding values of the objective
function.}
\label{fig: seq r m}
 \end{figure}

 \begin{figure}[!t]
  \centering
  \begin{minipage}[t!]{0.49\textwidth}
    \centering
    \includegraphics[width=0.95\linewidth]{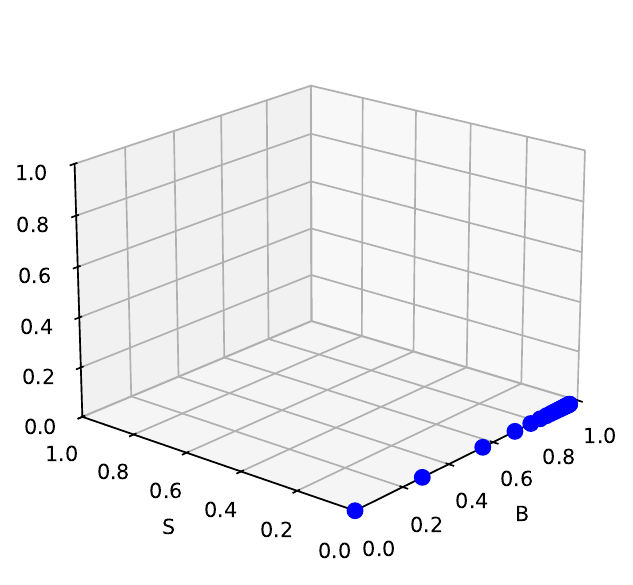}
\end{minipage}
\begin{minipage}[t!]{0.49\textwidth}
    \centering
    \includegraphics[width=0.95\linewidth]{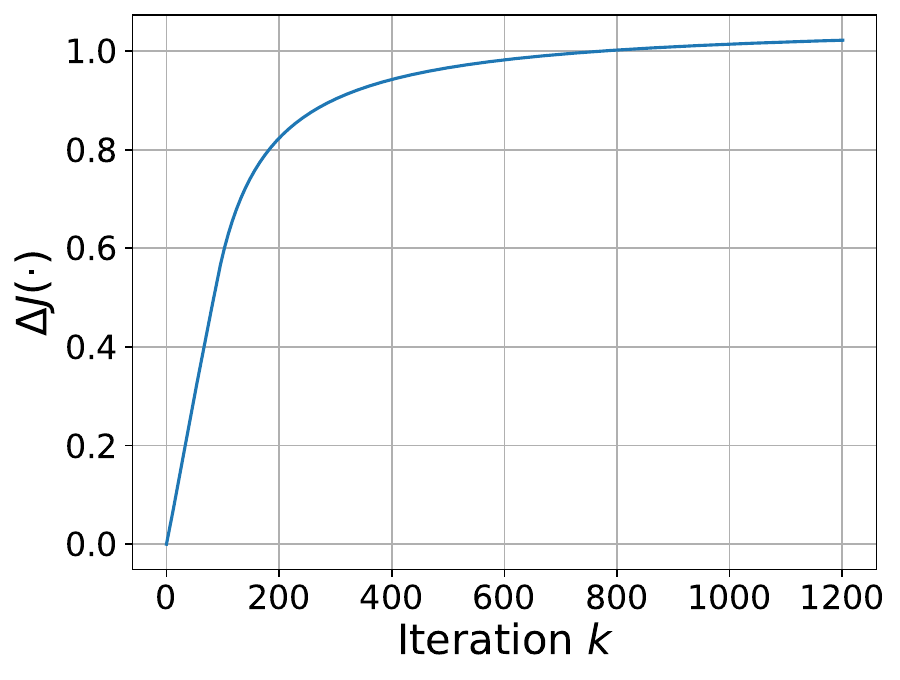}
\end{minipage}
\caption{\footnotesize (Left) The sequence of rewards $\widehat{r}_{M,k}$
computed by our algorithm. (Right) The corresponding values of the objective
function.}
\label{fig: seq r M}
 \end{figure}

 \subsection{Validation through Discretization}
 \label{apx: validation exp}

To understand if the values of
$\widehat{r}_{M},\widehat{r}_{m},\widehat{r}_{\cF,g},\widehat{M}_K,\widehat{m}_k$
computed by the execution of \rob (see Figure \ref{fig: road and fs}, right)
make sense, i.e., are close to the true values $r_M,r_m,r_{\cF,g}$, we have
computed them also through an ``exact'' method, by computing a discretization of
the feasible set and then taken the rewards that maximize/minimize the objective
function (for approximating $r_M,r_m$), and then averaged them. The results of
this computation are reported in Figure \ref{fig: true fs}. Clearly, these
values are close to those in Figure \ref{fig: road and fs}, thus the computation
carried out by \rob makes sense.

\begin{figure}[!t]
  \centering
  \includegraphics[width=0.46\linewidth]{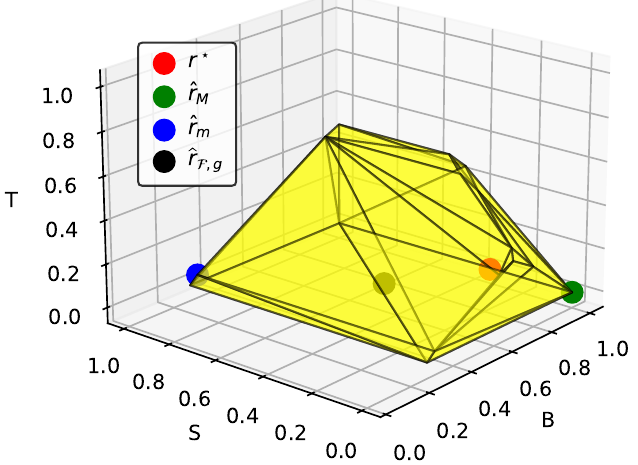}
\caption{\footnotesize The rewards $r_M,r_m,r_{\cF,g}$ computed through a
discretization of the feasible set.}
\label{fig: true fs}
 \end{figure}

\end{document}